\documentclass[10pt, a4paper, copyright, gdm]{google}

\usepackage[numbers, sort&compress]{natbib}
\uselogo{}

\title{Where MLLMs Fail and Why: Causal Task Decomposition for Capability Failure Diagnosis}

\correspondingauthor{xiahu@google.com}

\usepackage{hyperref}
\usepackage{url}
\usepackage{amsmath, amssymb, amsthm}
\usepackage{booktabs}
\usepackage{enumitem}
\usepackage{mathtools}    % for \coloneqq
\usepackage{microtype}
\usepackage{multirow}
\usepackage{graphicx}
\usepackage{subcaption}
\usepackage[capitalize,nameinlink]{cleveref} % must come after hyperref/amsthm
\usepackage{xspace}
\usepackage{tabularx}
\usepackage{xurl}

\newtheorem{theorem}{Theorem}[section]
\newtheorem{proposition}[theorem]{Proposition}

\newtheorem{definition}[theorem]{Definition}
\newtheorem{assumption}[theorem]{Assumption}
\theoremstyle{remark}
\newtheorem{remark}[theorem]{Remark}
\newtheorem{property}[theorem]{Property}

\newcommand{\ourdataset}{CADET\xspace}

\newcommand{\acg}[1]{\textcolor{red}{[#1]}}
\newcommand{\nop}[1]{}

\reportnumber{} % Leave blank if n/a

\renewcommand{\today}{2026-09-30}

\author[1]{Xia Hu}
\author[1]{Brian Potetz}
\author[2]{Chun-Ta Lu}
\author[1]{Huanfen Yao}
\author[1,3]{Leonidas Guibas}
\author[1]{Zhicheng Wang}
\author[1]{Howard Zhou}
\author[1]{Pengfei Xing}
\author[1]{Andrew Gallagher}

\affil[1]{\thepa{}{}}   
\affil[2]{Google Research}
\affil[3]{Stanford University}

\begin{abstract}
End-to-end accuracy on compositional tasks records how often MLLMs fail, but cannot distinguish whether a failure reflects an intrinsic deficit in the targeted capability or a cascading error from an upstream prerequisite.
We propose a causal decomposition framework that isolates these two failure modes through controlled interventions on the prerequisite dependencies of each task.
Our capability metrics (NC, IC, RC) score each task under unassisted, correct, or incorrect prerequisites to diagnose where failures arise; 
contribution metrics (N-Score, S-Score), adapted from probabilities of causation, quantify each prerequisite's necessity and sufficiency to determine why.
We instantiate the framework in CADET, a diagnostic benchmark of 10 composite tasks decomposed into 46 unit tasks with over 33,000 human-annotated questions spanning perception, spatial, temporal, and cognitive categories.
Diagnosing frontier MLLMs with our framework uncovers systematic patterns that end-to-end accuracy obscures. 
Capability-wise, supplying correct prerequisites eliminates 54\% of errors on cognitive tasks, lifting them from weakest to above spatial and temporal.
Prerequisite-wise, causal contributions are concentrated in a few critical prerequisites, and supplying the single most important one alone captures 84\% of the gain from supplying all prerequisites.
\end{abstract}

\begin{document}

\maketitle

\section{Introduction}
\label{sec:intro}

Multimodal large language models~(MLLMs) are increasingly evaluated on compositional tasks requiring multiple interdependent capabilities, from perceiving objects and estimating spatial layouts to tracking over time and reasoning~\citep{yue2024mmmu, yang2025thinking}. 
Such tasks are typically scored by a single end-to-end accuracy, which records how often a model fails but not where its failures originate.
The score conflates an intrinsic deficit, where the model lacks the targeted capability, with a cascading error, where it possesses that capability but fails on an upstream prerequisite. 
In Fig~\ref{fig:teaser}, a model failing on multi-view sequential may lack temporal reasoning or merely misjudge the camera viewpoint on which ordering depends, yielding the same error but opposite capability diagnoses. 
More broadly, low accuracy on high-level cognitive tasks is naturally read as a reasoning deficit, a reading that holds only if the perceptual and spatial steps beneath it are sound, and growing evidence suggests they often are not~\citep{fu2024blink,tong2024eyes}.

A common response examines the constituent steps directly, by decomposing the task into separately scored sub-questions~\citep{press2023measuring, zhou2025reasoning} or inspecting the reasoning trace for the first erroneous step~\citep{lightman2024let}. 
Both localize failure more finely than an aggregate score, yet observing untreated generation alone
\nop{\acg{acg: can you define or clarify ``natural generation''? Does this mean the model output, and / or thinking traces?, perhaps consider ``unmanipulated generation'' or ``untreated generation''}good point, changed to untreated} 
cannot disentangle the two failure modes.
% for two reasons. 
First, a model's generated reasoning trace or answers to isolated sub-questions need not reflect what it relies on to solve the full task~\citep{turpin2023language, chen2025reasoning}. 
Second, even a faithful readout reveals only co-occurrence: a prerequisite failing alongside the final task neither shows how much final failure is attributable to it, nor answers the counterfactual had it been solved correctly.
% of what would happen if it were solved correctly. 
Indeed, the outcome need not hinge on the weakest prerequisite:
in Fig~\ref{fig:teaser}, standalone accuracy is lowest for task-5, yet final performance turns on task-2. 
Input-level interventions do move beyond observation, supplying ground-truth intermediates~\citep{lu2024mathvista} or injecting false premises~\citep{guan2024hallusionbench}, but each tests a single condition\nop{in isolation}, without a dependency structure linking sub-tasks or a measure of each prerequisite's contribution. 
What is missing is an evaluation setting every prerequisite's state within a shared, explicit dependency structure, 
so that the same controlled contrasts apply across capabilities and models, and each prerequisite's contribution is measured rather than inferred.

We propose a causal decomposition framework that turns MLLM evaluation from observation into a controlled experiment.
% Specifically, we decompose each complex task into unit tasks, each testing a single constituent capability, and formalize their task-level prerequisite dependencies as a structural causal model (SCM), establishing a shared dependency structure across models using only standard input–output prompting. 
% Over this SCM, we introduce a ternary intervention scheme that sets each prerequisite to a correct, incorrect, or natural (unassisted) state; 
% contrasting all three states separates the benefit of factual correctness from the conditioning effect of supplying an explicit prerequisite answer, which binary oracle-versus-natural injection cannot distinguish. 
We decompose each task into unit tasks with explicit prerequisite dependencies, and intervene on each by supplying a correct answer, an incorrect answer, or nothing.
% -- the incorrect state, absent in binary oracle injection, isolates the effect of upstream error from upstream absence
To diagnose where failures arise, we define \emph{natural}, \emph{intrinsic}, and \emph{resilience} capability (NC, IC, RC) under these three prerequisite states, isolating a model's performance on a target task from the influence of its modeled prerequisites.
To explore why the overall outcome fails, we introduce the $N$-Score and $S$-Score, which measure each prerequisite's causal necessity and sufficiency for the final outcome: the fraction of success lost when it alone is degraded, and the fraction of failure recovered when it alone is supplied. 
These metrics adapt classical probabilities of causation~\citep{tian2000probabilities, pearl1999probabilities} to the LLM setting without assuming monotonicity, providing valid lower bounds on the probabilities of necessity and sufficiency.

To instantiate this framework, we introduce \ourdataset, a multimodal diagnostic benchmark in which each complex task is decomposed via an explicit structural causal model~(SCM)\nop{\acg{acg: (SCM)}} into unit tasks that each target a single capability, and every unit-task question is annotated with both a verified correct answer and a plausible incorrect answer to enable three-state interventions. 
The benchmark spans 10 tasks, yielding 46 unit tasks and over 33,000 unit-task questions across four capability categories: perception, spatial, temporal, and cognitive. 
All questions and annotations are newly created by human annotators to mitigate data contamination and assure benchmark quality.

Evaluation on 12 frontier MLLMs exposes diagnostic patterns that end-to-end accuracy obscures. 
At capability level, cognitive tasks scores lowest naturally ($\text{NC}{=}0.41$) yet rises to second under correct prerequisites~($\text{IC}{=}0.72$), eliminating $54\%$ of its natural errors; whereas spatial and temporal still retain $72$--$75\%$ of their error, locating the more persistent capability gaps there. 
Injecting incorrect prerequisites degrades all models on perception and spatial ($\text{RC}{<}\text{NC}$) yet improves cognitive in 10 of 12 models. 
Turning to causal contributions, $N$- and $S$-Scores reveal effective bottleneck sparsity, where a few critical prerequisites drive each task, and that a corrupted prerequisite
costs far more than an omitted one.
% is far more destructive than omitting them. 
% ($N(1,0){>}N(1,\varnothing)$), especially for cognitive prerequisites.
The single highest-$S$ prerequisite alone captures $84\%$ full-prerequisite gain, and rarely coincides with the lowest-$\text{NC}$ prerequisite. 
Model-wise exhibits distinct diagnostic profiles, especially, the top four achieve close intrinsic capabilities ($\text{IC}$) despite differing in $\text{NC}$, and pair lower degradation loss ($N$) with higher supply recovery ($S$) than others.
% lower prerequisite vulnerability and higher recoverability than the remaining models.
Together, our work
% framework and \ourdataset 
moves MLLM evaluation beyond scoring how often models fail to diagnosing where and why they fail.

\begin{figure}[t]
    \centering
    \includegraphics[width=0.9\linewidth]{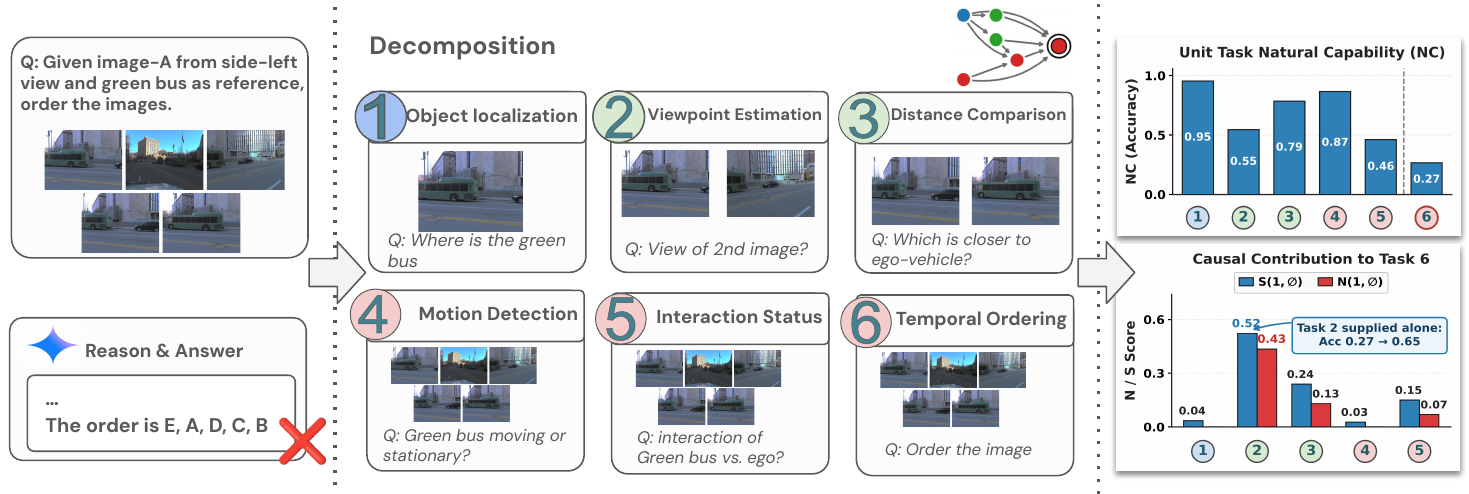}
    \caption{An example \ourdataset task (\textit{left}) decomposed into 6 unit tasks formalized as an SCM (\textit{middle}). \textit{Right}: natural capability (NC) of each unit task (\textit{top}), and their causal contribution ($N$/$S(1,\varnothing)$; \textit{bottom}) to final task-6, where supplying unit task-2 alone raises final accuracy from 0.27 to 0.65.}
    \label{fig:teaser}
\end{figure}

\section{Related Work}
\label{sec:related_work}

\textit{Multimodal Evaluation and Task Decomposition. }
Existing benchmarks evaluate MLLMs across broad capability taxonomies~\citep{yue2024mmmu, liu2024mmbench,zhou2025mathcheck}, 
while specialized suites target individual skills such as perception and spatial~\citep{fu2024blink,yang2025thinking}
or vary visual and textual inputs to test modality reliance~\citep{tong2024eyes,zhang2024mathverse}. 
Beyond evaluating tasks as a whole, one line of work decomposes complex problems into sub-questions and compares sub-task accuracy with overall performance~\citep{press2023measuring, zhou2025reasoning}. 
Another inspects a model's generated reasoning steps to locate where errors first occur~\citep{lightman2024let,zhang2025agent},
though recent studies show that stated chains-of-thought do not always reflect the computation behind the final answer~\citep{turpin2023language,chen2025reasoning}. 
Both read out the model's unassisted behavior, linking sub-task and final correctness by co-occurrence; we instead set each prerequisite to a known correct or incorrect state.

\noindent\textit{Causal Interventions for Model Diagnosis.} 
Causal interventions are used to diagnose MLLMs and LLMs 
by modifying specific components at inference time to measure their effect on outputs, such as patching internal activations \citep{meng2022rome,wu2023das} or perturbing visual and textual modalities~\citep{li2025treble,hasic2026cma}. 
For multi-step reasoning, recent studies apply step-level interventions to execution traces, editing intermediate outputs or re-executing trajectories to identify failure-inducing steps~\citep{ma2026dover,wang2026chief,bonagiri2026causalflow,shah2026car}. 
At task-input level, prior studies provide ground-truth intermediate results to isolate downstream reasoning~\citep{lu2024mathvista,shao2024visual} or inject false premises to test robustness under misleading context~\citep{guan2024hallusionbench,ming2025faitheval}. 
% Across these directions, interventions either rely on dependency structures tied to model internals and single-run trajectories, or test single input conditions in isolation. We instead apply ternary interventions to task-defined prerequisite capabilities over a shared per-task causal graph to diagnose MLLMs on a common basis.
These interventions either operate on model internals or single-run trajectories, or test one input condition in isolation; 
ours act on a task-defined causal graph shared across models.

\noindent\textit{Quantifying Prerequisite Contributions.}
Attributing MLLM behavior to individual components often involves distinguishing whether a factor is necessary to prevent failure or sufficient to drive success, a distinction formalized by counterfactual necessity and sufficiency~\citep{pearl1999probabilities,tian2000probabilities} and adopted in explainable AI for feature-level attribution~\citep{watson2021local,galhotra2021explaining}. 
In LLMs, recent studies adapt these causal concepts to test whether responses on synthetic logic problems conform to the ground-truth rules of the queries~\citep{gonzalez2024does}, or to add missing and prune redundant chain-of-thought steps~\citep{yu2026causal}. 
Both operate on the model's own outputs; we instead attribute final-task outcomes to prerequisite via necessity and sufficiency.
% We formulate necessity- and sufficiency-based diagnostic metrics over prerequisite unit tasks to pinpoint capability bottlenecks and high-leverage repair targets across MLLMs.

\section{Causal Decomposition Framework}
\label{sec:decomp}

% This section formalizes the framework:
% % outlined in Sec~\ref{sec:intro}:
% we represent task dependencies as structural causal model~(Sec~\ref{subsec:scm}), define ternary interventions tailored to LLM evaluation~(Sec~\ref{subsec:intervention}), derive capability metrics isolating intrinsic proficiency~(Sec~\ref{subsec:capability}) and contribution metrics quantifying prerequisite's causal role~(Sec~\ref{subsec:ns_score}).

\subsection{Task Decomposition via Structural Causal Models}
\label{subsec:scm}

Each evaluation task is decomposed into unit tasks testing a single constituent capability, with directed dependencies where some tasks rely on the outcomes of others. 
% Such dependencies cause upstream outcomes to causally propagate downstream, making it essential to model the dependency structure explicitly to disentangle cascading influences from those intrinsic to each capability. 
We represent these dependencies as a structural causal model~(SCM)~\citep{pearl2009causality, peters2017elements}, leveraging the established framework of causal interventions for the analysis below.

\begin{definition}[Task Decomposition SCM]
\label{def:scm}
Consider an evaluation task decomposed into $n$ constituent capabilities. 
A \emph{task decomposition} is a structural causal model (SCM) over unit task variables $X_1, X_2 \ldots, X_n$, where $X_i$ denotes the outcome of unit task $i$, given by the structural assignment:
\begin{equation}
\small
  X_i \;:=\; f_i\!\bigl(\mathrm{Pa}(X_i),\; N_i\bigr),
  \quad i = 1, \ldots, n,
\label{eq:scm}
\end{equation}
where $f_i$ is the structural function determining $X_i$,
$\mathrm{Pa}(X_i) \subseteq \{X_1, \ldots, X_n\}$ is the parents set of $X_i$,
and $N_i$ is a noise variable capturing unmodeled factors such as the model's sampling stochasticity.
\end{definition}

Root nodes (those with $\mathrm{Pa}(X_i) {=} \emptyset$) represent foundational capabilities that can be assessed independently, 
while the final node $Y {:=} X_n$ corresponds to the overall task outcome.

\subsection{Causal Intervention Design}
\label{subsec:intervention}

% Under the dependency structure of SCM, performance on any downstream task conflates two distinct factors: the intrinsic capability and the cascading influence propagated from its upstream prerequisites. 
To disentangle intrinsic capability from cascading upstream influence, and quantify each capability's causal role, we adopt the interventionist framework of causal inference~\citep{pearl1999probabilities,pearl2009causality,rubin1974estimating} and introduce a ternary intervention scheme tailored to LLM evaluation.

\begin{definition}[Ternary Intervention]
\label{def:ternary}
For unit task $X_i$, we define an intervention $A_i {=} a$, $a {\in} \{1, 0, \varnothing\}$:
(i) $A_i {=} 1$ denotes injecting the correct answer for $X_i$ into the prompt;
(ii) $A_i {=} 0$ denotes injecting an incorrect answer for $X_i$;
(iii) $A_i {=} \varnothing$ denotes the {natural state} where $X_i$ is not injected.
% \begin{itemize}[leftmargin=2em,itemsep=1pt,topsep=0pt,parsep=0pt,partopsep=0pt]
%     \item $A_i = 1$ denotes injecting the correct answer for $X_i$ into the prompt;
%     \item $A_i = 0$ denotes injecting an incorrect answer for $X_i$ into the prompt;
%     \item $A_i = \varnothing$ denotes the {natural state} where $X_i$ is not injected.
% \end{itemize}
\end{definition}

\noindent\textbf{Notation.} All probabilities are induced by a fixed model $\mathcal{M}$, whose
subscript we suppress. We write $\mathcal{A} {:=} \{1, 0, \varnothing\}$ for the ternary
intervention space. For a unit task set $S$ and $a {\in} \mathcal{A}$, we write
$do(\mathbf{A}_{S} {=} a)$ for the homogeneous intervention placing every $X_j {\in} S$ in state $a$, 
and abbreviate 
$\mathbf{A}_{-i} {:=} \mathbf{A}_{\mathrm{Pa}(Y) {\setminus} \{X_i\}}$.

Notably, the distinction between $A_i {=} \varnothing$, $A_i {=} 0$ is critical:
the former preserves the model's natural execution path by evaluating it unassisted,
where the latter imposes an active counterfactual intervention testing reasoning under a false premise.
Since LLM evaluation operates through input manipulation rather than passive observation, this ternary formulation emerges as the natural space of interventions, allowing us to systematically evaluate model behavior across each pairwise contrast of states.

\begin{assumption}[Exogeneity under Intervention-Only Identification]
\label{assum:interv}
In LLM evaluation setting, the state of a unit task $X_i \in \{1, 0\}$ admits no faithful observational readout and is identified only through the intervention $A_i {=} a$.
Consequently, probabilities conditioning on  $X_i {=} x$ for $x \in \{1, 0\}$ are identified with the corresponding interventional probabilities:
\begin{equation}
\small
  P(\cdot \mid X_i = x) \;=\; P(\cdot \mid \operatorname{do}(A_i = x)).
\end{equation}
\end{assumption}

\begin{remark}[Justification]
\label{remark:exogeneity}
Assumption~\ref{assum:interv} 
follows from the LLM evaluation setting:
% is not a modeling choice but a consequence of the LLM evaluation setting: 
the state of $X_i$ admits no faithful observational readout~\citep{turpin2023language,chen2025reasoning},
thus a reliable mechanism to place $X_i$ in a determinate state $x {\in} \{1, 0\}$ is to inject the corresponding answer into the prompt, 
making conditioning on $X_i {=} x$ and intervening $\operatorname{do}(A_i = x)$ operationally identical. 
This condition coincides with the \emph{exogeneity} assumption $P(y_x) {=} P(y \mid x)$ of \citet{pearl1999probabilities,tian2000probabilities}, but with an inverted justification: 
classical exogeneity invokes the absence of confounders to treat observations as interventions, whereas ours follows from the absence of a faithful observational channel making intervention the reliable state-assignment mechanism.
\end{remark}

\subsection{Capability Metrics}
\label{subsec:capability}

With the ternary intervention mechanism established, 
we can now place the prerequisites $\mathrm{Pa}(X_i)$ of any target task $X_i$ into controlled states and measure how the model's performance on $X_i$ responds. 
% A single evaluation under the natural state $A = \varnothing$ captures only the model's unassisted outcome, leaving its intrinsic proficiency on $X_i$ confounded with upstream influence.
By additionally evaluating under $\mathrm{do}(\mathbf{A}_{\mathrm{Pa}(X_i)} {=} 1)$ and $\mathrm{do}(\mathbf{A}_{\mathrm{Pa}(X_i)} {=} 0)$, 
we isolate the model's capability on $X_i$ when its declared prerequisites are correctly provided versus actively corrupted.

\begin{definition}[Causal Capability]
\label{def:cap}
Let $X_i$ be any task node in the SCM with prerequisite dependencies
$\mathrm{Pa}(X_i)$. For an intervention state $a \in \mathcal{A}$, the
\textbf{causal capability} of $\mathcal{M}$ on $X_i$ is
\begin{equation}
\small
\mathrm{Cap}(X_i; a) := P\big(X_i {=} 1 \mid do(\mathbf{A}_{\mathrm{Pa}(X_i)} {=} a)\big),
\end{equation}
i.e., the probability that $\mathcal{M}$ solves $X_i$ when all of its prerequisites are
placed in state $a$. This yields three complementary metrics:
(i) \textbf{natural capability} $\mathrm{NC}(X_i) := \mathrm{Cap}(X_i; \varnothing)$, the natural performance where no prerequisite is injected;
(ii) \textbf{intrinsic capability} $\mathrm{IC}(X_i) := \mathrm{Cap}(X_i; 1)$, the performance isolated from declared prerequisites;
% isolated from upstream dependencies; 
and (iii) \textbf{resilience capability} $\mathrm{RC}(X_i) := \mathrm{Cap}(X_i; 0)$, the robustness under corrupted prerequisite context. 
\end{definition}

For root nodes ($\mathrm{Pa}(X_i) {=} \emptyset$), no upstream intervention applies and only $\mathrm{NC}(X_i)$ is defined. 
For non-root nodes, the three metrics disentangle a model's proficiency on $X_i$ from the influence of modeled prerequisites. 
Specifically, comparing $\mathrm{IC}(X_i)$ with $\mathrm{NC}(X_i)$ separates two failure modes: 
a large gap $\mathrm{IC}(X_i) {-} \mathrm{NC}(X_i)$ indicates that the model can solve $X_i$ given correct prerequisites, pointing to upstream cascading errors as the primary bottleneck; 
a small gap with low $\mathrm{IC}(X_i)$ suggests an intrinsic deficit on $X_i$ even when all modeled prerequisites are satisfied.
% Finally, $\mathrm{RC}(X_i)$ measures performance under actively corrupted prerequisites, quantifying resilience to upstream misinformation.
Finally, comparing $\mathrm{RC}(X_i)$ with $\mathrm{NC}(X_i)$ shows whether false prerequisites mislead the model below its natural baseline, or still help guide its reasoning despite being factually wrong.

\subsection{Contribution Metrics: N-Score and S-Score}
\label{subsec:ns_score}

Building on node-level capability diagnostics, understanding how upstream unit tasks across the SCM affect the final outcome $Y$ requires quantifying their causal contributions. 
To formalize this causal attribution, we draw upon the \emph{probabilities of causation} framework~\citep{pearl1999probabilities,tian2000probabilities}, which evaluates counterfactual causal influence through the Probability of Necessity ($\mathrm{PN}$) and Probability of Sufficiency ($\mathrm{PS}$). 
We do not assume monotonicity: in LLM reasoning, injecting incorrect answer ($A_i {=} 0$) can help more than no injection ($A_i {=} \varnothing$), so $\mathrm{RC}$ may exceed $\mathrm{NC}$~(Sec~\ref{subsec:exp_cap}).

% While capability metrics diagnose the model's localized proficiency on specific nodes, they do not quantify how much each capability causally contributes to the overall task outcome.
% To measure this, we adopt the \emph{probabilities of causation} framework~\citep{pearl1999probabilities, tian2000probabilities}, which formalizes causal attribution along two complementary dimensions: necessity and sufficiency.

% In standard causal inference applications, probabilities of causation are often point-identified under a strict monotonicity assumption (i.e., treatment never harms the subject). However, in the context of LLM reasoning, this assumption is frequently violated: injecting an incorrect context might still provide beneficial structural hints compared to providing no context at all. We explicitly reject the monotonicity assumption, which implies that we can only obtain conservative lower bounds rather than exact values.

\begin{property}[Absence of Monotonicity]
\label{assum:non_monotonicity}
Under the LLM evaluation setting, the causal effect of intervening on a unit task $X_i$ via $A_i$ on downstream outcome task $Y$ is not assumed to satisfy monotonicity. That is, the natural ordering below is not required to hold in general:
\begin{equation}
\small
P(Y {=} 1 \mid do(A_i {=} 0)) \le P(Y {=} 1) \le P(Y {=} 1 \mid do(A_i {=} 1)).
\end{equation}
\end{property}

Without monotonicity, exact $\mathrm{PN}$ and $\mathrm{PS}$ values cannot be determined from interventional probabilities alone, 
so we introduce the \textbf{N-Score} and \textbf{S-Score} as observable lower bounds to capture the causal contribution of each prerequisite $X_i \in \mathrm{Pa}(Y)$ to the final outcome $Y$:

\begin{definition}[N-Score]
\label{def:nscore}
Let $Y$ be a target outcome task with prerequisite dependencies $\mathrm{Pa}(Y)$, let
$X_i \in \mathrm{Pa}(Y)$, and let $(a, a')$ be a pair of distinct states in $\mathcal{A}$.
The \textbf{N-Score} of $X_i$ on $Y$ is
\begin{equation}
\footnotesize
\mathrm{N}_{(a,a')}(X_i, Y) := \max\left\{0,\;
\frac{P\big(Y{=}1 \mid do(\mathbf{A}_{\mathrm{Pa}(Y)} {=} a)\big)
    - P\big(Y{=}1 \mid do(A_i {=} a', \mathbf{A}_{-i} {=} a)\big)}
     {P\big(Y{=}1 \mid do(\mathbf{A}_{\mathrm{Pa}(Y)} {=} a)\big)}\right\},
\end{equation}
measuring the fraction of $Y$'s success under state $a$ harmed\nop{\acg{acg: replace with harmed?}} by degrading $X_i$ alone from $a$ to $a'$.
\end{definition}

\begin{definition}[S-Score]
\label{def:sscore}
Under the setting of Definition~\ref{def:nscore}, the \textbf{S-Score} of $X_i$ on $Y$ is
\begin{equation}
\footnotesize
\mathrm{S}_{(a,a')}(X_i, Y) := \max\left\{0,\;
\frac{P\big(Y{=}1 \mid do(A_i {=} a, \mathbf{A}_{-i} {=} a')\big)
    - P\big(Y{=}1 \mid do(\mathbf{A}_{\mathrm{Pa}(Y)} {=} a')\big)}
     {1 - P\big(Y{=}1 \mid do(\mathbf{A}_{\mathrm{Pa}(Y)} {=} a')\big)}\right\},
\end{equation}
measuring the fraction of $Y$'s failure under state $a'$ recovered by upgrading $X_i$ alone from $a'$ to $a$.
\end{definition}

\begin{theorem}
\label{thm:pn_ps_bounds}
Under Assumption~\ref{assum:interv} and Property~\ref{assum:non_monotonicity}, N-Score and S-Score provide valid lower bounds on the Probability of Necessity ($\mathrm{PN}$) and Probability of Sufficiency ($\mathrm{PS}$).
\end{theorem}

Specifically, {N-Score} quantifies \textit{causal necessity}
by measuring how much of the model's success on $Y$ depends on $X_i$. 
% by reporting the fraction of success under state $a$ lost when $X_i$ alone is degraded to $a'$.
A high N-Score pinpoints critical performance bottlenecks where one degraded task breaks overall success even when all other prerequisites remain at $a$. 
Conversely, {S-Score} quantifies \textit{causal sufficiency} 
by measuring how much of the model's failure on $Y$ can be overcome through $X_i$.
% by reporting the fraction of failure under state $a'$ recovered when $X_i$ alone is upgraded to $a$.
A high S-Score identifies high-leverage prerequisites where supplying a single task rescues overall success even when other prerequisites remain at $a'$. 
Setting the other prerequisites $\mathbf{A}_{-i}$ asymmetrically to $a$ for N-Score and $a'$ for S-Score aligns with realistic evaluation 
by testing single-task vulnerability under favorable prerequisite conditions and single-task rescue under degraded ones,
making the two metrics complementary rather than directly convertible~\citep{pearl1999probabilities}. 
% as under a shared background state. 
We provide the general definitions of N-Score, S-Score under arbitrary $\mathbf{A}_{-i}$ and theoretical discussions in Appendix.

\begin{table}[t]
    \centering
    \small
    \caption{Operational semantics of N-Score, S-Score across 4 ordered state contrasts $(a, a')$.}
    \label{tab:causal_contrasts}
    \resizebox{0.95\linewidth}{!}{
    \begin{tabular}{c p{0.48\textwidth} p{0.48\textwidth}}
    \toprule
    \textbf{Contrast} & \textbf{N-Score}~(Success drop from $\cdots$) & \textbf{S-Score}~(Failure recovery from $\cdots$) \\
    \midrule
    $(1, 0)$ & corrupting $X_i$ into a wrong answer {in oracle env}. & fixing $X_i$ into the correct answer {in all-wrong env}. \\
    \addlinespace[1pt]
    $(1, \varnothing)$ & leaving $X_i$ unassisted {in oracle env}. & providing the correct $X_i$ alone {in natural env}. \\
    \addlinespace[1pt]
    $(\varnothing, 0)$ & injecting a wrong answer for $X_i$ {in natural env}. & removing a wrong answer for $X_i$ {in all-wrong env}. \\
    \addlinespace[1pt]
    $(0, \varnothing)$ & removing $X_i$'s wrong answer {in all-wrong env}. & injecting a wrong $X_i$ as a scaffold {in natural env}. \\
    \bottomrule
    \end{tabular}
    }
\end{table}

% The two scores are evaluated from opposite baselines by design: necessity is assessed from
% the treatment context $\mathbf{A}_{-i} = a$, where $Y$ is likely to succeed, while
% sufficiency is assessed from the control context $\mathbf{A}_{-i} = a'$, where $Y$ is
% likely to fail. We discuss this coupling and its role in reducing the
% $|\mathcal{A}|^{|\mathrm{Pa}(Y)|}$ context space in Section~\ref{sec:controlling}.
% %
% The N-Score quantifies the minimum probability that a task that succeeded under $X_i = x$ would have failed under $X_i = x'$---a conservative measure of the \emph{necessity} of state $x$. The S-Score quantifies the minimum probability that a task that failed under $X_i = x'$ would have succeeded under $X_i = x$---a conservative measure of the \emph{sufficiency} of state $x$. Both are \emph{lower bounds} computable from experimentally observable marginal probabilities $p_x$ and $p_{x'}$ alone, without requiring the joint distribution $P(Y_x, Y_{x'})$.

% Following the multi-valued treatment framework of \citet{li2024probabilities}, we reduce the ternary treatment to pairwise binary contrasts. For states $x, x' \in \{1, 0, \varnothing\}$ with $x \neq x'$, each pair $(x, x')$ defines a binary causal comparison with distinct semantics:

In practice, we evaluate N-Score and S-Score across four ordered state contrasts $(a, a')$ over $\mathcal{A} {=} \{1, 0, \varnothing\}$: $(1, 0)$, $(1, \varnothing)$, $(\varnothing, 0)$, and $(0, \varnothing)$. 
As summarized in Tab~\ref{tab:causal_contrasts}, the first two contrasts 
%anchor on the correct state ($a=1$) to 
measure both degradation from prerequisite corruption or omission and recovery through targeted correction or repair, whereas the latter two evaluate both directions between $\varnothing$ and $0$ to capture either misinformation harm or structural scaffolding under Property~\ref{assum:non_monotonicity}.
Together, our node-level capability metrics ($\text{NC}, \text{IC}, \text{RC}$) and edge-level contribution metrics ($\text{N-Score}, \text{S-Score}$) form a causal diagnostic framework for complex multimodal tasks, which we instantiate in Sec~\ref{sec:benchmark}.

\section{Benchmark}
\label{sec:benchmark}

To apply the causal decomposition framework to frontier MLLMs and uncover where and why they fail, 
we construct \ourdataset, a multimodal benchmark where complex tasks are decomposed into unit tasks connected via SCM and every prerequisite admits controlled interventions~(details in Appendix).

\noindent\textbf{Benchmark Overview.}
\ourdataset comprises 10 multimodal evaluation tasks across over 1,100 instances, decomposed into 46 atomic unit tasks and over 33,000 unit-task questions, 
with both overall and their constituent unit tasks categorized into four core capabilities: perception, spatial, temporal, and cognitive.
To operationalize our causal framework, each overall task is paired with an explicit SCM defining its prerequisite structure, 
and every unit task is annotated with both a correct and an incorrect answer, along with visual annotations on the images where applicable, to support the ternary intervention scheme ($A_i {\in} \{1, 0, \varnothing\}$) at every node. 
The multimodal inputs span single, multi-images and videos, while the outputs use structured formats (multiple-choice, integer, or list) for 43 of the 46 tasks and free-form text for the remaining three.

\begin{figure}[t]
    \centering
    \includegraphics[width=0.8\linewidth]{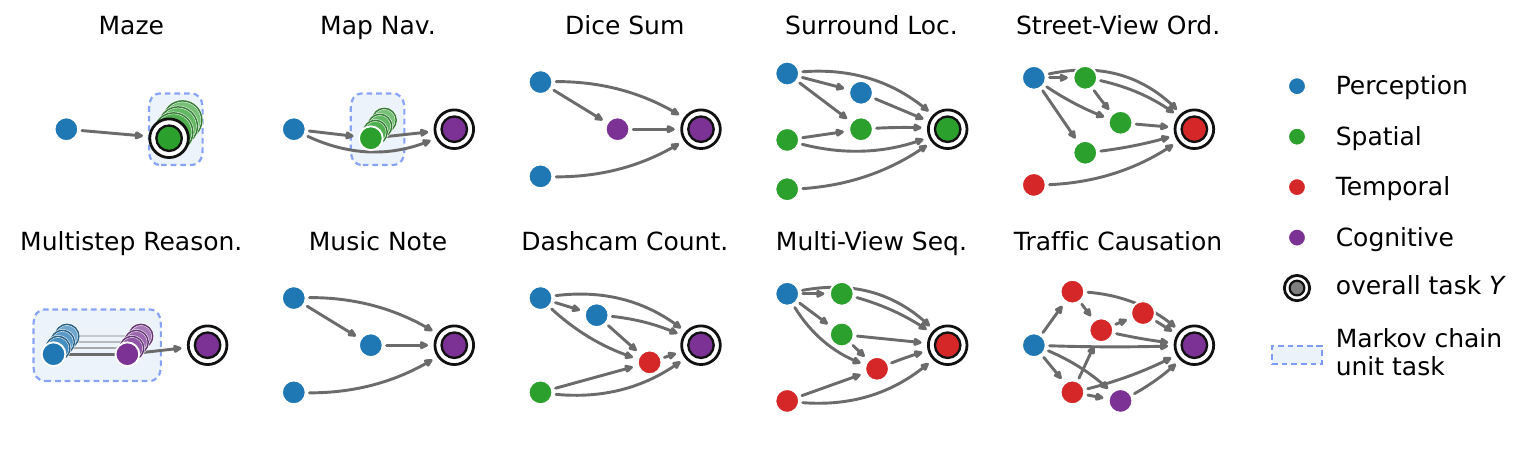}
    \caption{Overview of SCMs corresponding to 10 tasks in \ourdataset.}
    \label{fig:scm_gallery}
\end{figure}

\noindent\textbf{Task Design and SCM Structure.}
Tasks are drawn primarily from real-world scenarios and designed as compositional problems requiring interdependent capabilities on which frontier MLLMs are far from saturation, so that diagnosis targets gaps these models still face.
% We design for broad coverage across input modalities, application domains, and capability compositions, with the four capability categories spanning diverse roles from foundational perceptual inputs to high-level cognitive reasoning. 
We further cover diverse input modalities, application domains, and capability compositions, from foundational perception to high-level cognitive reasoning.
For each task, we formalize the prerequisite dependencies among its unit tasks into an SCM~(Fig~\ref{fig:scm_gallery}). 
These structures range from shallow aggregations (e.g., Dice Sum) to deep multi-path cascades (e.g., Traffic Causation).
In decomposition, we identify unit tasks that involve iterative solving of structurally recurring sub-problems, where each step depends on the previous step's outcome and reasoning depth itself becomes a variable; 
we define these as \textbf{Markov chain unit tasks}.
In the task-level SCM, a Markov chain unit task participates as a single node; for finer-grained analysis, we unfold it into a step-wise chain to track how capabilities evolve with reasoning depth.
Tasks whose primary structure is a Markov chain are analyzed separately in Sec~\ref{subsec:exp_markov}.

\noindent\textbf{Data Construction and Quality Assurance.}
All evaluation content, including questions, answers, and unit-task annotations, is newly created by an experienced rater pool to mitigate data contamination; 
source images and videos are drawn from publicly available datasets with permissive licenses~(e.g., \citet{wilson2023argoverse2}). 
The benchmark is constructed in two stages: raters first generate overall task questions and ground-truth answers under predefined task definitions. 
Based on each task's SCM, we decompose its unit tasks into standardized atomic annotation tasks, such as bounding box localization and semantic labeling. 
Raters provide the corresponding visual annotations and ground-truth labels, which are then converted into unit-task questions and correct answers via structured templates, ensuring consistent alignment between all evaluation items and the underlying causal structure. 
Incorrect answers are randomly sampled from the valid answer space rather than manually crafted, providing a standardized, model-independent perturbation.
% an unbiased estimate of the effect of misinformation. 
To ensure benchmark quality, each annotation across both stages is independently verified by a quality-check rater who returns deficient items for iterative revision until acceptance. 
Beyond per-item verification, instances flagged\nop{by annotators} as ambiguous or un-annotatable are removed during final assembly to ensure reliable causal diagnosis.

\section{Experiments and Findings}
\label{sec:exp}

\subsection{Experiment Setting}
\label{sec:exp:setup}

\textbf{Evaluated Models.}
We evaluate 12 frontier MLLMs across six model families, covering both proprietary series: Gemini~\citep{gemini3pro2025}, GPT~\citep{gpt54}, Grok~\citep{grok4modelcard2025}; and open-weight series: Gemma~\citep{gemma4}, Qwen~\citep{qwen3.5}, Kimi~\citep{team2026kimi25}.
This selection encompasses different model families at varying capability scales, enabling both cross-family and within-family diagnostic comparisons. 
All models are queried through public APIs with default parameter setting.
% and maximum supported context length.
Thinking mode is set to high or left at its default if not-configurable. 
% For models that support configurable reasoning effort, the thinking mode is set to high or left at its default active state.

\noindent\textbf{Intervention Implementation.}
To operationalize the ternary intervention, for each prerequisite unit task $X_i$, we inject its intervened state~($A_i {\in} \{1,0\}$) into prompt as a structured JSON question-answer pair. 
For $A_i {=} 1$, alongside the correct textual answer, we render corresponding visual annotations (e.g., bounding box, path) 
% onto the image 
whenever the prerequisite's ground truth directly represents spatial information in the scene. 
For $A_i {=} 0$, incorrect answer is injected via text without modifying the image, to prevent shifting the visual referent of the query and ensure we test reasoning under a false premise about the intended target. 
Under natural state $A_i {=} \varnothing$, no prerequisite information is provided for $X_i$.

\noindent\textbf{Evaluation Protocol.}
We report the capability metrics (NC, IC, RC) for each unit task, and the contribution metrics (N-Score, S-Score) on the overall task across four state contrasts: $(1, 0)$, $(1, \varnothing)$, $(\varnothing, 0)$, and $(0, \varnothing)$. 
Model predictions are evaluated via exact match for structured outputs, with Gemini 3.6 Flash serving as LLM judger~\citep{zheng2023judging} for free-form text tasks and format mismatches. 
Full experimental settings are provided in Appendix.

%%%%%%%%%%%%%%%%%%%%%%%%%%%%%%%%%%%%%%%%%%%%%%%%%%%%%%%%%%%%%%%%%%%%%%%%
\subsection{Capability Diagnostics}
\label{subsec:exp_cap}

\begin{figure}[t]
    \centering
    \includegraphics[width=\linewidth]{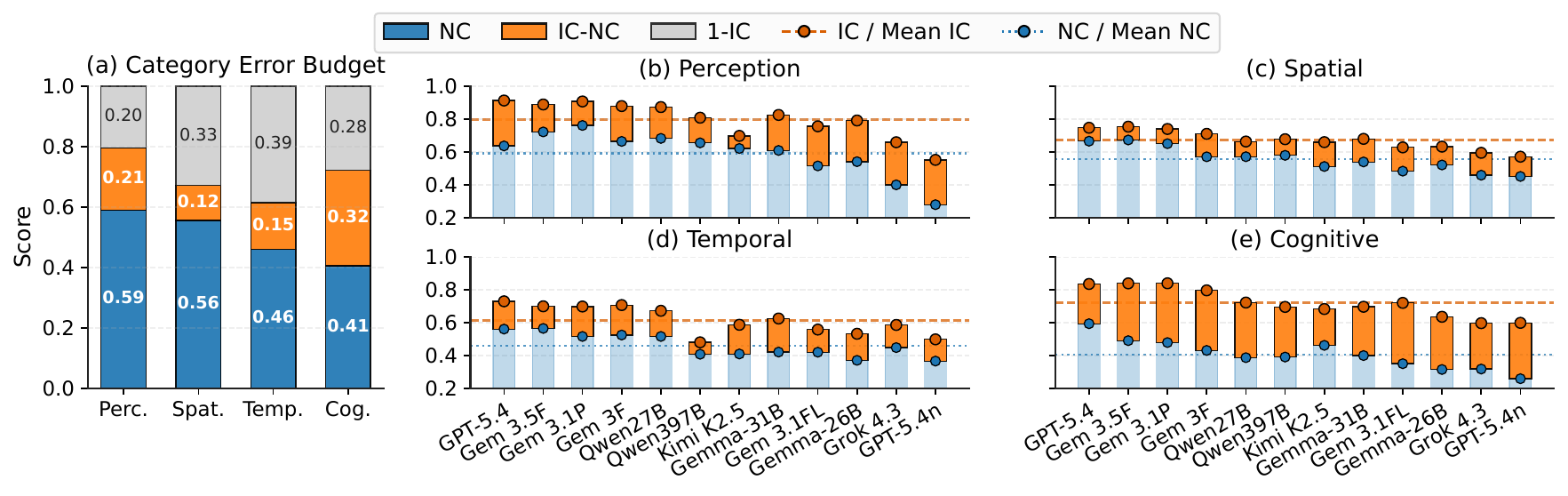}
    \caption{Natural(NC) and intrinsic(IC) capabilities on non-root unit tasks: \textit{(a)} category-level NC and IC averaged over 12 models; \textit{(b--e)} per-model NC and IC within each category.\nop{\acg{acg: if you have space please lower the subplot labels a bit, e.g., (d) Temporal is almost touching graph (b). - xia: I know lol i adjusted this to save space}}}
    \label{fig:finding_cap_combined_design3}
% \end{figure}
% \begin{figure}[t]
%     \centering
    \begin{minipage}[t]{0.57\linewidth}
        \vspace{0pt}%
        \centering
        \includegraphics[width=0.9\linewidth]{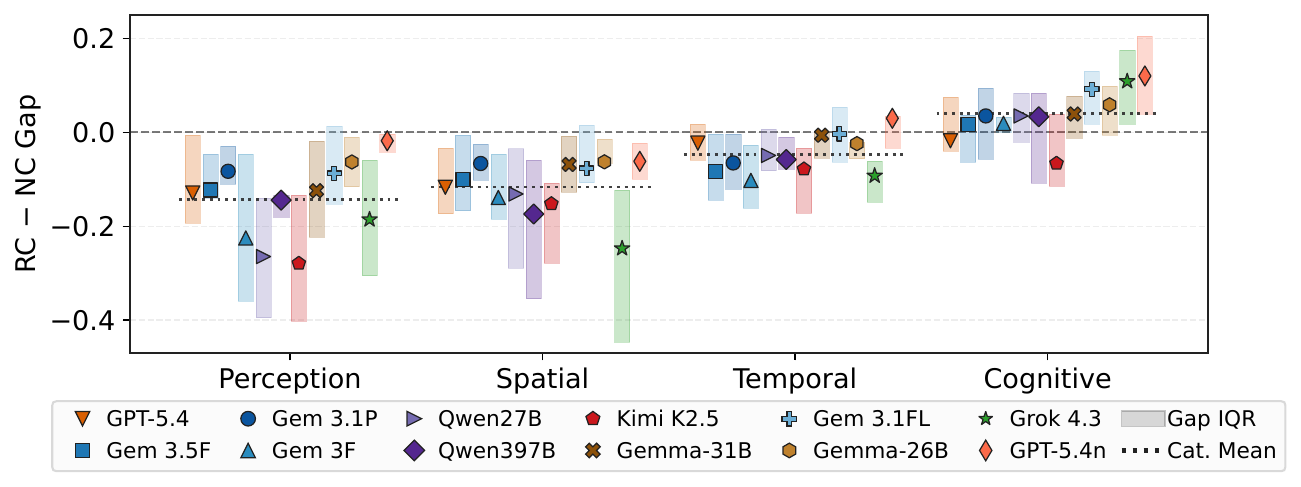} 
        \caption{Per-model $\text{RC}{-}\text{NC}$ gap across categories on non-root unit tasks.}
        \label{fig:finding_gap_rc_nc}
    \end{minipage}
    \hspace{0.02\linewidth}%
    \begin{minipage}[t]{0.38\linewidth}
        \vspace{0pt}%
        \centering
        \includegraphics[width=0.9\linewidth]{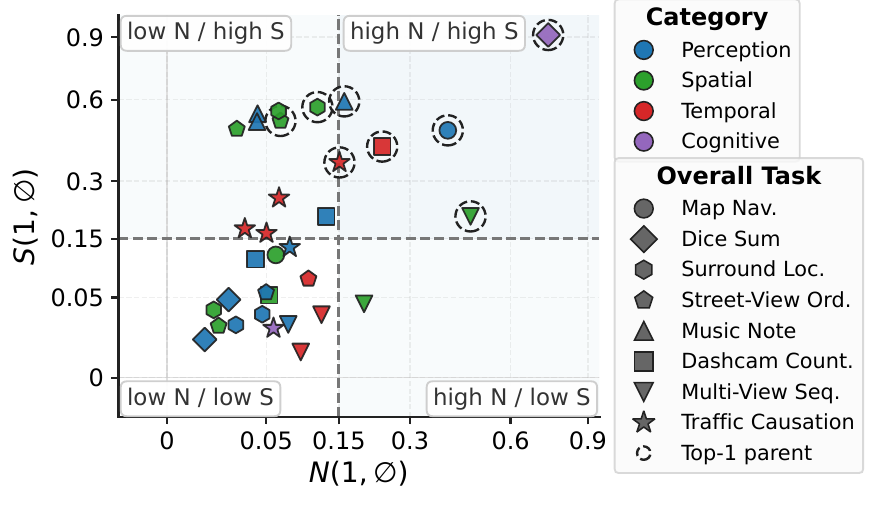}
        \caption{Prerequisite causal contributions $N$/$S(1,\varnothing)$ to overall task $Y$.}
        \label{fig:finding_contribution_quadrant_1_empty}
    \end{minipage}
\end{figure}

\textbf{Cascading errors reorder the apparent capability ranking.} 
Under end-to-end evaluation (NC), accuracy falls across perception, spatial, temporal, and cognitive tasks~(0.59, 0.56, 0.46, 0.41; Fig~\ref{fig:finding_cap_combined_design3}(a)), 
a pattern that would attribute the largest deficits to high-level reasoning. 
Supplying correct prerequisites (IC) reorders the categories: cognitive reasoning moves to second (0.72), above spatial (0.67) and temporal (0.61), behind only perception (0.80). 
Decomposing the total error into a cascading component $\mathrm{IC}{-}\mathrm{NC}$ and a residual component $1{-}\mathrm{IC}$ accounts for this shift, 
as cascading errors constitute 54\% of cognitive failures (0.32 of 0.59).
% and 2.7$\times$ the spatial cascading component (0.12). 
Spatial and temporal reasoning show the opposite composition, with residual errors~(0.33, 0.39) covering 75\% and 72\% of their failures and small gains from correct prerequisites~(0.12, 0.15), locating the more persistent intrinsic capability gaps at these two categories rather than at cognition.

\noindent\textbf{Frontier models share similar intrinsic ceilings but differ in cascading errors}~(Fig~\ref{fig:finding_cap_combined_design3}(b-e)).
Supplying correct prerequisites (IC) brings the top 4 SOTA models to nearly identical ceilings across categories (0.88-0.91 in perception, 0.80-0.83 in cognition), tracing their end-to-end (NC) gaps to cascading errors ($\mathrm{IC}{-}\mathrm{NC}$): 
GPT-5.4 and Gemini family trade places, with larger cascading loss for 
GPT-5.4 in perception~(0.28 vs. 0.15--0.22) and for Gemini in cognition~(0.35--0.37 vs. 0.24).
Across all 12 models, perception exhibits the widest spread in residual errors ($1{-}\mathrm{IC}$: 0.09--0.45), spatial shows uniformly high residual errors (0.25--0.43), and cognitive carries uniformly large cascading errors ($\mathrm{IC}{-}\mathrm{NC}$: 0.22--0.37).
Model-wise, close NC scores mask divergent error compositions: Qwen3.5 397B and Kimi K2.5 match neighboring models in NC yet carry substantially higher residual errors ($1{-}\mathrm{IC}$) in different categories.

\noindent\textbf{Incorrect prerequisites harm perception and spatial but mildly aid cognition.}
The $\text{RC}{-}\text{NC}$ gap rises across the hierarchy~(Fig~\ref{fig:finding_gap_rc_nc}):
% (P$-0.14$, S$-0.12$, T$-0.05$, C$+0.04$; 
all 12 models fall below zero in perception and spatial, while 10 of 12 turn positive in cognition. 
The gap also varies widely across models within each category ($-0.28$ to $+0.12$).
The positive cognitive gaps directly instantiate the non-monotonicity formalized in Property~\ref{assum:non_monotonicity}: incorrect prerequisites do not uniformly degrade performance.

% \paragraph{Finding 4: Non-monotonicity --- incorrect context (RC) produces scaffolding for Cognitive (RC $\geq$ NC) but toxicity for Perception/Spatial (RC $\ll$ NC).}
% Cognitive: most models RC $\geq$ NC (wrong answers still provide structural reasoning cues); Perception/Spatial: all models RC $<$ NC (wrong facts directly poison).
% Consistent across 12 models, clusters by category not SCM complexity; Kimi K2.5 anomalous.

%%%%%%%%%%%%%%%%%%%%%%%%%%%%%%%%%%%%%%%%%%%%%%%%%%%%%%%%%%%%%%%%%%%%%%%%
\subsection{Causal Contribution Analysis}
\label{sec:exp:contribution}

% TODO: Write introductory text.
% Core scores: N(1,0), N(1,empty), S(0,empty)
% N(1,0): necessity of correctness
% N(1,empty): value of providing correct information
% S(0,empty): benefit of wrong information (non-monotonicity diagnostic)

\textbf{Causal contributions exhibit effective bottleneck sparsity.}
Evaluating prerequisite contributions via $N(1, \varnothing)$ and $S(1, \varnothing)$ (Fig~\ref{fig:finding_contribution_quadrant_1_empty}) across 8 non-Markov tasks shows that causal influence is highly uneven across parents, with each task driven by one or a few critical prerequisites.
These critical parents act either as core bottlenecks that cause large oracle drops when omitted and strong natural recovery when provided (high $N$/high $S$), or as compensable levers whose omission is offset by other oracle prerequisites yet still rescue natural failures alone (low $N$/high $S$).
By contrast, 
the remaining prerequisites have individually small effects 
% nearly half of the prerequisites have minimal standalone impact 
(low $N$/low $S$), while conjunctive bottlenecks that hurt oracle performance without aiding natural recovery are nearly absent (high $N$/low $S$).
% \paragraph{Finding-1: on model-wise everage, in figure4, all task have a critical parent, in either high N high S; high N low S; low N high S. } notably, high N measures the sentivitity in oracle env; high S measures the fliping success in natural env. 

\begin{figure}[t]
    \centering
    \includegraphics[width=0.9\linewidth]{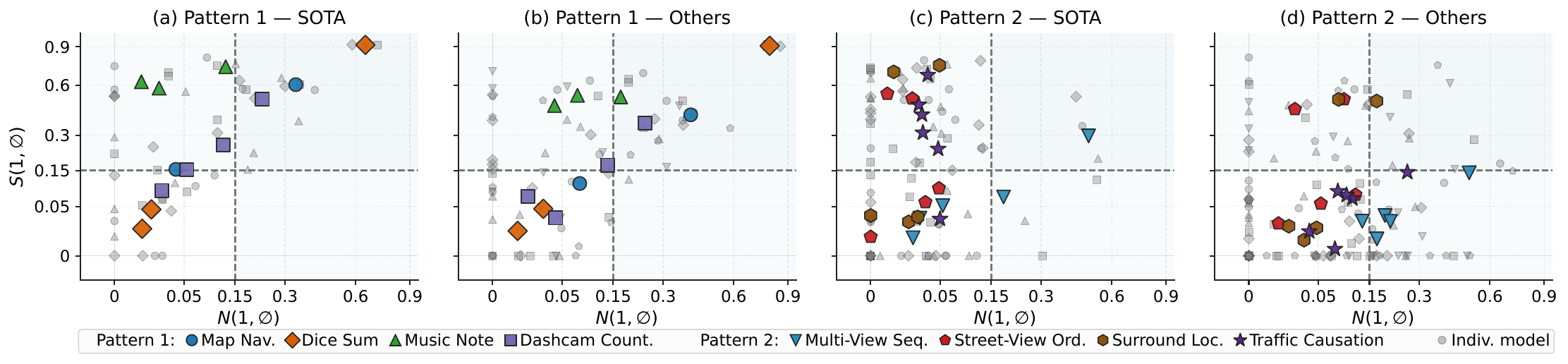} 
    \caption{$N/S(1,\varnothing)$ across Pattern~1(\textit{a,b}) and Pattern~2 (\textit{c,d}) tasks for top-4 SOTA (\textit{a,c}) and remaining 8 Others (\textit{b,d}) models ($sqrt({\cdot})$-scaled axes, similar in Fig~\ref{fig:finding_contribution_quadrant_1_empty}).}
    \label{fig:finding_contrib_4panel_quadrants}
\end{figure}

\noindent\textbf{Disaggregating by model reveals two distinct criticality patterns.}
Expanding to individual models (Fig~\ref{fig:finding_contrib_4panel_quadrants}) reveals two task patterns and a broad split between top-4 (SOTA) and remaining (Others) models.
Under Pattern 1 (Fig~\ref{fig:finding_contrib_4panel_quadrants}a,b), both groups exhibit critical prerequisites with high single-point recovery capacity (high $S$), with Others shifting moderately toward higher $N$ and lower $S$.
Under Pattern 2 (Fig~\ref{fig:finding_contrib_4panel_quadrants}c,d), SOTA keeps prerequisites at high $S$ with low oracle omission sensitivity (low $N$), whereas Others shift substantially toward higher $N$ and lower $S$, losing effective single-point recoverability while suffering larger oracle drops from single-task omission.
Across both patterns, SOTA maintains low oracle vulnerability and often has multiple high-$S$ prerequisites per task that each could rescue natural failures alone;
while Others drift rightward and downward, moderately in Pattern 1 and severely in Pattern 2, reflecting higher oracle sensitivity and weaker single-task recovery.
Structurally, the prerequisite count $|\mathrm{Pa}(Y)|$ averages 3.0 in Pattern 1 and 5.25 in Pattern 2.
% 
% \paragraph{Finding-2: model x task. in figure 5}
% pattern 1: on 4 tasks, all models show same patterns (either high N high S or low N high S). demonstrating the same critical parent unit task; 
% pattern2: others v.s. 4 sota, higher N and lower S. 
% pattern 1: avg k < pattern 2 avg k 
% :: maybe an ablation study exp: group redundant parents, measure and report 

\begin{figure}[t]
    \centering
        \begin{minipage}[t]{0.34\linewidth}
        \vspace{0pt}%
        \centering
        \includegraphics[width=0.95\linewidth]{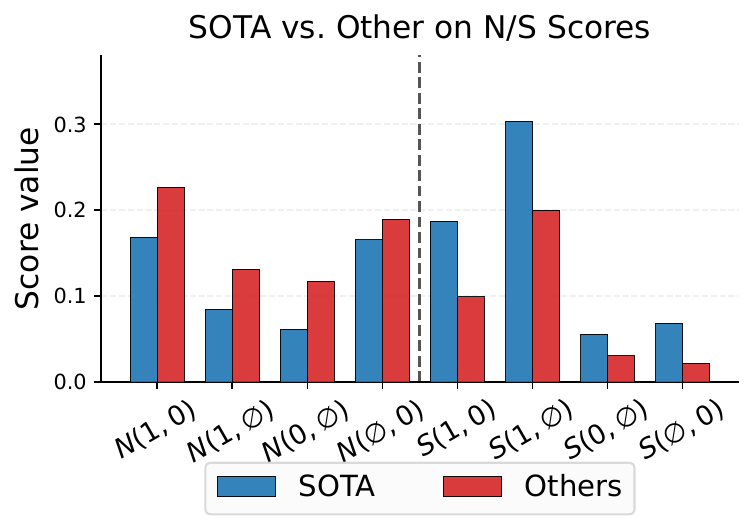}
        \caption{$N$/$S$-Scores across the four state contrasts for SOTA vs.\ Others models.}
        \label{fig:finding_contrib_ns_sota_vs_others}
    \end{minipage}
    \hfill
    \begin{minipage}[t]{0.65\linewidth}
        \vspace{0pt}
        \centering
        \includegraphics[width=0.9\linewidth]{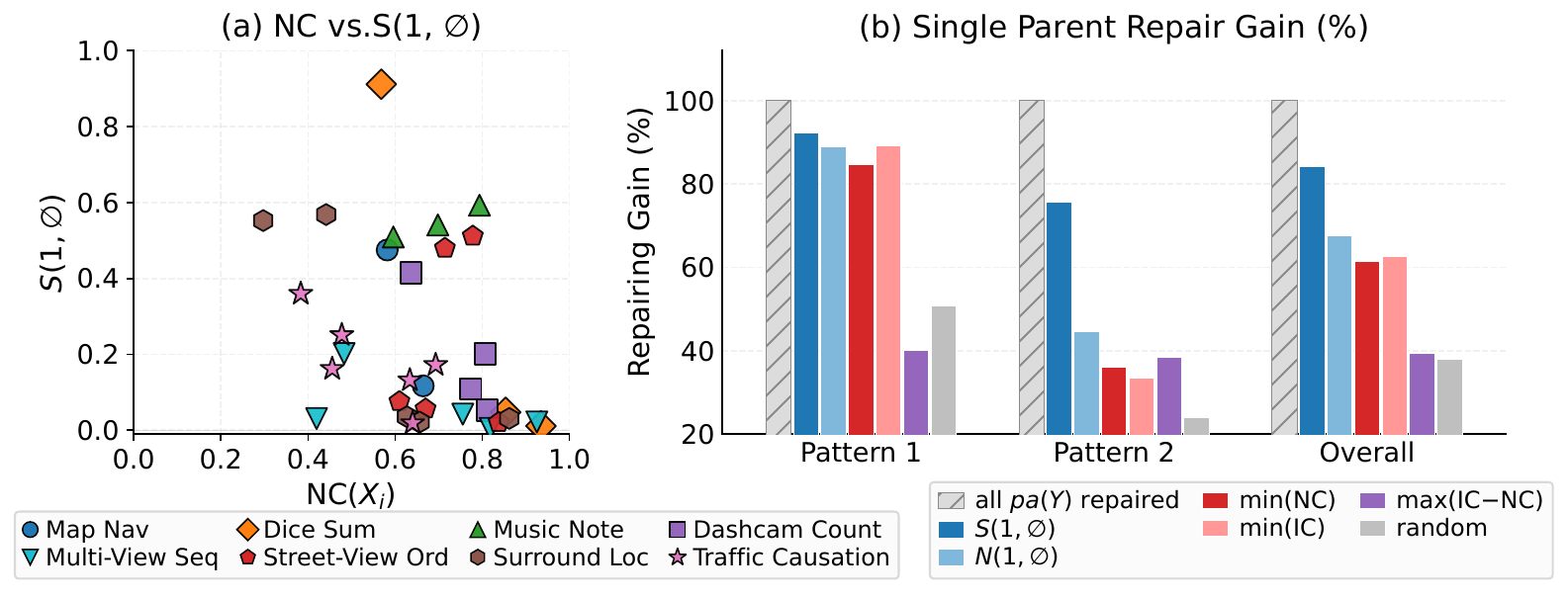}
        \caption{\textit{(a)} $\text{NC}(X_i)$ vs.\ $S(1,\varnothing)$; \textit{(b)} percentage of full all-prerequisite gain($\text{IC}{-}\text{NC}$) recovered by single-parent supply across selection criteria.}
        \label{fig:finding_contrib_repair_gain_s1empty}
    \end{minipage}
    % \hfill
\end{figure}
\begin{figure}[t]
    \centering
    \begin{minipage}[t]{0.28\linewidth}
        \vspace{0pt}%
        \centering
        \includegraphics[width=0.85\linewidth]{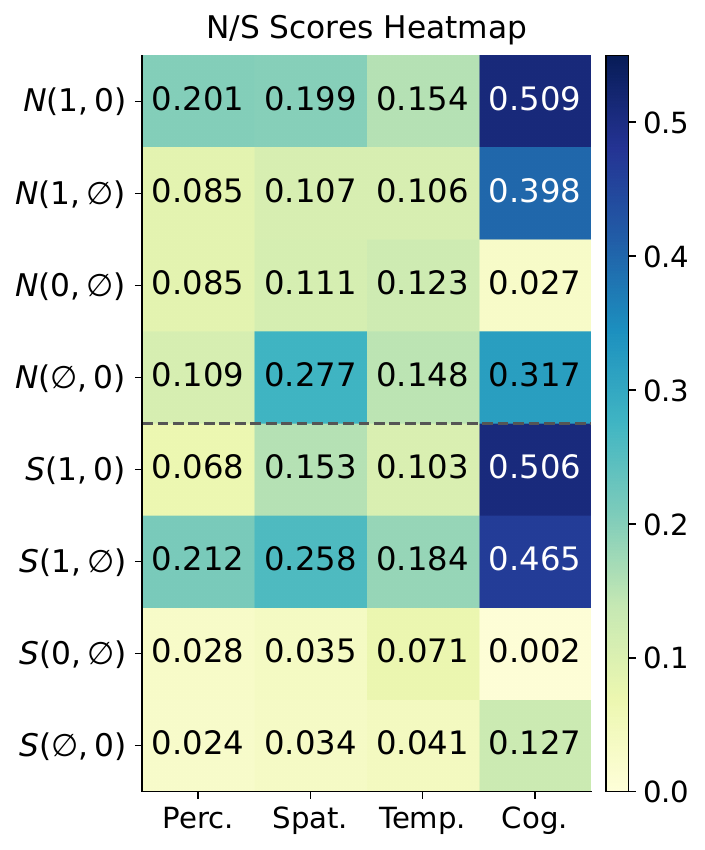}
        \caption{Category-level prerequisite $N$/$S$-Scores across 4 state contrasts.}
        \label{fig:finding_contrib_ns_heatmap}
    \end{minipage}
    \hspace{0.01\linewidth}%
    % Left figure
    \begin{minipage}[t]{0.7\linewidth}
        \vspace{0pt}%
        \centering
        \includegraphics[width=\linewidth]{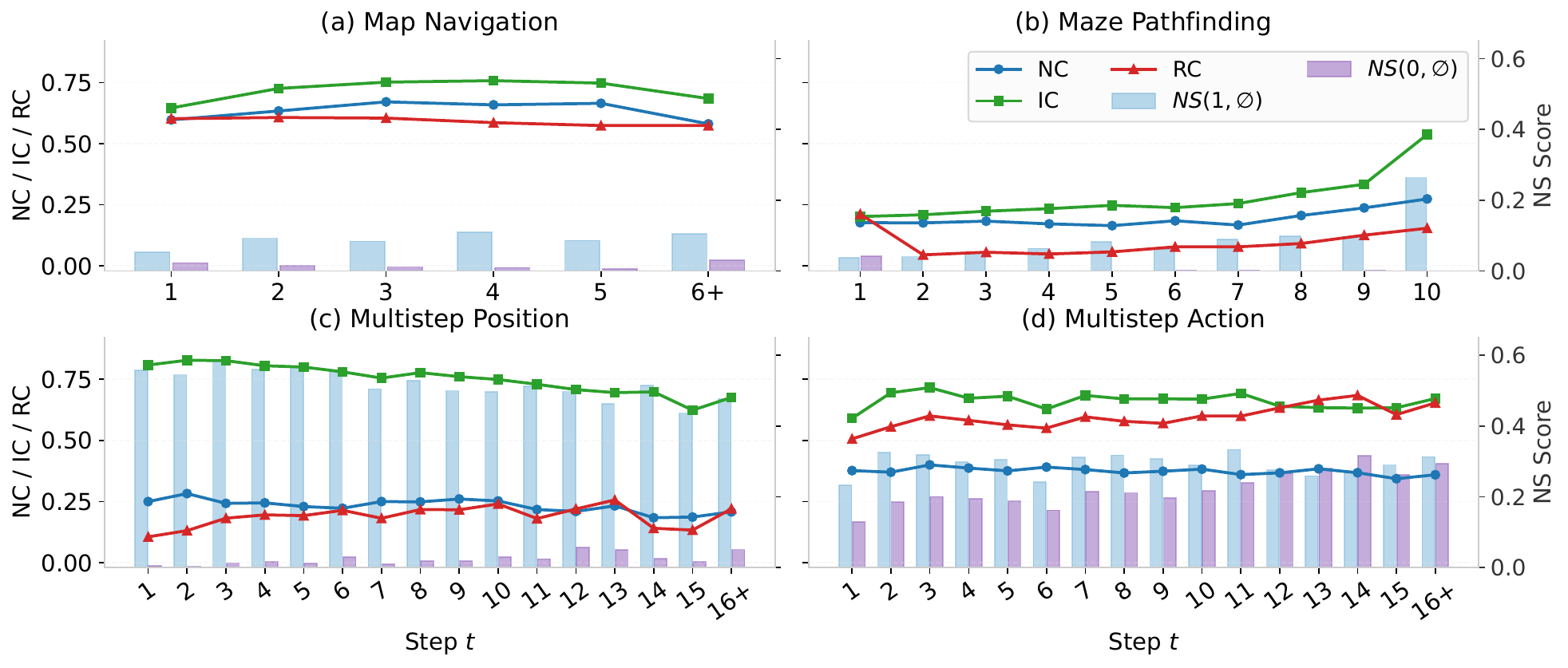}
        \caption{Step-wise capabilities (NC, IC, RC) and causal contributions ($NS(1,\varnothing)$, $NS(0,\varnothing)$) along step $t$ across Markov chain unit tasks.}
        \label{fig:finding_markov_causal_profiles}
    \end{minipage}
\end{figure}

\noindent\textbf{A single supplied prerequisites recovers most full-prerequisite gains.}
Supplying only the single most sufficient prerequisite ($\max S(1, \varnothing)$) recovers 84\% of the full all-prerequisite gain ($\mathrm{IC} {-} \mathrm{NC}$) overall (Fig~\ref{fig:finding_contrib_repair_gain_s1empty}b), showing that one targeted intervention already accounts for most of the attainable gain.
This leverage is not readable from standalone accuracy: $S(1, \varnothing)$ does not track $\mathrm{NC}(X_i)$ across prerequisites (Fig~\ref{fig:finding_contrib_repair_gain_s1empty}a).
To test whether cheaper node-level signals could substitute for $S(1, \varnothing)$ in identify the same prerequisite, 
we repeat single-parent repair under alternative criteria (including the weakest prerequisite $\min(\mathrm{NC})$; Fig~\ref{fig:finding_contrib_repair_gain_s1empty}b), none of which recovers this single-parent optimum overall.
The shortfall of $\mathrm{NC}$ stays around 8 pp in Pattern~1 tasks, where models share a single dominant bottleneck;
but widens sharply in Pattern~2 tasks with denser multi-parent dependencies, where $S(1, \varnothing)$-chosen parent sustains 76\% of the full gain while $\mathrm{NC}$-based only 36\%.
% 
% \paragraph{Finding-3: which unit task to repair? }
% figure 6. repair gain of since unit task (max S1,e) obtains more than $80\%$ of fix all 
% NC based repair: 65\% v.s. 86\%

% \begin{figure}[t]
%     \centering
%     \includegraphics[width=\linewidth]{figures/fig_contrib_finding4_ternary_profiles.pdf}
%     \caption{Caption}
%     \label{fig:contrib_finding4_ternary_profiles}
% \end{figure}

\noindent\textbf{Incorrect prerequisites hurt far more than omission.}
Extending to full ternary contrast space (Fig~\ref{fig:finding_contrib_ns_sota_vs_others},\ref{fig:finding_contrib_ns_heatmap}), injecting incorrect prerequisites is consistently more destructive than leaving them unassisted ($N(1,0) {>} N(1,\varnothing)$ across all categories), with minor gains from incorrect answers on temporal~($S(0,\varnothing){=}0.071$).
Unlike cognitive target tasks benefiting from incorrect prerequisites (Sec~\ref{subsec:exp_cap}), the two cognitive prerequisites are highly sensitive to correctness: corrupting or omitting them causes the steepest overall task drops ($N(1,0), N(1,\varnothing)$) with no gain from incorrect answers ($S(0,\varnothing)$), while fixing or removing them drives strong recovery ($S(1,0), S(\varnothing,0)$). 
SOTA models~(Fig~\ref{fig:finding_contrib_ns_sota_vs_others}) pair uniformly lower $N$ with higher $S$ than Others across all four contrasts, 
% nearly doubling recovery in all-wrong contexts ($S(1,0)$), 
showing stronger robustness to single-prerequisite degradation with superior single-point recovery even under corruption.

\subsection{Markov Chain Unit Task Analysis}
\label{subsec:exp_markov}

We analyze the Markov chain unit tasks~(Sec~\ref{sec:benchmark}),
in which each $X_t$ depends solely on its predecessor $X_{t-1}$.
Under this single-parent structure, N, S-Score converge into
a single NS-Score~(${NS}(a,a'){=}{N}{(a,a')}{\cdot}\mathrm{Cap}(X_t;a)$; Appendix).
Fig~\ref{fig:finding_markov_causal_profiles} reports
$\mathrm{NC}$, $\mathrm{IC}$, $\mathrm{RC}$ and
% ${NS}(1,\varnothing)$, ${NS}(0,\varnothing)$
$NS$ score along step~$t$.
% , revealing three findings.
%
First, capability degrades nonlinearly with a single-step effective horizon: 
in goal-directed tasks, $\mathrm{IC}$ and ${NS}(1, \varnothing)$ remain flat until one step from the exit (Fig~\ref{fig:finding_markov_causal_profiles}b, $t{=}10$); 
symmetrically, in process-cumulative tasks where $\mathrm{NC}$ collapses to a flat floor immediately at $t{=}1$, $\mathrm{IC}$ and ${NS}(1, \varnothing)$ peak at the first step (Fig~\ref{fig:finding_markov_causal_profiles}c) and gradually decline with depth.
Second, alternating Position (Fig~\ref{fig:finding_markov_causal_profiles}c) and Action (Fig~\ref{fig:finding_markov_causal_profiles}d) nodes within the same chain respond divergently to incorrect prerequisites: 
${NS}(0, \varnothing) {\approx} 0$ for spatial tracking versus substantial for action planning, consistent with the finding in Sec~\ref{subsec:exp_cap} that incorrect prerequisites benefit cognitive reasoning but not perception.
Third, as depth $t$ grows, Position continues to require factual correctness (${NS}(0, \varnothing) \le 0.06 \ll \text{NS}(1, \varnothing)$), whereas in cognitive Action, ${NS}(0, \varnothing)$ steadily rises from $0.12$ ($t{=}1$) to $0.30$ ($t{\ge}14$) to match $\text{NS}(1, \varnothing)$ ($\text{RC} {\approx} \text{IC}$), showing that explicit prerequisite conditioning, regardless of correctness, increasingly accounts for cognitive recovery at deeper horizons.

\section{Conclusion}
\label{sec:conclusion}

We introduced a causal decomposition framework and the CADET benchmark, applying controlled ternary interventions over explicit prerequisite structures to diagnose where and why MLLMs fail. 
Our capability metrics (NC, IC, RC) separate cascading influence from intrinsic deficits, showing that supplying correct prerequisites substantially lifts cognitive tasks while spatial and temporal reasoning retain persistent intrinsic gaps, and that incorrect prerequisites often still aid cognition. 
The $N$-, $S$-Scores reveal effective bottleneck sparsity, where a few critical prerequisites drive each task, and supplying the single most sufficient one recovers most full-prerequisite gain yet rarely coincides with the least accurate prerequisite. 
Across models, top ones exhibit lower degradation alongside stronger single-point recovery. 
More broadly, our diagnostic framework can inform future improvements, from inference-time CoT guidance, query rewriting to training and multi-step system design.

% \section*{Acknowledgments}
% ...

\bibliography{iclr2027_conference} % consider renaming the .bib to reference.bib

\newpage
\appendix
\clearpage

\section{Details and Proofs of Causal Decomposition Framework}

% more general definition of n-score, s-score 

In this section, we provide the general definitions of N-Score and S-Score under arbitrary prerequisite states, followed by the theoretical proofs and properties of our causal decomposition framework.
% \todo{add a sentence pointing to classic work, and add citations}

\textbf{The do-operator.}
We adopt the standard do-operator~(Section~\ref{subsec:intervention}; \citet{pearl2009causality}). For $a \in \{1, 0\}$,
$do(A_i {=} a)$ replaces the structural assignment of $X_i$ in \eqref{eq:scm} with
$X_i := a$ while keeping all other assignments unchanged, and $do(A_i {=} \varnothing)$ leaves
the assignment of $X_i$ intact; interventions on multiple unit tasks apply these operations
jointly. $P(\cdot \mid do(\cdot))$ denotes the distribution induced by the resulting SCM.

\subsection{General Definition of N-Score, S-Score}

In Section~\ref{subsec:ns_score} (Definitions~\ref{def:nscore} and~\ref{def:sscore}), when quantifying the causal contribution of a prerequisite unit task $X_i \in \mathrm{Pa}(Y)$ to the target outcome $Y$ across an ordered contrast $(a, a')$, we specifically set the remaining prerequisites $\mathbf{A}_{-i} := \mathbf{A}_{\mathrm{Pa}(Y) \setminus \{X_i\}}$ to the homogeneous state $a$ for N-Score and $a'$ for S-Score. Here, we generalize this formulation by allowing the remaining prerequisites to take any arbitrary intervention state configuration. 

Specifically, let $\mathbf{c} \in \mathcal{A}^{|\mathrm{Pa}(Y) \setminus \{X_i\}|}$ denote an arbitrary intervention state configuration assigned to $\mathbf{A}_{-i}$, denoted as $do(\mathbf{A}_{-i} = \mathbf{c})$, which sets the intervention state of each prerequisite $X_j \in \mathrm{Pa}(Y) {\setminus} \{X_i\}$ to $A_j = c_j$ with $c_j \in \mathcal{A}$. Following the notation in Section~3.2, when all unit tasks in $\mathbf{A}_{-i}$ are placed in a single homogeneous state $c \in \mathcal{A}$ (i.e., $c_j = c$ for all $X_j \in \mathrm{Pa}(Y) {\setminus} \{X_i\}$), we write $do(\mathbf{A}_{-i} = c)$ without boldface.
Based on this general state configuration $\mathbf{c}$ for the remaining prerequisites, we define the general N-Score and S-Score as follows.
\begin{definition}[General N-Score]
\label{def:general_n_score}
Let $Y$ be a target outcome task with prerequisite dependencies $\mathrm{Pa}(Y)$, let $X_i \in \mathrm{Pa}(Y)$, and let $(a, a')$ be an ordered pair of distinct states in $\mathcal{A}$. For any intervention state configuration $\mathbf{c} \in \mathcal{A}^{|\mathrm{Pa}(Y) \setminus \{X_i\}|}$ assigned to the remaining prerequisites $\mathbf{A}_{-i}$, the general N-Score of $X_i$ on $Y$ is defined as:
\begin{equation}
\small
\label{eq:general_n_score}
    \mathrm{N}_{(a,a')}(X_i, Y \mid \mathbf{c}) := \max \left\{ 0, \, \frac{P\big(Y{=}1 \mid do(A_i{=}a, \mathbf{A}_{-i}{=}\mathbf{c})\big) - P\big(Y{=}1 \mid do(A_i{=}a', \mathbf{A}_{-i}{=}\mathbf{c})\big)}{P\big(Y{=}1 \mid do(A_i{=}a, \mathbf{A}_{-i}{=}\mathbf{c})\big)} \right\},
\end{equation}
measuring the fraction of $Y$'s success under $do(A_i{=}a, \mathbf{A}_{-i}{=}\mathbf{c})$ that is lost when $X_i$ alone is degraded from $a$ to $a'$.
\end{definition}

\begin{definition}[General S-Score]
\label{def:general_s_score}
Under the setting of Definition~\ref{def:general_n_score}, for any intervention state configuration $\mathbf{c} \in \mathcal{A}^{|\mathrm{Pa}(Y) \setminus \{X_i\}|}$ assigned to $\mathbf{A}_{-i}$, the general S-Score of $X_i$ on $Y$ is defined as:
% {\small
\begin{equation}
\small
\label{eq:general_s_score}
    \mathrm{S}_{(a,a')}(X_i, Y \mid \mathbf{c}) := \max \left\{ 0, \, \frac{P\big(Y{=}1 \mid do(A_i{=}a, \mathbf{A}_{-i}{=}\mathbf{c})\big) - P\big(Y{=}1 \mid do(A_i{=}a', \mathbf{A}_{-i}{=}\mathbf{c})\big)}{1 - P\big(Y{=}1 \mid do(A_i{=}a', \mathbf{A}_{-i}{=}\mathbf{c})\big)} \right\},
\end{equation}
% }
measuring the fraction of $Y$'s failure under $do(A_i{=}a', \mathbf{A}_{-i}{=}\mathbf{c})$ that is recovered when $X_i$ alone is upgraded from $a'$ to $a$.
\end{definition}
Under these general definitions, the N-Score (Definition~3.7) and S-Score (Definition~3.8) introduced in Section~3.4 correspond to the special cases where every element of $\mathbf{c}$ is set to the homogeneous state $c_j = a$ (i.e., $do(\mathbf{A}_{-i} = a)$) for N-Score and $c_j = a'$ (i.e., $do(\mathbf{A}_{-i} = a')$) for S-Score, respectively.

\subsection{Lower Bound Guarantee of N-Score and S-Score}
\label{app:lower_bound_proofs}

This section provides the formal statement and proof of Theorem~\ref{thm:pn_ps_bounds} in Section~\ref{subsec:ns_score}. We establish the lower-bound guarantee for the general N-Score and S-Score (Definitions~\ref{def:general_n_score} and~\ref{def:general_s_score}) under an arbitrary prerequisite configuration $\mathbf{c}$ (Theorems~\ref{thm:general_n_score_bound} and~\ref{thm:general_s_score_bound}), where Theorem~\ref{thm:pn_ps_bounds} corresponds to the special cases $\mathbf{c}=a$ for N-Score and $\mathbf{c}=a'$ for S-Score.

\begin{theorem}
\label{thm:general_n_score_bound}
Under Assumption~\ref{assum:interv} (Exogeneity), for any prerequisite unit task $X_i \in \mathrm{Pa}(Y)$, any ordered contrast $(a, a')$ of distinct states in $\mathcal{A}$, and any intervention state configuration $\mathbf{c} \in \mathcal{A}^{|\mathrm{Pa}(Y)\setminus\{X_i\}|}$ assigned to the remaining prerequisites $\mathbf{A}_{-i}$, the general N-Score (Definition~\ref{def:general_n_score}) provides a valid lower bound on the Probability of Necessity ($\mathrm{PN}$) under prerequisite context $do(\mathbf{A}_{-i}{=}\mathbf{c})$:
\begin{equation}
\small
    \mathrm{PN}_{(a,a')}(X_i, Y \mid \mathbf{c}) \ge \mathrm{N}_{(a,a')}(X_i, Y \mid \mathbf{c})
\end{equation}
In particular, setting every element of $\mathbf{c}$ to the homogeneous state $a$ (i.e., $do(\mathbf{A}_{-i}{=}a)$) directly establishes the N-Score lower bound in Theorem~\ref{thm:pn_ps_bounds}.
\end{theorem}

\begin{proof}
Note that under exogeneity and non-monotonicity, the mathematical technique of bounding probabilities of causation via Fr\'{e}chet inequalities~\citep{frechet1935generalisation} is established in the foundational works of \citet{pearl1999probabilities,tian2000probabilities,pearl2009causality,balke1994counterfactual}. Here, we adapt and instantiate their derivation step-by-step to demonstrate how it applies to our multi-parent LLM evaluation setting under prerequisite configuration $do(\mathbf{A}_{-i}{=}\mathbf{c})$, Assumption~\ref{assum:interv}, and Property~\ref{assum:non_monotonicity}.

For notational conciseness under the prerequisite intervention $do(\mathbf{A}_{-i}{=}\mathbf{c})$, in below proof, let $P_{\mathbf{c}}(\cdot) := P(\cdot \mid do(\mathbf{A}_{-i}{=}\mathbf{c}))$ denote the probability measure in the induced submodel $\mathcal{M}_{\mathbf{c}}$ \citep{pearl2009causality}, where $Y_a$ and $Y_{a'}$ denote the potential outcomes of $Y$ under further regime interventions $do(A_i{=}a)$ and $do(A_i{=}a')$ within $\mathcal{M}_{\mathbf{c}}$. Following \citet{pearl1999probabilities}, the Probability of Necessity ($\mathrm{PN}$) across contrast $(a, a')$ under context $\mathbf{c}$ is defined as the counterfactual probability that $Y$ fails under $do(A_i{=}a')$, given that $Y$ succeeded under regime $A_i{=}a$:
\begin{equation*}
\small
    \mathrm{PN}_{(a,a')}(X_i, Y \mid \mathbf{c}) := P_{\mathbf{c}}(Y_{a'}{=}0 \mid A_i{=}a,\, Y{=}1)
\end{equation*}
If $P_{\mathbf{c}}(A_i{=}a,\, Y{=}1) = 0$, the success baseline is zero and the bound holds trivially by convention ($\mathrm{N}_{(a,a')} = 0$). Thus, assuming $P_{\mathbf{c}}(A_i{=}a,\, Y{=}1) > 0$, we first apply the definition of conditional probability:
\begin{equation*}
\small
    P_{\mathbf{c}}(Y_{a'}{=}0 \mid A_i{=}a,\, Y{=}1) = \frac{P_{\mathbf{c}}(Y_{a'}{=}0,\, A_i{=}a,\, Y{=}1)}{P_{\mathbf{c}}(A_i{=}a,\, Y{=}1)}
\end{equation*}
By the consistency property of structural causal models, the observed outcome under regime $A_i{=}a$ coincides with the potential outcome $Y_a$. Substituting $Y = Y_a$ into the numerator yields:
\begin{equation*}
\small
    P_{\mathbf{c}}(Y_{a'}{=}0 \mid A_i{=}a,\, Y{=}1) = \frac{P_{\mathbf{c}}(Y_{a'}{=}0,\, A_i{=}a,\, Y_a{=}1)}{P_{\mathbf{c}}(A_i{=}a,\, Y{=}1)}
\end{equation*}
Next, we factor out the marginal regime probability $P_{\mathbf{c}}(A_i{=}a)$ in the numerator and divide both the numerator and the denominator by it:
\begin{equation*}
\small
\begin{split}
    P_{\mathbf{c}}(Y_{a'}{=}0 \mid A_i{=}a,\, Y{=}1) &= \frac{P_{\mathbf{c}}(Y_{a'}{=}0,\, Y_a{=}1 \mid A_i{=}a) \, P_{\mathbf{c}}(A_i{=}a)}{P_{\mathbf{c}}(A_i{=}a,\, Y{=}1)} \\
    &= \frac{P_{\mathbf{c}}(Y_{a'}{=}0,\, Y_a{=}1 \mid A_i{=}a)}{P_{\mathbf{c}}(Y{=}1 \mid A_i{=}a)}
\end{split}
\end{equation*}
Under the exogeneity of the regime assignment (Assumption~\ref{assum:interv}), the potential outcomes $(Y_a, Y_{a'})$ are independent of $A_i$, and conditioning on $A_i$ is operationally identical to intervening on it. This removes the conditioning event in both the numerator and the denominator:
\begin{equation*}
\small
    P_{\mathbf{c}}(Y_{a'}{=}0 \mid A_i{=}a,\, Y{=}1) = \frac{P_{\mathbf{c}}(Y_{a'}{=}0,\, Y_a{=}1)}{P_{\mathbf{c}}(Y_a{=}1)}
\end{equation*}
Since we do not assume monotonicity (Property~\ref{assum:non_monotonicity}) in LLM reasoning, the joint counterfactual probability $P_{\mathbf{c}}(Y_{a'}{=}0,\, Y_a{=}1)$ cannot be point-identified from marginal probabilities alone. Applying the classical Fr\'{e}chet lower bound $P_{\mathbf{c}}(E_1, E_2) \ge \max\{0, \, P_{\mathbf{c}}(E_1) + P_{\mathbf{c}}(E_2) - 1\}$ \citep{pearl1999probabilities} to the numerator gives:
\begin{equation*}
\small
    \mathrm{PN}_{(a,a')}(X_i, Y \mid \mathbf{c}) \ge \max \left\{ 0, \, \frac{P_{\mathbf{c}}(Y_{a'}{=}0) + P_{\mathbf{c}}(Y_a{=}1) - 1}{P_{\mathbf{c}}(Y_a{=}1)} \right\}
\end{equation*}
Using the complement rule $P_{\mathbf{c}}(Y_{a'}{=}0) = 1 - P_{\mathbf{c}}(Y_{a'}{=}1)$ gives:
\begin{equation*}
\small
    \mathrm{PN}_{(a,a')}(X_i, Y \mid \mathbf{c}) \ge \max \left\{ 0, \, \frac{P_{\mathbf{c}}(Y_a{=}1) - P_{\mathbf{c}}(Y_{a'}{=}1)}{P_{\mathbf{c}}(Y_a{=}1)} \right\}
\end{equation*}
Finally, unpacking the submodel notation back into full interventional probabilities via $P_{\mathbf{c}}(Y_a{=}1) = P(Y{=}1 \mid do(A_i{=}a,\, \mathbf{A}_{-i}{=}\mathbf{c}))$ and $P_{\mathbf{c}}(Y_{a'}{=}1) = P(Y{=}1 \mid do(A_i{=}a',\, \mathbf{A}_{-i}{=}\mathbf{c}))$ yields:
\begin{equation*}
\small
\begin{aligned}
    \mathrm{PN}_{(a,a')}(X_i, Y \mid \mathbf{c})
    &\ge \max \left\{ 0, \, \frac{P(Y{=}1 \mid do(A_i{=}a,\, \mathbf{A}_{-i}{=}\mathbf{c})) - P(Y{=}1 \mid do(A_i{=}a',\, \mathbf{A}_{-i}{=}\mathbf{c}))}{P(Y{=}1 \mid do(A_i{=}a,\, \mathbf{A}_{-i}{=}\mathbf{c}))} \right\}
\end{aligned}
\end{equation*}
Thus
\begin{equation*}
\small
\begin{aligned}
    \mathrm{PN}_{(a,a')}(X_i, Y \mid \mathbf{c})
    &\ge \mathrm{N}_{(a,a')}(X_i, Y \mid \mathbf{c}).
\end{aligned}
\end{equation*}
\end{proof}

\begin{theorem}
\label{thm:general_s_score_bound}
Under Assumption~\ref{assum:interv} (Exogeneity), for any prerequisite unit task $X_i \in \mathrm{Pa}(Y)$, any ordered contrast $(a, a')$ of distinct states in $\mathcal{A}$, and any intervention state configuration $\mathbf{c} \in \mathcal{A}^{|\mathrm{Pa}(Y)\setminus\{X_i\}|}$ assigned to the remaining prerequisites $\mathbf{A}_{-i}$, the general S-Score (Definition~\ref{def:general_s_score}) provides a valid lower bound on the Probability of Sufficiency ($\mathrm{PS}$) under prerequisite context $do(\mathbf{A}_{-i}{=}\mathbf{c})$:
\begin{equation}
\small
    \mathrm{PS}_{(a,a')}(X_i, Y \mid \mathbf{c}) \ge \mathrm{S}_{(a,a')}(X_i, Y \mid \mathbf{c})
\end{equation}
In particular, setting every element of $\mathbf{c}$ to the homogeneous state $a'$ (i.e., $do(\mathbf{A}_{-i}{=}a')$) directly establishes the S-Score lower bound in Theorem~\ref{thm:pn_ps_bounds}.
\end{theorem}

\begin{proof}
Similar to the proof of Theorem~\ref{thm:general_n_score_bound}, let $P_{\mathbf{c}}(\cdot) := P(\cdot \mid do(\mathbf{A}_{-i}{=}\mathbf{c}))$ denote the probability measure in the submodel $\mathcal{M}_{\mathbf{c}}$, where $Y_a$ and $Y_{a'}$ represent the potential outcomes of $Y$ under $do(A_i{=}a)$ and $do(A_i{=}a')$. Following \citet{pearl1999probabilities}, the Probability of Sufficiency ($\mathrm{PS}$) across contrast $(a, a')$ under context $\mathbf{c}$ measures the counterfactual probability that upgrading $X_i$ to state $a$ would rescue $Y$ to success, given that $Y$ failed under regime $A_i{=}a'$:
\begin{equation*}
\small
    \mathrm{PS}_{(a,a')}(X_i, Y \mid \mathbf{c}) := P_{\mathbf{c}}(Y_a{=}1 \mid A_i{=}a',\, Y{=}0)
\end{equation*}
If $P_{\mathbf{c}}(A_i{=}a',\, Y{=}0) = 0$, the failure baseline is zero and the bound holds trivially by convention ($\mathrm{S}_{(a,a')} = 0$). Thus, assuming $P_{\mathbf{c}}(A_i{=}a',\, Y{=}0) > 0$, we expand this conditional counterfactual probability using the definition of conditional probability:
\begin{equation*}
\small
    P_{\mathbf{c}}(Y_a{=}1 \mid A_i{=}a',\, Y{=}0) = \frac{P_{\mathbf{c}}(Y_a{=}1,\, A_i{=}a',\, Y{=}0)}{P_{\mathbf{c}}(A_i{=}a',\, Y{=}0)}
\end{equation*}
By consistency, the observed outcome under regime $A_i{=}a'$ coincides with the potential outcome $Y_{a'}$, so we replace the observed failure $Y{=}0$ in the numerator with $Y_{a'}{=}0$:
\begin{equation*}
\small
    P_{\mathbf{c}}(Y_a{=}1 \mid A_i{=}a',\, Y{=}0) = \frac{P_{\mathbf{c}}(Y_a{=}1,\, A_i{=}a',\, Y_{a'}{=}0)}{P_{\mathbf{c}}(A_i{=}a',\, Y{=}0)}
\end{equation*}
Factoring out $P_{\mathbf{c}}(A_i{=}a')$ in the numerator and dividing both the numerator and the denominator by it yields:
\begin{equation*}
\small
\begin{split}
    P_{\mathbf{c}}(Y_a{=}1 \mid A_i{=}a',\, Y{=}0) &= \frac{P_{\mathbf{c}}(Y_a{=}1,\, Y_{a'}{=}0 \mid A_i{=}a') \, P_{\mathbf{c}}(A_i{=}a')}{P_{\mathbf{c}}(A_i{=}a',\, Y{=}0)} \\
    &= \frac{P_{\mathbf{c}}(Y_a{=}1,\, Y_{a'}{=}0 \mid A_i{=}a')}{P_{\mathbf{c}}(Y{=}0 \mid A_i{=}a')}
\end{split}
\end{equation*}
Under the exogeneity of the regime assignment (Assumption~\ref{assum:interv}), the potential outcomes $(Y_a, Y_{a'})$ are independent of $A_i$, and conditioning on $A_i$ is operationally identical to intervening on it. This removes the conditioning event in both the numerator and the denominator:
\begin{equation*}
\small
    P_{\mathbf{c}}(Y_a{=}1 \mid A_i{=}a',\, Y{=}0) = \frac{P_{\mathbf{c}}(Y_a{=}1,\, Y_{a'}{=}0)}{1 - P_{\mathbf{c}}(Y_{a'}{=}1)}
\end{equation*}
Since monotonicity is not assumed (Property~\ref{assum:non_monotonicity}), the joint counterfactual probability $P_{\mathbf{c}}(Y_a{=}1,\, Y_{a'}{=}0)$ is again not point-identified. Applying the Fr\'{e}chet lower bound $P_{\mathbf{c}}(Y_a{=}1,\, Y_{a'}{=}0) \ge \max\{0, \, P_{\mathbf{c}}(Y_a{=}1) + P_{\mathbf{c}}(Y_{a'}{=}0) - 1\}$ \citep{pearl1999probabilities} to the numerator yields:
\begin{equation*}
\small
    \mathrm{PS}_{(a,a')}(X_i, Y \mid \mathbf{c}) \ge \max \left\{ 0, \, \frac{P_{\mathbf{c}}(Y_a{=}1) + P_{\mathbf{c}}(Y_{a'}{=}0) - 1}{1 - P_{\mathbf{c}}(Y_{a'}{=}1)} \right\}
\end{equation*}
Substituting $P_{\mathbf{c}}(Y_{a'}{=}0) = 1 - P_{\mathbf{c}}(Y_{a'}{=}1)$ into the numerator simplifies the fraction to:
\begin{equation*}
\small
    \mathrm{PS}_{(a,a')}(X_i, Y \mid \mathbf{c}) \ge \max \left\{ 0, \, \frac{P_{\mathbf{c}}(Y_a{=}1) - P_{\mathbf{c}}(Y_{a'}{=}1)}{1 - P_{\mathbf{c}}(Y_{a'}{=}1)} \right\}
\end{equation*}
Finally, unpacking the submodel notation $P_{\mathbf{c}}(Y_a{=}1) = P(Y{=}1 \mid do(A_i{=}a,\, \mathbf{A}_{-i}{=}\mathbf{c}))$ and $P_{\mathbf{c}}(Y_{a'}{=}1) = P(Y{=}1 \mid do(A_i{=}a',\, \mathbf{A}_{-i}{=}\mathbf{c}))$ establishes the general S-Score lower bound:
\begin{equation*}
\small
\begin{aligned}
    \mathrm{PS}_{(a,a')}(X_i, Y \mid \mathbf{c})
    &\ge \max \left\{ 0, \, \frac{P(Y{=}1 \mid do(A_i{=}a,\, \mathbf{A}_{-i}{=}\mathbf{c})) - P(Y{=}1 \mid do(A_i{=}a',\, \mathbf{A}_{-i}{=}\mathbf{c}))}{1 - P(Y{=}1 \mid do(A_i{=}a',\, \mathbf{A}_{-i}{=}\mathbf{c}))} \right\} 
\end{aligned}
\end{equation*}
Thus
\begin{equation*}
\small
\begin{aligned}
    \mathrm{PS}_{(a,a')}(X_i, Y \mid \mathbf{c})
    &\ge \mathrm{S}_{(a,a')}(X_i, Y \mid \mathbf{c}).
\end{aligned}
\end{equation*}
\end{proof}

\subsection{Theoretical Properties}
\label{appd:subsec:theory_prop}

In this section, we discuss two key theoretical properties of our contribution metrics that directly underpin our experimental design and diagnostic analysis in the main text: (i) the exact identification of PN and PS under monotonicity versus our conservative lower bounds under LLM non-monotonicity, and (ii) the rationale for our asymmetric prerequisite control scheme over $\mathbf{A}_{-i}$.

\subsubsection{Exact Identification under Monotonicity}

In the probabilities of causation framework~\citep{pearl1999probabilities, tian2000probabilities}, 
when the monotonicity assumption holds, the Fr\'echet lower bounds in Theorems~\ref{thm:general_n_score_bound} and \ref{thm:general_s_score_bound} tighten to exact equality.

\begin{proposition}[Exact Identification under Monotonicity; \citep{tian2000probabilities}]
\label{prop:monotonicity_degeneration}
Under the setting of Theorems~\ref{thm:general_n_score_bound} and \ref{thm:general_s_score_bound}, if $Y$ is monotonic with respect to $X_i$ across contrast $(a, a')$ under context $\mathbf{c}$ (i.e., $P_{\mathbf{c}}(Y_a=0, Y_{a'}=1) = 0$), then N-Score and S-Score equal PN and PS.
\end{proposition}
% exactly

For the detailed proof, we refer the reader to \citet{tian2000probabilities}. 
In LLM evaluation, however, we do not impose the monotonicity assumption (Property~\ref{assum:non_monotonicity}), as we discussed in Section~\ref{sec:decomp}.
Meanwhile, as our experiments in Section~\ref{subsec:exp_cap} show, injecting an incorrect answer ($A_i=0$) can provide a reasoning scaffold that guides downstream deduction better than providing no information ($A_i=\varnothing$), 
causing $\text{RC}(X_i)$ to exceed $\text{NC}(X_i)$.

\subsubsection{Remaining Prerequisite $\mathbf{A}_{-i}$ Controlling}
\label{appdx:subsec:remaining_parent}

\paragraph{Controlling for Remaining Prerequisites $\mathbf{A}_{-i}$.}
In a multi-parent SCM, the causal effect of a unit task $X_i$ on the final outcome $Y$ is also influenced by the states of the remaining prerequisites $\mathbf{A}_{-i}{:=}\mathbf{A}_{\mathrm{Pa}(Y) \setminus \{X_i\}}$. When $\mathbf{A}_{-i}$ is left un-intervened in the natural state $\varnothing$, the actual realizations of these prerequisites remain unobserved and depend on each model's upstream proficiency, which conflates $X_i$'s causal effect with model-specific upstream variations. To isolate $X_i$'s contribution under controlled conditions, our general definitions (Definitions~\ref{def:general_n_score} and~\ref{def:general_s_score}) allow $\mathbf{A}_{-i}$ to take any intervention state configuration $\mathbf{c} \in \mathcal{A}^{|\mathrm{Pa}(Y)|-1}$. Because each prerequisite takes one of three states in $\mathcal{A}{=}\{1, 0, \varnothing\}$, the full configuration space contains $3^{|\mathrm{Pa}(Y)|-1}$ possible combinations. To evaluate $X_i$ across representative prerequisite contexts without enumerating this exponential space, we restrict $\mathbf{A}_{-i}$ to homogeneous configurations where all remaining prerequisites share the same state $c \in \{1, 0, \varnothing\}$. This reduction yields three interpretable background environments: the oracle environment ($\mathbf{A}_{-i}{=}1$), the corrupted environment ($\mathbf{A}_{-i}{=}0$), and the natural environment ($\mathbf{A}_{-i}{=}\varnothing$).

\paragraph{Asymmetric Prerequisite Control in N-Score and S-Score.}
To capture complementary diagnostic properties across these homogeneous environments, our main-text definitions of N-Score and S-Score (Definitions~\ref{def:nscore} and~\ref{def:sscore}) assign asymmetric background states to $\mathbf{A}_{-i}$ for any ordered contrast $(a, a')$. Specifically, N-Score sets $\mathbf{A}_{-i}{=}a$, whereas S-Score sets $\mathbf{A}_{-i}{=}a'$. This design aligns each metric with its intended diagnostic role in LLM evaluation. For N-Score, conditioning on $\mathbf{A}_{-i}{=}a$ establishes a baseline environment where all prerequisites begin at state $a$, allowing the metric to measure single-task vulnerability when $X_i$ alone degrades to $a'$. For S-Score, conditioning on $\mathbf{A}_{-i}{=}a'$ establishes a baseline environment where all prerequisites begin at state $a'$, allowing the metric to measure single-task recovery when $X_i$ alone upgrades to $a$. Furthermore, combining this asymmetric control with our ordered contrasts $(1, 0)$, $(1, \varnothing)$, $(\varnothing, 0)$, and $(0, \varnothing)$ systematically covers all three homogeneous environments $\mathbf{A}_{-i} \in \{1, 0, \varnothing\}$ and assigns each contrast a distinct operational meaning. For instance, under contrast $(1, \varnothing)$, $\mathrm{S}_{(1, \varnothing)}$ evaluates how providing the correct answer for $X_i$ alone improves overall task accuracy in the natural environment ($\mathbf{A}_{-i}{=}\varnothing$), directly identifying high-leverage targets for single-task repair. Under the same contrast, $\mathrm{N}_{(1, \varnothing)}$ measures how leaving $X_i$ unassisted reduces accuracy in the oracle environment ($\mathbf{A}_{-i}{=}1$). Similarly, contrasts $(1, 0)$, $(\varnothing, 0)$, and $(0, \varnothing)$ quantify the impact of prerequisite corruption and structural scaffolding across oracle, natural, and corrupted environments.

\paragraph{Algebraic Coupling under Identical Prerequisite Control.}
Under our general definitions of N-Score and S-Score (Definitions~\ref{def:general_n_score} and~\ref{def:general_s_score}), when both metrics are evaluated under the same prerequisite configuration $\mathbf{c} \in \mathcal{A}^{|\mathrm{Pa}(Y)|-1}$, they share a common interventional background $\mathrm{do}(\mathbf{A}_{-i}{=}\mathbf{c})$. Within this shared environment, necessity and sufficiency are algebraically coupled through the Probability of Necessity and Sufficiency ($\mathrm{PNS}$) \citep{pearl1999probabilities, tian2000probabilities}.

\begin{proposition}[Algebraic Coupling under Identical Context]
\label{prop:pns_relationship}
For any prerequisite unit task $X_i \in \mathrm{Pa}(Y)$, any ordered contrast $(a, a')$ of distinct states in $\mathcal{A}$, and any fixed prerequisite configuration $\mathbf{c} \in \mathcal{A}^{|\mathrm{Pa}(Y)|-1}$, let $p_{a|\mathbf{c}} := P(Y{=}1 \mid \mathrm{do}(A_i{=}a, \mathbf{A}_{-i}{=}\mathbf{c}))$ and $p_{a'|\mathbf{c}} := P(Y{=}1 \mid \mathrm{do}(A_i{=}a', \mathbf{A}_{-i}{=}\mathbf{c}))$. Denote the conditional Probability of Necessity and Sufficiency under context $\mathbf{c}$ as $\mathrm{PNS}_{(a,a')}(X_i, Y \mid \mathbf{c}) := P(Y_a{=}1, Y_{a'}{=}0 \mid \mathrm{do}(\mathbf{A}_{-i}{=}\mathbf{c}))$. Under exogeneity assumption~(Assumption~\ref{assum:interv}), the exact counterfactual probabilities satisfy:
\begin{equation}
\small
\label{eq:exact_pns_coupling}
\mathrm{PN}_{(a,a')}(X_i, Y \mid \mathbf{c}) \cdot p_{a|\mathbf{c}} = \mathrm{PS}_{(a,a')}(X_i, Y \mid \mathbf{c}) \cdot (1 - p_{a'|\mathbf{c}}) = \mathrm{PNS}_{(a,a')}(X_i, Y \mid \mathbf{c}),
\end{equation}
and the general N-Score and S-Score satisfy the parallel lower-bound coupling:
\begin{equation}
\small
\label{eq:ns_pns_bound}
\mathrm{N}_{(a,a')}(X_i, Y \mid \mathbf{c}) \cdot p_{a|\mathbf{c}} = \mathrm{S}_{(a,a')}(X_i, Y \mid \mathbf{c}) \cdot (1 - p_{a'|\mathbf{c}}) = \max\big\{0, \; p_{a|\mathbf{c}} - p_{a'|\mathbf{c}}\big\} \le \mathrm{PNS}_{(a,a')}(X_i, Y \mid \mathbf{c}).
\end{equation}
\end{proposition}

\begin{proof}
Let $P_{\mathbf{c}}(\cdot) := P(\cdot \mid \mathrm{do}(\mathbf{A}_{-i}{=}\mathbf{c}))$ denote the probability measure in the submodel induced by $\mathrm{do}(\mathbf{A}_{-i}{=}\mathbf{c})$~\citep{pearl2009causality}.
Under exogeneity ssumption, conditioning on $A_i$ within this submodel is operationally identical to intervening on it, so $P_{\mathbf{c}}(Y_a{=}1) = p_{a|\mathbf{c}}$ and $P_{\mathbf{c}}(Y_{a'}{=}0) = 1 - p_{a'|\mathbf{c}}$. By consistency and the definition of conditional probability, $\mathrm{PN}_{(a,a')}(X_i, Y \mid \mathbf{c}) \cdot p_{a|\mathbf{c}} = P_{\mathbf{c}}(Y_{a'}{=}0 \mid A_i{=}a, Y{=}1) \cdot P_{\mathbf{c}}(Y{=}1 \mid A_i{=}a) = P_{\mathbf{c}}(Y_a{=}1, Y_{a'}{=}0) = \mathrm{PNS}_{(a,a')}(X_i, Y \mid \mathbf{c})$. An analogous expansion yields $\mathrm{PS}_{(a,a')}(X_i, Y \mid \mathbf{c}) \cdot (1 - p_{a'|\mathbf{c}}) = \mathrm{PNS}_{(a,a')}(X_i, Y \mid \mathbf{c})$, establishing Equation~\eqref{eq:exact_pns_coupling}. 

Next, multiplying Definition~\ref{def:general_n_score} by $p_{a|\mathbf{c}}$ and Definition~\ref{def:general_s_score} by $1 - p_{a'|\mathbf{c}}$ directly gives $\max\{0, p_{a|\mathbf{c}} - p_{a'|\mathbf{c}}\}$. Applying the classical Fr\'echet inequality yields $\mathrm{PNS}_{(a,a')}(X_i, Y \mid \mathbf{c}) \ge \max\{0, p_{a|\mathbf{c}} + (1 - p_{a'|\mathbf{c}}) - 1\} = \max\{0, p_{a|\mathbf{c}} - p_{a'|\mathbf{c}}\}$, which establishes Equation~\ref{eq:ns_pns_bound}.
\end{proof}

\paragraph{Metric Independence under Asymmetric Control.}
Proposition~\ref{prop:pns_relationship} shows that necessity and sufficiency are algebraically convertible when evaluated within the same background environment $\mathbf{c}$. 
However, note that in our main-text formulation (Definitions~\ref{def:nscore} and~\ref{def:sscore}), N-Score and S-Score evaluate $X_i$ across two distinct background environments: 
N-Score conditions on $\mathbf{A}_{-i}{=}a$, whereas S-Score conditions on $\mathbf{A}_{-i}{=}a'$. 
Whenever $Y$ depends on multiple prerequisites ($|\mathrm{Pa}(Y)| > 1$), the contrast states $a \neq a'$ place the remaining prerequisites into two different interventional configurations. 
Because the two metrics operate under distinct prerequisite states, their numerators correspond to the Fr\'echet lower bounds of two different conditional $\mathrm{PNS}$ quantities, $\mathrm{PNS}_{(a,a')}(X_i, Y \mid a)$ and $\mathrm{PNS}_{(a,a')}(X_i, Y \mid a')$, where in general $p_{a|a} - p_{a'|a} \neq p_{a|a'} - p_{a'|a'}$. 
Consequently, our main-text N-Score and S-Score cannot be algebraically derived from one another through marginal outcome probabilities. 
A special case arises when $|\mathrm{Pa}(Y)| {=} 1$, representing the single-parent setting where $\mathbf{A}_{-i}$ is empty; in this case, both metrics operate within the same unconditional environment and reduce directly to the classical coupled relationship in Equation~\ref{eq:ns_pns_bound}. 
For general multi-parent tasks, this LLM-specific decoupling ensures that N-Score and S-Score serve as independent diagnostic dimensions, allowing each paired $(a, a')$ score to provide distinct and practical diagnostic value for pinpointing failure bottlenecks versus high-leverage repair targets.

\subsection{Markov Chain Structure}

In Section~\ref{sec:benchmark}, we introduced Markov chain unit tasks to represent sub-problems that require iterative solving across sequential steps. In the task-level SCM (Figure~\ref{fig:scm_gallery}), a Markov chain unit task is represented as a single composite node that receives inputs from upstream prerequisites and feeds into downstream tasks.
To analyze how model capabilities and error propagation evolve over long reasoning horizons, we unfold this composite node into an explicit sequential chain of unit tasks~(Section~\ref{subsec:exp_markov}).
Below, we formalize the Markov chain structure within a task decomposition SCM and derive its step-wise causal properties.

\begin{definition}[Markov Chain Structure in Task Decomposition SCM]
\label{def:markov_chain}
Consider a task decomposition SCM (Definition~\ref{def:scm}). A {Markov chain structure} is a recursive sub-model region defined over $T$ sequential reasoning steps ($t = 1, \dots, T$), where each step consists of a recurring cycle of $p$ constituent unit tasks ($p \ge 1$), represented by unit task variables $X_1^{(t)}, X_2^{(t)}, \dots, X_p^{(t)}$.
Each unit task variable $X_i^{(t)}$ is given by the structural assignment:
\begin{equation}
X_i^{(t)} := f_i^{(t)}\Big(\mathrm{Pa}\big(X_i^{(t)}\big),\; N_i^{(t)}\Big), \quad \text{for } i = 1, \dots, p,\;\; t = 1, \dots, T,
\end{equation}
where $f_i^{(t)}$ is the structural function determining $X_i^{(t)}$, $N_i^{(t)}$ is the exogenous noise variable, and the parent set $\mathrm{Pa}\big(X_i^{(t)}\big)$ is uniquely defined by:
\begin{equation}
\mathrm{Pa}\big(X_i^{(t)}\big) = \begin{cases}
\big\{ X_{i-1}^{(t)} \big\}, & \text{if } i \in \{2, \dots, p\}, \\[4pt]
\big\{ X_p^{(t-1)} \big\}, & \text{if } i = 1 \text{ and } t \in \{2, \dots, T\},
\end{cases}
\end{equation}
with the initial prerequisite set $\mathrm{Pa}\big(X_1^{(1)}\big)$ given by the upstream prerequisites external to the Markov chain region in the SCM (or $\emptyset$ if the chain originates at a root node).
Unfolding this recursive dependency across all $T$ steps yields the explicit sequential causal chain:
\begin{equation}
\small
X_1^{(1)} \to X_2^{(1)} \to \cdots \to X_p^{(1)} \;\to\; X_1^{(2)} \to \cdots \to X_p^{(T-1)} \;\to\; X_1^{(T)} \to \cdots \to X_p^{(T)}.
\end{equation}
\end{definition}

\begin{remark}
\label{rem:markov_pns_coupling}
In the unfolded Markov chain structure (Definition~\ref{def:markov_chain}), every non-initial unit task variable $X_j$ has a single causal prerequisite, i.e., $|\mathrm{Pa}(X_j)| = 1$. Consequently, for any such unit task $X_j$ and its unique direct parent $X_{\mathrm{prev}} \in \mathrm{Pa}(X_j)$, the set of remaining prerequisites is empty:
\begin{equation}
\mathbf{A}_{-\mathrm{prev}} := \mathbf{A}_{\mathrm{Pa}(X_j) \setminus \{X_{\mathrm{prev}}\}} = \emptyset.
\end{equation}
Under this single-parent setting, $\mathbf{c} \in \mathcal{A}^{|\mathrm{Pa}(X_j) \setminus \{X_{\mathrm{prev}}\}|}$ in Definitions~\ref{def:general_n_score} and~\ref{def:general_s_score} becomes $\emptyset$, and the N-Score and S-Score simplify to:
\begin{equation}
\small
\begin{aligned}
\mathrm{N}_{(a, a')}(X_{\mathrm{prev}}, X_j) &= \max\left\{0,\; \frac{P\big(X_j=1 \mid do(A_{\mathrm{prev}}=a)\big) - P\big(X_j=1 \mid do(A_{\mathrm{prev}}=a')\big)}{P\big(X_j=1 \mid do(A_{\mathrm{prev}}=a)\big)}\right\}, \\[4pt]
\mathrm{S}_{(a, a')}(X_{\mathrm{prev}}, X_j) &= \max\left\{0,\; \frac{P\big(X_j=1 \mid do(A_{\mathrm{prev}}=a)\big) - P\big(X_j=1 \mid do(A_{\mathrm{prev}}=a')\big)}{1 - P\big(X_j=1 \mid do(A_{\mathrm{prev}}=a')\big)}\right\}.
\end{aligned}
\end{equation}
As an immediate consequence, Proposition~\ref{prop:pns_relationship} directly holds without requiring external context conditioning:
\begin{equation}
\small
\begin{aligned}
\mathrm{N}_{(a, a')}(X_{\mathrm{prev}}, X_j) \cdot p_a 
&= \mathrm{S}_{(a, a')}(X_{\mathrm{prev}}, X_j) \cdot (1-p_{a'}) = \max\{0,\; p_a - p_{a'}\},
\end{aligned}
\end{equation}
where $p_a = P\big(X_j=1 \mid do(A_{\mathrm{prev}}=a)\big)$,
$p_{a'} = P\big(X_j=1 \mid do(A_{\mathrm{prev}}=a')$, and
$\max\{0,\; p_a - p_{a'}\}$ is the $NS$ score reported in Sec.~\ref{subsec:exp_markov}, denoting lower bound of PNS. 
\end{remark}

In our \ourdataset benchmark, three tasks contain Markov chain structures with cycle lengths $p=1$~(Maze, Map Navigation) and $p=2$~(Multistep vision reasoning) (Figure~\ref{fig:scm_gallery}), corresponding to the step-wise evaluations in Sec.~\ref{subsec:exp_markov} (Figure~\ref{fig:finding_markov_causal_profiles}).

\section{Details of Benchmark}
\label{app:benchmark}

\subsection{Benchmark Composition}

\subsubsection{Task Inventory and Statistics}
\label{subsubsec:bench_inventory}

\begin{table}[t]
\centering
\caption{Overview of the 10 overall evaluation tasks in \ourdataset.}
\label{tab:cadet_tasks_overview}
\resizebox{0.8\linewidth}{!}{
\begin{tabular}{lllcl}
\toprule
\textbf{Task Name} & \textbf{Modality} & \textbf{Category} & \textbf{\# Instances} & \textbf{Answer Format} \\
\midrule
Music Note Counting    & Image       & Perception & 114   & Integer \\
Dice Sum               & Image       & Perception & 130   & Integer \\
Map Navigation         & Image       & Spatial    & 107   & Integer \\
Maze                   & Image       & Spatial    & 80    & Multiple-Choice \\
Surround Localization  & Multi-Image & Spatial    & 154   & Integer\\
Multi-View Sequencing  & Multi-Image & Temporal   & 154   & Letter List \\
Street-View Ordering   & Multi-Image & Temporal   & 91    & Letter List \\
Dashcam Counting       & Video       & Cognitive  & 79    & Integer \\
Traffic Causation      & Video       & Cognitive  & 107   & Free-form Text \\
Multistep Reasoning    & Image       & Cognitive  & 128   & Free-form Text \\
\bottomrule
\end{tabular}
}
\end{table}

\begin{table}[t]
\centering
\caption{Structure of the task-level SCM of each evaluation task. -- on multistep reasoning is due to it start from a Markov chain structure.}
\label{tab:scm_structure}
\resizebox{0.9\linewidth}{!}{
\begin{tabular}{lccccrrc}
\toprule
Task & \#Units & \#Roots & Depth & $|\mathrm{Pa}(Y)|$ & Category & \#Questions & Has Markov \\
\midrule
Music Note Counting   & 4 & 2   & 2 & 3 & P/C     &  1,163 & No  \\
Dice Sum              & 4 & 2   & 2 & 3 & P/C     &  1,919 & No  \\
Map Navigation        & 3 & 1   & 2 & 2 & P/S/C   &    918 & Yes \\
Maze                  & 2 & 1   & 1 & 1 & P/S     &    880 & Yes \\
Surround Localization & 6 & 3   & 2 & 5 & P/S     &  4,928 & No  \\
Multi-View Sequencing & 6 & 2   & 3 & 5 & P/S/T   &  4,158 & No  \\
Street-View Ordering  & 6 & 2   & 3 & 5 & P/S/T   &  3,268 & No  \\
Dashcam Counting      & 5 & 2   & 3 & 4 & P/S/T/C &  6,895 & No  \\
Traffic Causation     & 7 & 1   & 4 & 6 & P/T/C   &  2,974 & No  \\
Multistep Reasoning   & 3 & -- & 2 & 2 & P/C     &  6,828 & Yes \\
\bottomrule
\end{tabular}
}
\end{table}

This section reports the benchmark details following our Section~\ref{sec:benchmark}.
Table~\ref{tab:cadet_tasks_overview} lists the 10 overall evaluation tasks, and
Table~\ref{tab:scm_structure} reports the overview and structure of their task-level SCMs, reported depth is the length of the longest directed path terminating at $Y$.

\subsubsection{Capability Categories}
\label{subsubsec:bench_capability}

We group unit tasks into four categories according to the kind of information the task
requires the model to produce. Each unit task is assigned to the category of the
capability it primarily targets.
\begin{itemize}
\item \textbf{Perception}: reading out the presence, identity, category, or image position
of an object from the visual signal, without relating separate elements to one another.
\item \textbf{Spatial}: relating elements to each other or to the observer, covering
distance, size, orientation, viewpoint, and navigation.
\item \textbf{Temporal}: relating observations across frames or views, covering motion,
ordering, and localization in time.
\item \textbf{Cognitive}: operating on already extracted facts, covering counting,
arithmetic, causal attribution, and multi-step or goal-directed reasoning.
\end{itemize}
The overall task of each SCM is labelled with the same four categories
(Table~\ref{tab:cadet_tasks_overview}). Across the 46 unit tasks, the four categories contain
15, 12, 10, and 9 tasks respectively.

\subsubsection{Per-Task SCM and Unit Task}
\label{subsubsec:bench_scm_spec}

Table~\ref{tab:unit_spec} specifies the unit tasks and SCM structure of every evaluation task, giving for each node its capability, category, answer format, number of questions, and parent set.
Nodes are indexed in a topological order of the task-level SCM.
As a result, \ourdataset contains over 33,000 unit task level questions.

\begin{table}[p]
\centering
\caption{Unit task specification of the 10 evaluation tasks in CADET. $Y$ denotes the
overall task outcome; $\circlearrowright$ marks a Markov chain unit task.}
\label{tab:unit_spec}
\resizebox{0.775\linewidth}{!}{
\begin{tabular}{llclrl}
\toprule
Node & Unit Task & Cat. & Answer Format & \#Q & Parents \\
\midrule
\multicolumn{6}{l}{\textit{Music Note Counting}} \\
$X_1$ & Object Identification        & P & Multiple-choice &   821 & --- \\
$X_2$ & Object Classification        & P & Multiple-choice &   114 & $X_1$ \\
$X_3$ & Completeness Verification    & P & Multiple-choice &   114 & --- \\
$Y$   & Object Counting              & C & Integer         &   114 & $X_1, X_2, X_3$ \\
\midrule
\multicolumn{6}{l}{\textit{Dice Sum}} \\
$X_1$ & Object Detection             & P & Integer         &   130 & --- \\
$X_2$ & 3D Surface Recognition       & P & Multiple-choice &   791 & --- \\
$X_3$ & Fine-grained Counting        & C & Integer         &   868 & $X_2$ \\
$Y$   & Arithmetic Summation         & C & Integer         &   130 & $X_1, X_2, X_3$ \\
\midrule
\multicolumn{6}{l}{\textit{Map Navigation}} \\
$X_1$ & Object Localization          & P & Multiple-choice &   443 & --- \\
$X_2$ & Sequential Navigation $\circlearrowright$ & S & Multiple-choice & 368 & $X_1$ \\
$Y$   & Path-Based Counting          & C & Integer         &   107 & $X_1, X_2$ \\
\midrule
\multicolumn{6}{l}{\textit{Maze}} \\
$X_1$ & Object Localization          & P & Multiple-choice &    80 & --- \\
$Y$   & Sequential Navigation $\circlearrowright$ & S & Multiple-choice & 800 & $X_1$ \\
\midrule
\multicolumn{6}{l}{\textit{Surround Localization}} \\
$X_1$ & Object Identification        & P & Multiple-choice & 1,078 & --- \\
$X_2$ & Front-View Identification    & S & Multiple-choice & 1,232 & --- \\
$X_3$ & Viewpoint Estimation         & S & Multiple-choice & 2,156 & --- \\
$X_4$ & Camera View Identification   & S & Multiple-choice &   154 & $X_1, X_2$ \\
$X_5$ & Object Localization          & P & Multiple-choice &   154 & $X_1$ \\
$Y$   & Clock Position Estimation    & S & Integer         &   154 & $X_1, \dots, X_5$ \\
\midrule
\multicolumn{6}{l}{\textit{Multi-View Sequencing}} \\
$X_1$ & Object Localization        & P & Multiple-choice &   770 & --- \\
$X_2$ & Viewpoint Estimation         & S & Multiple-choice & 1,386 & $X_1$ \\
$X_3$ & Distance Comparison          & S & Multiple-choice & 1,540 & $X_1$ \\
$X_4$ & Motion Detection             & T & Multiple-choice &   154 & --- \\
$X_5$ & Interaction Status Identification   & T & Multiple-choice &   154 & $X_1, X_3, X_4$ \\
$Y$   & Temporal Ordering            & T & Letter list     &   154 & $X_1, \dots, X_5$ \\
\midrule
\multicolumn{6}{l}{\textit{Street-View Ordering}} \\
$X_1$ & Object Localization          & P & Multiple-choice &   455 & --- \\
$X_2$ & Size Comparison              & S & Multiple-choice &   897 & $X_1$ \\
$X_3$ & Distance Comparison          & S & Multiple-choice &   910 & $X_1, X_2$ \\
$X_4$ & Horizontal Position Comparison & S & Multiple-choice & 824 & $X_1$ \\
$X_5$ & Motion Detection             & T & Multiple-choice &    91 & --- \\
$Y$   & Temporal Ordering            & T & Letter list     &    91 & $X_1, \dots, X_5$ \\
\midrule
\multicolumn{6}{l}{\textit{Dashcam Counting}} \\
$X_1$ & Object Classification        & P & Multiple-choice & 2,860 & --- \\
$X_2$ & Spatial Relation             & S & Multiple-choice & 2,237 & --- \\
$X_3$ & Object Re-identification     & P & Multiple-choice & 1,378 & $X_1$ \\
$X_4$ & Motion Classification        & T & Multiple-choice &   341 & $X_1, X_2, X_3$ \\
$Y$   & Temporal Object Counting     & C & Integer         &    79 & $X_1, \dots, X_4$ \\
\midrule
\multicolumn{6}{l}{\textit{Traffic Causation}} \\
$X_1$ & Object Localization          & P & Multiple-choice & 2,290 & --- \\
$X_2$ & State Change Recognition     & T & Multiple-choice &   200 & $X_1$ \\
$X_3$ & Reference-Based Action Recognition & T & Multiple-choice & 153 & $X_1$ \\
$X_4$ & Temporal Localization        & T & Multiple-choice &    93 & $X_2, X_3$ \\
$X_5$ & Temporal Sequencing          & T & Multiple-choice &   105 & $X_4$ \\
$X_6$ & False Correlation Rejection  & C & Multiple-choice &    26 & $X_1, X_3$ \\
$Y$   & Causal Reasoning             & C & Free-form text  &   107 & $X_1, \dots, X_6$ \\
\midrule
\multicolumn{6}{l}{\textit{Multistep Reasoning}} \\
$X_1$ & Position Tracking $\circlearrowright$ & P & Integer list   & 3,350 & $X_2$ \\
$X_2$ & Action Prediction $\circlearrowright$ & C & Free-form text & 3,350 & $X_1$ \\
$Y$   & Multi-step Reasoning         & C & Free-form text  &   128 & $X_1, X_2$ \\
\bottomrule
\end{tabular}
}
\end{table}

\subsection{Benchmark Design Rationale}
\label{appdx:subsec:bench_design_retionale}

\textbf{Task Selection.}
The 10 evaluation tasks in CADET are selected based on three considerations. First, we focus on problems where current frontier MLLMs remain far from saturated under end-to-end evaluation, so that the diagnosis concerns capability gaps that current models still struggle with. Second, we select compositional tasks that require multiple interdependent capabilities to reach the final answer, so that each task admits a meaningful decomposition into unit tasks. Third, we prioritize real-world visual scenarios across single-image, multi-image, and video modalities, and the resulting 46 unit tasks cover perception, spatial, temporal, and cognitive capabilities.

\noindent\textbf{Prerequisite Identification and SCM Structure.}
Each unit task corresponds to a single capability. A capability is included in the decomposition of an evaluation task when it is a required information for answering the overall task question. For example, to tell the number on the top face of a dice, one needs to know where the dice is and which face is its top face; to tell the direction of a pedestrian relative to the ego vehicle, one needs to locate the pedestrian and determine the viewpoint of the camera that observes it. We specify a prerequisite dependency from $X_j$ to $X_i$ when solving $X_i$ requires the information or the capability represented by $X_j$. The same task structure is used for all evaluated models. We discuss the completeness of the specified SCMs in Appendix~\ref{sec:discussion}.

\noindent\textbf{Unit-Task Question Granularity.}
As shown in Tables~\ref{tab:cadet_tasks_overview} and~\ref{tab:unit_spec}, the number of questions of a unit task is not necessarily equal to the number of overall task instances. 
When an overall task is built on multiple images, video frames, or target objects, the capability of a unit task does not necessarily apply at the same granularity, and querying all of them within a single bundled question would turn the unit task into a multi-target aggregation problem. 
In such cases, we formulate separate questions at the granularity at which the capability applies.
For example, when a task instance contains five images, object identification is asked separately for each image; when a task instance contains several target objects, object localization is asked separately for each object.

\subsection{Benchmark Construction}

\subsubsection{Image/Video Source and License}

\begin{table}[t]
\centering
\caption{Source media used in \ourdataset.}
\label{tab:source_media}
\resizebox{0.95\linewidth}{!}{
\setlength{\tabcolsep}{4pt}
\renewcommand{\arraystretch}{1.15}
\newcolumntype{L}[1]{>{\raggedright\arraybackslash}p{#1\linewidth}}
\
\begin{tabular}{@{}L{0.16} L{0.42} L{0.15} L{0.20}@{}}
\toprule
\textbf{Source} & \textbf{Source Media Used} & \textbf{Reference} & \textbf{License} \\
\midrule
Argoverse 2 & Urban driving scenes: multi-camera ring imagery and ego-view video & \citet{wilson2023argoverse2} & CC BY-NC-SA 4.0 \\
ROADWork & Roadwork and construction-zone driving images and videos & \citet{ghosh2025roadwork} & ODC-By v1.0 \\
BabyVision & Language-independent core visual reasoning images & \citet{chen2026babyvision} & MIT \\
MEGA-Bench & Real-world multimodal task images & \citet{chen2025megabench} & Apache-2.0 \\
EMMA & Scientific and mathematical reasoning figures & \citet{hao2025emma} & Apache-2.0 \\
Mutopia Project & Engraved sheet-music scores & Mutopia Website & CC BY licensed$^*$ \\
Wikimedia Commons & Sheet-music and general-purpose images & Wikimedia Website & Common License$^*$ \\
Procedurally generated & Synthetic mazes, navigation maps, dice renderings & -- & Ours (1P) \\
\bottomrule
\end{tabular}
}

\vspace{3pt}
{\scriptsize\raggedright
\textbf{Source URLs.}
Argoverse 2: \nolinkurl{https://www.argoverse.org/av2.html};
ROADWork: \nolinkurl{https://github.com/anuragxel/roadwork-dataset};
BabyVision: \nolinkurl{https://huggingface.co/datasets/UnipatAI/BabyVision};
MEGA-Bench: \nolinkurl{https://github.com/TIGER-AI-Lab/MEGA-Bench};
EMMA: \nolinkurl{https://github.com/EMMA-Bench/EMMA};
Mutopia: \nolinkurl{https://www.mutopiaproject.org/};
Wikimedia Commons: \nolinkurl{https://commons.wikimedia.org/}.
\par

* The Mutopia Project and Wikimedia Commons both host media under a mixture of licenses that are declared individually for each file. For these two sources we applied the site-level CC BY license filter when browsing and retained only those images distributed under a CC BY license.
}
\end{table}

As described in Section~\ref{sec:benchmark}, while all evaluation questions, task decompositions, and unit-task annotations in \ourdataset are newly created to mitigate data contamination, the underlying source images and videos are curated from publicly available datasets and open repositories with permissive licenses. Table~\ref{tab:source_media} summarizes all source media~(image and video) used in \ourdataset, along with their corresponding descriptions, references, and license terms.

\subsubsection{Data Construction and Quality Assurance}

\textbf{Two-Stage Construction and Closed-Form Annotation.}
We construct \ourdataset in two stages following predefined task structures. In the first stage, for each of the 10 overall tasks, we define the task scope, a question-answer template, and detailed task instructions on the selected source images and videos. Raters follow these instructions and templates to write the overall task question and its ground-truth answer. In the second stage, to decompose each overall task into unit tasks without introducing noise from free-form question writing, we design the decomposition step as closed-form annotation and labeling tasks. Depending on the unit task, raters draw bounding boxes or paths on the images and provide labels that follow strict format specifications. Once these annotations are collected, we extract the labeled information and convert it into unit-task questions, correct answers ($A_i = 1$), and rendered visual annotations through structured templates. For the incorrect answers ($A_i = 0$), structured unit tasks use random sampling over plausible candidates; specifically, for numerical tasks, we sample within a bounded local neighborhood around the ground truth (e.g., within $\pm 2$ of a groundtruth value of $5$) rather than drawing distant values (e.g., $20$) that models would trivially rule out. For the remaining three free-form text tasks where random sampling is not applicable, raters write the incorrect answers following instructions designed to minimize subjective bias.

\noindent\textbf{Two-Rater Iterative Verification.}
All annotations are conducted by a pool of over 130 experienced multimodal raters. Across both stages, every instance is processed by two raters in sequence: the first rater completes the annotation, and a second quality-check rater independently inspects the result. If the quality-check rater identifies an inaccurate answer, an imprecise bounding box or path, or a label that violates the required format specification, the instance is returned to the first rater for revision. This review cycle repeats until the quality-check rater accepts the annotation.

\noindent\textbf{Instance Revision and Corner-Case Filtering.}
At every annotation step, raters can flag issues with a candidate instance. When a question requires minor rephrasing or an answer can be refined against clear visual content, raters directly update the item within the review loop; 
when an instance cannot be resolved or cleanly annotated at the unit-task level, raters are instructed to drop it. 
Beyond ambiguities in the overall task, drops in the second stage occur when an intermediate unit task cannot be reliably annotated across frames or views. 
Typical cases include a spatial reference (such as a roadside curb) shifting to the image boundary with insufficient visibility, a target vehicle tracked across views becoming too small or partially occluded in a later frame, or a second vehicle of the same color partially entering the scene and preventing unique target identification. 
Removing such instances ensures that every unit task in the SCM admits a well-defined intervention state. 
Across 1,467 candidate instances dispatched for annotation, 323~(22.0\%) were dropped under these criteria, yielding the final 1,144 instances and over 33,000 unit-task questions in \ourdataset.

\section{Experiments Setup}

\textbf{Evaluated Models.}
We evaluate 12 MLLMs across six model families: Gemini 3.1 Pro~\citep{gemini3pro2025}, Gemini 3.5 Flash~\citep{gemini35flash2026}, Gemini 3 Flash~\citep{gemini3flash}, Gemini 3.1 Flash-Lite~\citep{gemini31flashlite}, GPT-5.4~\citep{gpt54}, GPT-5.4 Nano~\citep{gpt54}, Grok 4.3~\citep{grok4modelcard2025}, Gemma 4 31B~\citep{gemma4}, Gemma 4 26B~\citep{gemma4}, Qwen3.5-397B-A17B~\citep{qwen3.5}, Qwen3.6 27B~\citep{qwen3.6_27b}, and Kimi K2.5~\citep{team2026kimi25}.
The Gemini, GPT, Grok and Gemma models are queried through the official APIs, while the remaining open-weight models are queried via OpenRouter.
All models are run with their default sampling settings, maximum supported output tokens, and thinking mode set to \textit{high} (or left at the default active setting when not configurable).

\noindent\textbf{LLM-as-a-judge details.}
Among the 46 unit tasks, 43 use structured output formats (multiple-choice, integer, or coordinate) and are evaluated via deterministic rule-based parsing and exact match, while the remaining 3 free-form text unit tasks are evaluated using Gemini 3.6 Flash~\citep{gemini36flash2026} by comparing model predictions against ground-truth reference answers. Although these 3 unit tasks do not constrain outputs to option letters or numbers, each still pairs with a concise, specific reference answer (e.g., an atomic action such as ``turn left''), so the judge only checks whether the predicted phrase conveys the same meaning as the reference. For the 43 structured unit tasks, Gemini 3.6 Flash is invoked strictly as a fallback when a response deviates from the required surface format (e.g., outputting ``So the answer is **C**.'' instead of ``C''). Across all evaluation runs on structured unit tasks, deterministic exact match directly resolves $98.03\%$ of predictions, with only $1.97\%$ triggering the LLM-as-a-judge fallback.

\noindent\textbf{Prompts.}
Figures~\ref{fig:prompt_intervention} and~\ref{fig:prompt_judge} present the prompt templates used to assemble intervened prerequisite context and to run the LLM judge, respectively. For active interventions ($A_i \in \{1, 0\}$), the prerequisites $\text{Pa}(X_i)$ are formatted as a structured JSON question--answer block and prepended to the target task query using the same template for both $A_i = 1$ and $A_i = 0$, differing only in the filled prerequisite answer values (along with the corresponding visual annotations rendered on the image when $A_i = 1$ represents spatial ground truth). 
Under the natural state ($A_i = \varnothing$), the prerequisite context block is omitted while the target question and output-format instructions remain the same.

\begin{figure}[t]
  \centering
  \includegraphics[width=\linewidth]{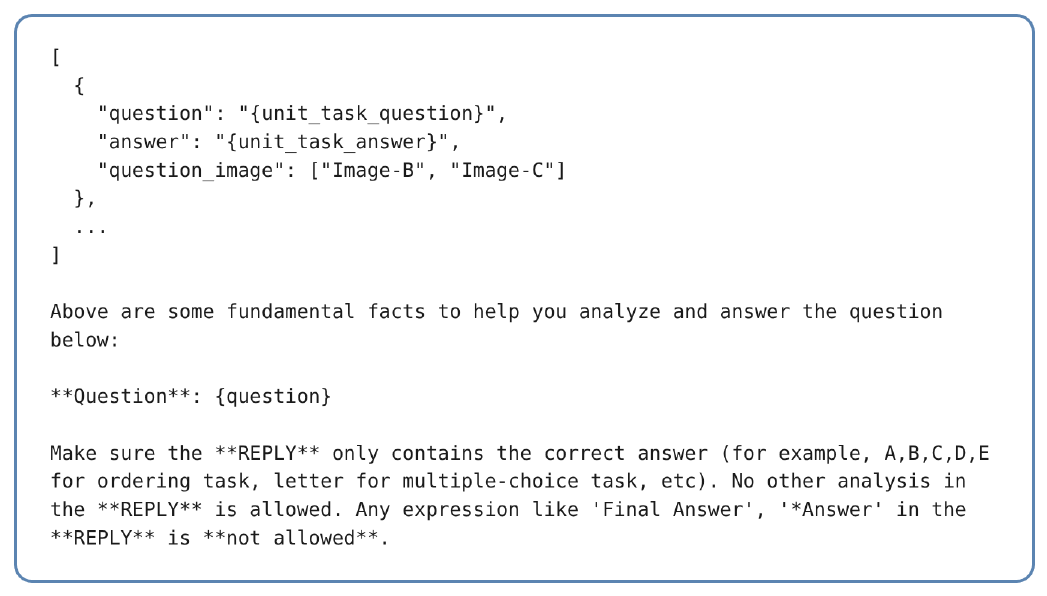}
  \caption{Prompt template for our evaluation under active causal interventions ($A_i \in \{1, 0\}$). Under the natural state ($A_i = \varnothing$), only the target question and last paragraph of output format instructions are prompted.}
  \label{fig:prompt_intervention}
\end{figure}

\begin{figure}[t]
  \centering
  \includegraphics[width=\linewidth]{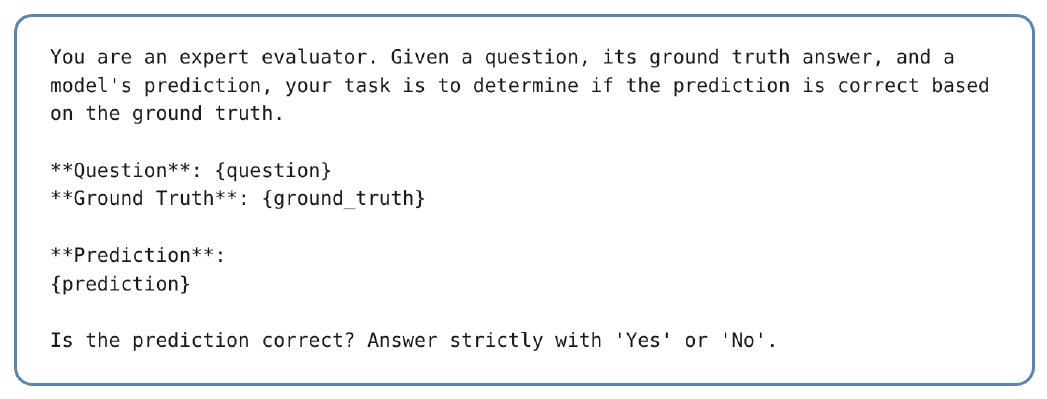}
  \caption{LLM-as-a-judge prompt template.}
  \label{fig:prompt_judge}
\end{figure}

\section{More Experiment Results}

\subsection{More Capability Diagnosis Results (NC, IC, RC)}
\label{app:more-capability-results}

\begin{figure}[t]
\centering
\includegraphics[width=0.4\linewidth]{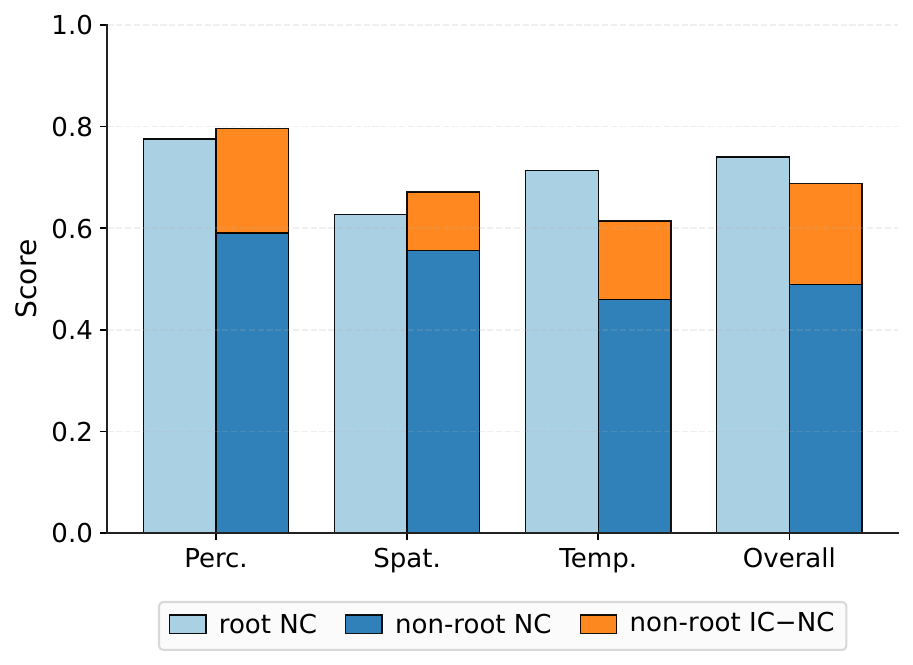}
\caption{Comparing NC on root unit tasks against NC and IC on non-root unit tasks, per category, averaged over the 12 models. Cognitive is not shown as it has no root unit task. The non-root bar stacks IC$-$NC on top of NC, so its top is IC.}
\label{fig:appd-root-vs-nonroot-category}
\end{figure}

\begin{figure}[t]
\centering
\includegraphics[width=0.55\linewidth]{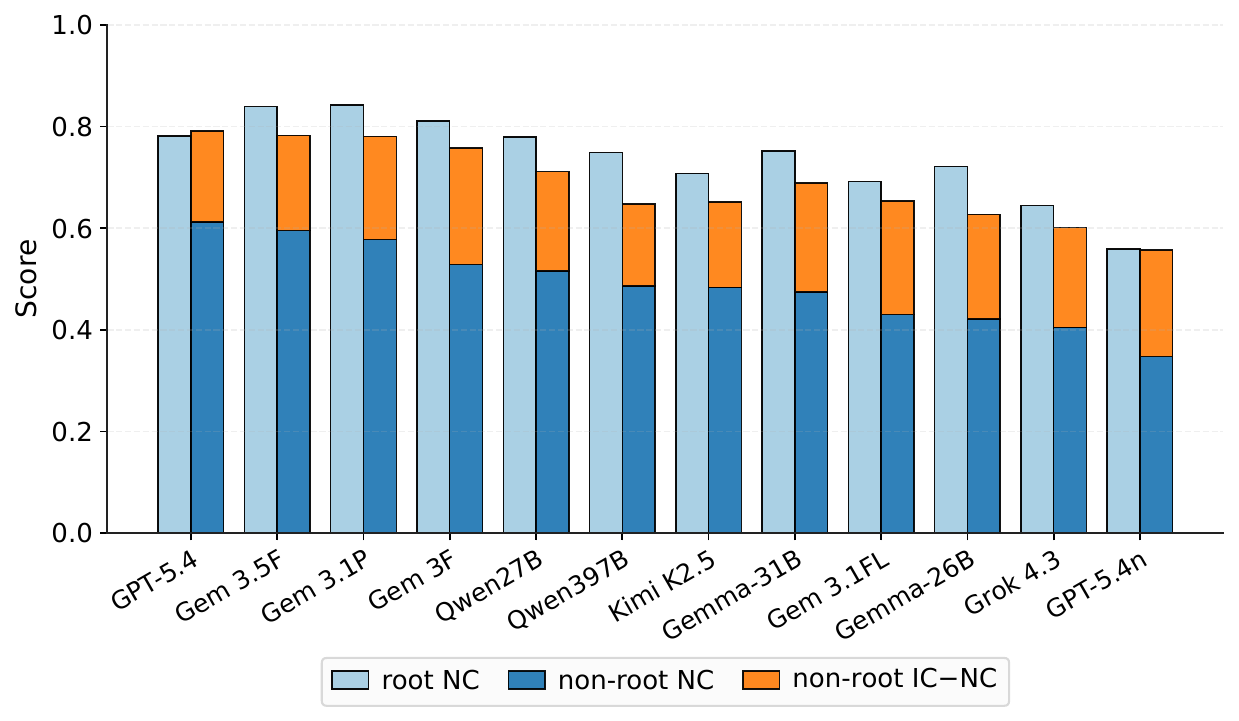}
\caption{Comparing NC on root unit tasks against NC and IC on non-root unit tasks for each of the 12 models. There are 16 root unit tasks and 30 non-root unit tasks. The non-root bar stacks IC$-$NC on top of NC, so its top is IC. Models are ordered by overall non-root NC.}
\label{fig:appd-root-vs-nonroot-model}
\end{figure}

\begin{figure}[t]
\centering
\includegraphics[width=0.3\linewidth]{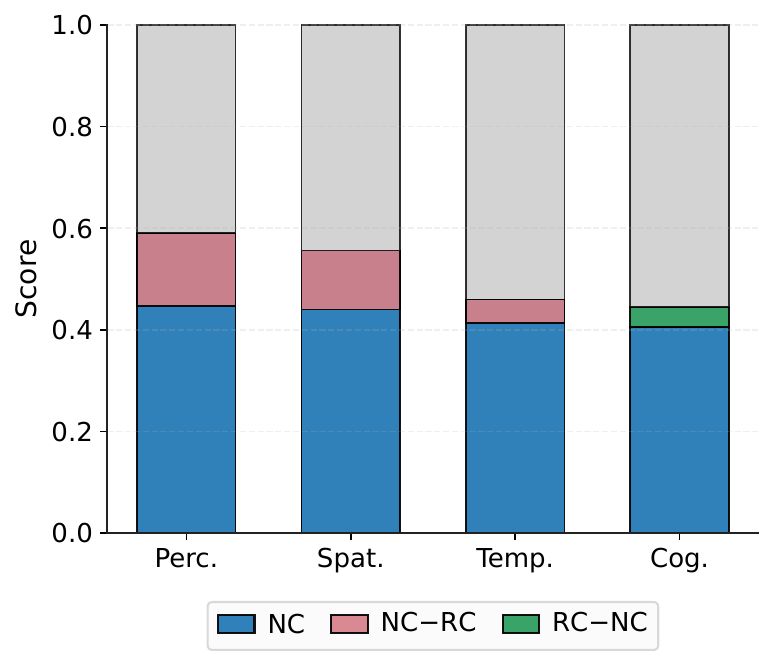}
\caption{Comparing NC and RC across the four categories on non-root unit tasks, averaged over the 12 models. Each bar shows NC with the signed NC/RC difference stacked out of it, so the boundary of the shaded band is RC.}
\label{fig:appd-nc-rc-budget}
\end{figure}
\begin{figure}[t]
\centering
\includegraphics[width=0.7\linewidth]{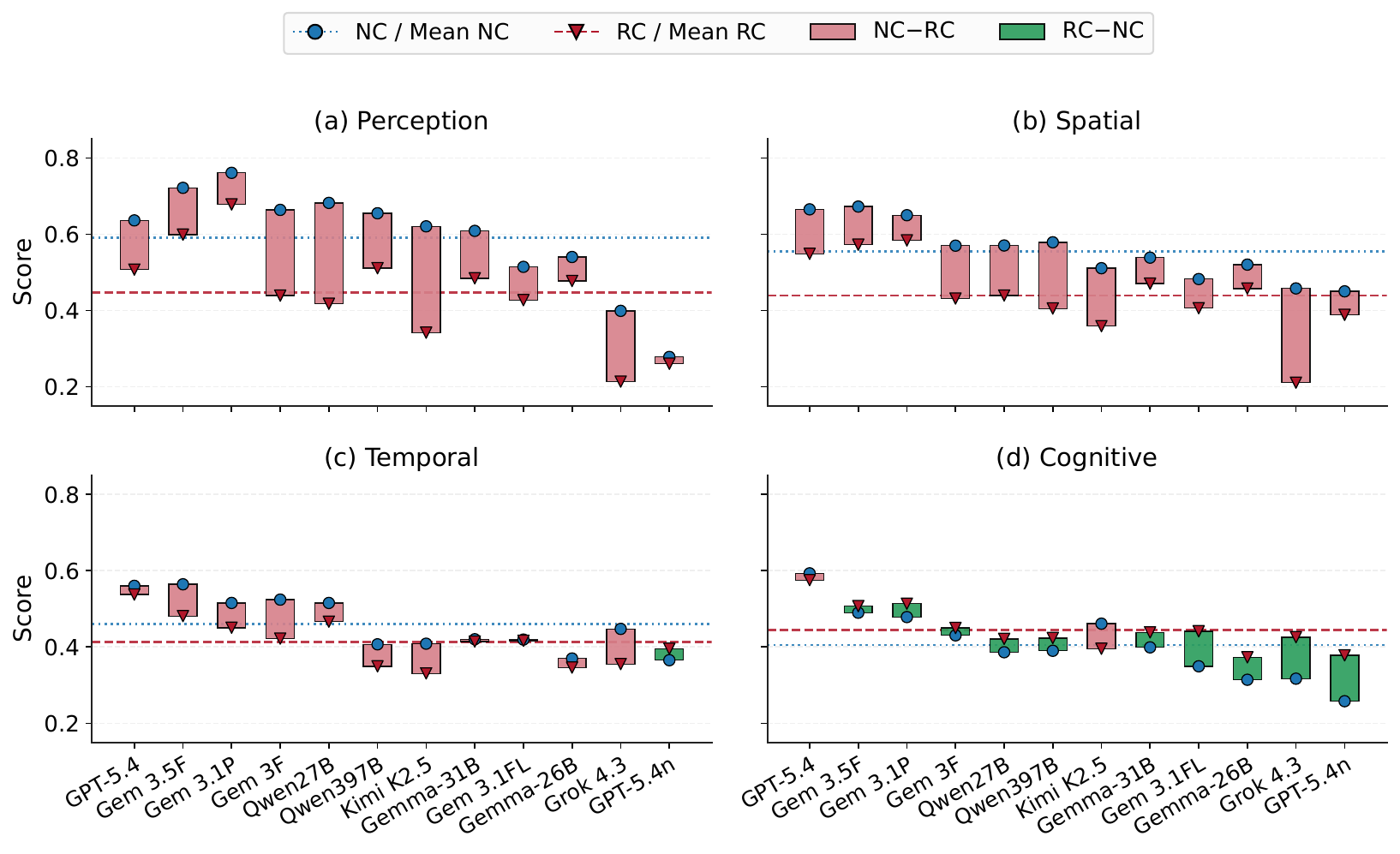}
\caption{Comparing NC and RC for each of the 12 models within each category on non-root unit tasks. Circles mark NC and triangles mark RC; the shaded band spans the two and is colored by the sign of RC $-$ NC. The horizontal lines are the means over the 12 models.}
\label{fig:appd-per-model-nc-rc}
\end{figure}

\begin{table}[t]
\centering
\caption{NC on root unit tasks and NC, IC on non-root unit tasks, per category, with the number of unit tasks in each group. Entries are the unweighted mean over unit tasks, averaged over the 12 models, $\pm$ the standard deviation across models.}
\label{tab:appd-root-vs-nonroot-category}
\small
\resizebox{0.8\linewidth}{!}{%
\begin{tabular}{lcccccc}
\toprule
& \multicolumn{2}{c}{\#unit tasks} & \multicolumn{1}{c}{Root}
& \multicolumn{3}{c}{Non-root} \\
\cmidrule(lr){2-3} \cmidrule(lr){4-4} \cmidrule(lr){5-7}
Category & root & non-root & NC & NC & IC & IC $-$ NC \\
\midrule
Perception & 11 & 4 & 0.775 $\pm$ 0.077 & 0.590 $\pm$ 0.138 & 0.796 $\pm$ 0.112 & +0.206 \\
Spatial & 3 & 9 & 0.627 $\pm$ 0.100 & 0.556 $\pm$ 0.077 & 0.671 $\pm$ 0.059 & +0.115 \\
Temporal & 2 & 8 & 0.714 $\pm$ 0.119 & 0.460 $\pm$ 0.072 & 0.614 $\pm$ 0.087 & +0.154 \\
Cognitive & 0 & 9 & -- & 0.406 $\pm$ 0.092 & 0.721 $\pm$ 0.089 & +0.316 \\
\bottomrule
\end{tabular}}
\end{table}

\begin{table}[t]
\centering
\caption{NC on root unit tasks and NC, IC on non-root unit tasks for each of the 12 models. Entries are the unweighted mean over unit tasks $\pm$ the standard deviation across unit tasks (16 root, 30 non-root). The last two columns are differences of the corresponding means.}
\label{tab:appd-root-vs-nonroot-model}
\small
\resizebox{0.9\linewidth}{!}{%
\begin{tabular}{lccccc}
\toprule
Model & Root NC & Non-root NC & Non-root IC
& NC(root)${-}$NC(non-root) & NC(root)${-}$IC(non-root) \\
\midrule
GPT-5.4 & 0.781 $\pm$ 0.108 & 0.612 $\pm$ 0.180 & 0.791 $\pm$ 0.167 & +0.169 & $-$0.010 \\
Gem 3.5F & 0.839 $\pm$ 0.099 & 0.595 $\pm$ 0.217 & 0.783 $\pm$ 0.172 & +0.244 & +0.056 \\
Gem 3.1P & 0.842 $\pm$ 0.090 & 0.577 $\pm$ 0.202 & 0.781 $\pm$ 0.173 & +0.265 & +0.062 \\
Gem 3F & 0.811 $\pm$ 0.119 & 0.528 $\pm$ 0.206 & 0.757 $\pm$ 0.180 & +0.282 & +0.053 \\
Qwen27B & 0.780 $\pm$ 0.132 & 0.515 $\pm$ 0.247 & 0.711 $\pm$ 0.219 & +0.264 & +0.069 \\
Qwen397B & 0.749 $\pm$ 0.144 & 0.487 $\pm$ 0.240 & 0.647 $\pm$ 0.226 & +0.262 & +0.101 \\
Kimi K2.5 & 0.708 $\pm$ 0.200 & 0.483 $\pm$ 0.199 & 0.652 $\pm$ 0.213 & +0.224 & +0.056 \\
Gemma-31B & 0.752 $\pm$ 0.161 & 0.475 $\pm$ 0.235 & 0.689 $\pm$ 0.202 & +0.278 & +0.063 \\
Gem 3.1FL & 0.692 $\pm$ 0.186 & 0.430 $\pm$ 0.217 & 0.654 $\pm$ 0.209 & +0.262 & +0.038 \\
Gemma-26B & 0.722 $\pm$ 0.196 & 0.421 $\pm$ 0.248 & 0.627 $\pm$ 0.218 & +0.301 & +0.094 \\
Grok 4.3 & 0.644 $\pm$ 0.175 & 0.405 $\pm$ 0.210 & 0.601 $\pm$ 0.205 & +0.239 & +0.043 \\
GPT-5.4n & 0.559 $\pm$ 0.239 & 0.347 $\pm$ 0.190 & 0.557 $\pm$ 0.236 & +0.212 & +0.002 \\
\bottomrule
\end{tabular}}
\end{table}

\begin{table}[t]
\centering
\caption{NC, IC and RC per category on the 30 non-root unit tasks. Entries are the unweighted mean over unit tasks, averaged over the 12 models, $\pm$ the standard deviation across models.}
\label{tab:appd-category-nc-ic-rc}
\small
\resizebox{0.8\linewidth}{!}{%
\begin{tabular}{lccccc}
\toprule
Category & NC & IC & RC & IC $-$ NC & RC $-$ NC \\
\midrule
Perception & 0.590 $\pm$ 0.138 & 0.796 $\pm$ 0.112 & 0.446 $\pm$ 0.131 & +0.206 & $-$0.144 \\
Spatial & 0.556 $\pm$ 0.077 & 0.671 $\pm$ 0.059 & 0.439 $\pm$ 0.102 & +0.115 & $-$0.116 \\
Temporal & 0.460 $\pm$ 0.072 & 0.614 $\pm$ 0.087 & 0.414 $\pm$ 0.063 & +0.154 & $-$0.046 \\
Cognitive & 0.406 $\pm$ 0.092 & 0.721 $\pm$ 0.089 & 0.445 $\pm$ 0.060 & +0.316 & +0.040 \\
\bottomrule
\end{tabular}}
\end{table}

\begin{table}[t]
\centering
\caption{NC, IC and RC per model and category on the 30 non-root unit tasks, with \emph{Overall} aggregating all 30. Entries are the unweighted mean over unit tasks $\pm$ the standard deviation across unit tasks.}
\label{tab:appd-model-category-nc-ic-rc}
\small
\providecommand{\cell}[2]{\begin{tabular}[c]{@{}c@{}}#1\\[-1pt]{\scriptsize\color{gray}\itshape$\pm$#2}\end{tabular}}
\setlength{\tabcolsep}{3.2pt}
\resizebox{\linewidth}{!}{%
\begin{tabular}{l*{15}{c}}
\toprule
& \multicolumn{3}{c}{Perception} & \multicolumn{3}{c}{Spatial} & \multicolumn{3}{c}{Temporal} & \multicolumn{3}{c}{Cognitive} & \multicolumn{3}{c}{Overall} \\
\cmidrule(lr){2-4} \cmidrule(lr){5-7} \cmidrule(lr){8-10} \cmidrule(lr){11-13} \cmidrule(lr){14-16}
Model & NC & IC & RC & NC & IC & RC & NC & IC & RC & NC & IC & RC & NC & IC & RC \\
\midrule
GPT-5.4 & \cell{0.636}{0.194} & \cell{0.913}{0.087} & \cell{0.507}{0.288} & \cell{0.665}{0.202} & \cell{0.747}{0.214} & \cell{0.549}{0.213} & \cell{0.560}{0.148} & \cell{0.730}{0.181} & \cell{0.537}{0.219} & \cell{0.593}{0.192} & \cell{0.835}{0.094} & \cell{0.575}{0.274} & \cell{0.612}{0.180} & \cell{0.791}{0.167} & \cell{0.548}{0.232} \\
Gem 3.5F & \cell{0.722}{0.192} & \cell{0.889}{0.042} & \cell{0.599}{0.194} & \cell{0.672}{0.231} & \cell{0.754}{0.213} & \cell{0.572}{0.252} & \cell{0.564}{0.154} & \cell{0.699}{0.202} & \cell{0.481}{0.187} & \cell{0.490}{0.232} & \cell{0.839}{0.091} & \cell{0.507}{0.219} & \cell{0.595}{0.217} & \cell{0.783}{0.172} & \cell{0.532}{0.212} \\
Gem 3.1P & \cell{0.761}{0.155} & \cell{0.907}{0.027} & \cell{0.678}{0.158} & \cell{0.650}{0.217} & \cell{0.740}{0.230} & \cell{0.584}{0.207} & \cell{0.515}{0.164} & \cell{0.698}{0.175} & \cell{0.450}{0.222} & \cell{0.478}{0.174} & \cell{0.839}{0.088} & \cell{0.513}{0.194} & \cell{0.577}{0.202} & \cell{0.781}{0.173} & \cell{0.540}{0.206} \\
Gem 3F & \cell{0.664}{0.250} & \cell{0.879}{0.072} & \cell{0.439}{0.304} & \cell{0.570}{0.268} & \cell{0.710}{0.253} & \cell{0.431}{0.240} & \cell{0.524}{0.122} & \cell{0.707}{0.163} & \cell{0.422}{0.205} & \cell{0.430}{0.154} & \cell{0.796}{0.121} & \cell{0.450}{0.223} & \cell{0.528}{0.206} & \cell{0.757}{0.180} & \cell{0.435}{0.222} \\
Qwen27B & \cell{0.682}{0.167} & \cell{0.873}{0.119} & \cell{0.418}{0.237} & \cell{0.571}{0.308} & \cell{0.663}{0.252} & \cell{0.439}{0.249} & \cell{0.515}{0.248} & \cell{0.672}{0.210} & \cell{0.466}{0.276} & \cell{0.386}{0.159} & \cell{0.722}{0.220} & \cell{0.421}{0.213} & \cell{0.515}{0.247} & \cell{0.711}{0.219} & \cell{0.438}{0.233} \\
Qwen397B & \cell{0.655}{0.126} & \cell{0.808}{0.167} & \cell{0.511}{0.184} & \cell{0.579}{0.287} & \cell{0.677}{0.269} & \cell{0.405}{0.217} & \cell{0.407}{0.164} & \cell{0.480}{0.150} & \cell{0.349}{0.196} & \cell{0.390}{0.238} & \cell{0.695}{0.194} & \cell{0.423}{0.246} & \cell{0.487}{0.240} & \cell{0.647}{0.226} & \cell{0.410}{0.212} \\
Kimi K2.5 & \cell{0.621}{0.112} & \cell{0.698}{0.127} & \cell{0.342}{0.161} & \cell{0.511}{0.245} & \cell{0.659}{0.300} & \cell{0.359}{0.231} & \cell{0.409}{0.156} & \cell{0.586}{0.175} & \cell{0.331}{0.227} & \cell{0.461}{0.204} & \cell{0.682}{0.187} & \cell{0.395}{0.138} & \cell{0.483}{0.199} & \cell{0.652}{0.213} & \cell{0.360}{0.189} \\
Gemma-31B & \cell{0.609}{0.291} & \cell{0.825}{0.035} & \cell{0.485}{0.272} & \cell{0.539}{0.275} & \cell{0.679}{0.279} & \cell{0.470}{0.247} & \cell{0.420}{0.182} & \cell{0.625}{0.163} & \cell{0.414}{0.236} & \cell{0.399}{0.200} & \cell{0.696}{0.182} & \cell{0.438}{0.217} & \cell{0.475}{0.235} & \cell{0.689}{0.202} & \cell{0.448}{0.227} \\
Gem 3.1FL & \cell{0.514}{0.294} & \cell{0.756}{0.117} & \cell{0.427}{0.264} & \cell{0.482}{0.221} & \cell{0.627}{0.260} & \cell{0.406}{0.207} & \cell{0.419}{0.177} & \cell{0.557}{0.230} & \cell{0.416}{0.232} & \cell{0.349}{0.220} & \cell{0.720}{0.141} & \cell{0.441}{0.212} & \cell{0.430}{0.217} & \cell{0.654}{0.209} & \cell{0.422}{0.211} \\
Gemma-26B & \cell{0.540}{0.351} & \cell{0.791}{0.115} & \cell{0.477}{0.301} & \cell{0.520}{0.287} & \cell{0.632}{0.285} & \cell{0.458}{0.260} & \cell{0.370}{0.171} & \cell{0.532}{0.212} & \cell{0.346}{0.196} & \cell{0.314}{0.186} & \cell{0.635}{0.154} & \cell{0.373}{0.216} & \cell{0.421}{0.248} & \cell{0.627}{0.218} & \cell{0.405}{0.230} \\
Grok 4.3 & \cell{0.399}{0.237} & \cell{0.659}{0.096} & \cell{0.214}{0.223} & \cell{0.458}{0.249} & \cell{0.594}{0.273} & \cell{0.211}{0.161} & \cell{0.447}{0.215} & \cell{0.585}{0.220} & \cell{0.355}{0.218} & \cell{0.317}{0.150} & \cell{0.597}{0.174} & \cell{0.426}{0.215} & \cell{0.405}{0.210} & \cell{0.601}{0.205} & \cell{0.314}{0.213} \\
GPT-5.4n & \cell{0.278}{0.250} & \cell{0.551}{0.331} & \cell{0.260}{0.234} & \cell{0.450}{0.204} & \cell{0.571}{0.232} & \cell{0.389}{0.202} & \cell{0.365}{0.176} & \cell{0.498}{0.233} & \cell{0.395}{0.222} & \cell{0.258}{0.120} & \cell{0.598}{0.233} & \cell{0.378}{0.166} & \cell{0.347}{0.190} & \cell{0.557}{0.236} & \cell{0.370}{0.196} \\
\bottomrule
\end{tabular}}
\end{table}

This section provides additional visualizations and the complete numerical results supporting the capability diagnostics in Section~5.2. First, whereas Section~\ref{subsec:exp_cap} focuses on the 30 non-root unit tasks ($\text{Pa}(X_i) \neq \emptyset$) that admit upstream interventions, Figure~\ref{fig:appd-root-vs-nonroot-category},\ref{fig:appd-root-vs-nonroot-model} and Table~\ref{tab:appd-root-vs-nonroot-model},\ref{tab:appd-root-vs-nonroot-category} incorporate the 16 foundational root unit tasks ($\text{Pa}(X_i) = \emptyset$, for which only natural capability $\text{NC}$ is defined), comparing root $\text{NC}$ against non-root $\text{NC}$ and $\text{IC}$ across individual models (Figure~\ref{fig:appd-root-vs-nonroot-model}, Table~\ref{tab:appd-root-vs-nonroot-model}) and capability categories (Figure.~\ref{fig:appd-root-vs-nonroot-category}, Table~\ref{tab:appd-root-vs-nonroot-category}). Second, Figure~\ref{fig:appd-nc-rc-budget},\ref{fig:appd-per-model-nc-rc} and Table~\ref{tab:appd-category-nc-ic-rc},\ref{tab:appd-model-category-nc-ic-rc} expand the compact $\text{RC}-\text{NC}$ summary in Figure~4 of the main text by plotting absolute $\text{NC}$ and $\text{RC}$ levels alongside their signed differences at the category level (Figure~\ref{fig:appd-nc-rc-budget}, Table~\ref{tab:appd-category-nc-ic-rc}) and across all 12 models within each category (Figure~\ref{fig:appd-per-model-nc-rc}, Table~\ref{tab:appd-model-category-nc-ic-rc}), reporting the full per-group means and standard deviations.

\noindent\textbf{Comparing Non-Root Capabilities ($\text{NC}$, $\text{IC}$) with Foundational Root Capabilities ($\text{NC}$).}
Because root unit tasks have no modeled upstream prerequisites in the task SCM ($\text{Pa}(X_i) = \emptyset$), their natural capability ($\text{NC}$) assesses foundational proficiency independently of upstream prerequisite states. Under unassisted evaluation ($\text{do}(\mathbf{A}_{\text{Pa}(X_i)} = \emptyset)$), models achieve higher accuracy on the 16 root unit tasks than on the 30 downstream non-root unit tasks: as reported in Table~\ref{tab:appd-root-vs-nonroot-model} and Figure~\ref{fig:appd-root-vs-nonroot-model}, mean $\text{NC}$ across the 12 models is $0.740$ on root nodes versus $0.490$ on non-root nodes (a difference of $+0.250$, ranging from $+0.169$ to $+0.301$ across individual models). A consistent $\text{NC}$ difference between root and non-root nodes appears within each category that contains root tasks (Table~\ref{tab:appd-root-vs-nonroot-category}, Figure~\ref{fig:appd-root-vs-nonroot-category}): Perception ($0.775$ vs.\ $0.590$), Spatial ($0.627$ vs.\ $0.556$), and Temporal ($0.714$ vs.\ $0.460$), while Cognitive unit tasks occur exclusively as non-root nodes ($\text{Pa}(X_i) \neq \emptyset$) and score lowest under unassisted evaluation ($\text{NC} = 0.406$). When correct prerequisite answers are supplied ($\text{do}(\mathbf{A}_{\text{Pa}(X_i)} = 1)$), mean non-root $\text{IC}$ across models rises to $0.688$ against the $0.740$ root $\text{NC}$ baseline 
(narrowing the overall gap from $0.250$ to $0.052$, with differences as small as $-0.010$ for GPT-5.4 and $0.002$ for GPT-5.4n). 
Across categories, this comparison reinforces the capability profile in Section~5.2: Perception maintains high accuracy on both root tasks ($\text{NC} = 0.775$) and prerequisite-supplied non-root tasks ($\text{IC} = 0.796$), and Cognitive non-root tasks rise from $\text{NC} = 0.406$ to $\text{IC} = 0.721$, approaching the average root-node baseline ($0.740$). By contrast, Spatial exhibits the lowest root-node accuracy ($\text{NC} = 0.627$) and reaches a similarly bounded ceiling on non-root tasks when prerequisites are supplied ($\text{IC} = 0.671$), while Temporal non-root tasks reach $\text{IC} = 0.614$ (vs.\ $0.714$ on root tasks), consistent with the persistent residual gaps ($1-\text{IC}$) in spatial and temporal capabilities identified in the main text.

\noindent\textbf{Expanded Breakdown of Resilience Capability ($\text{RC}$ vs.\ $\text{NC}$).}
Figure~\ref{fig:appd-nc-rc-budget} and Table~\ref{tab:appd-category-nc-ic-rc} complement Figure~4 by displaying $\text{NC}$ and $\text{RC}$ on the same absolute capability scale across categories, while Figure~\ref{fig:appd-per-model-nc-rc} and Table~\ref{tab:appd-model-category-nc-ic-rc} provide the disaggregated per-model profiles parallel to Figure~3(b--e). On average across the 12 models (Table~\ref{tab:appd-category-nc-ic-rc}), injecting incorrect prerequisite answers ($\text{do}(\mathbf{A}_{\text{Pa}(X_i)} = 0)$) yields lower accuracy than unassisted evaluation in Perception ($0.590 \to 0.446$, $\text{RC}-\text{NC} = -0.144$), Spatial ($0.556 \to 0.439$, $-0.116$), and Temporal ($0.460 \to 0.414$, $-0.046$), corresponding to an overall non-root difference of $-0.054$ ($0.490 \to 0.435$), whereas Cognitive tasks exhibit a modest positive shift ($0.406 \to 0.445$, $+0.040$). The per-model breakdown in Figure~\ref{fig:appd-per-model-nc-rc} and Table~\ref{tab:appd-model-category-nc-ic-rc} shows that these category-level trends hold consistently across model families and scales: all 12 models exhibit $\text{RC} < \text{NC}$ in Perception and Spatial, 11 of 12 decline in Temporal (with GPT-5.4n showing a $+0.030$ shift), and 10 of 12 show positive $\text{RC}-\text{NC}$ shifts in Cognitive (all except GPT-5.4 at $-0.018$ and Kimi K2.5 at $-0.066$), providing the full per-model evidence for the non-monotonicity discussed in Property~3.6 and Section~5.2.

\subsection{More Causal Contribution Diagnosis Results ($N$, $S$-Score)}
\label{app:contribution}

In this section, we provide additional experimental results complementing the causal contribution analysis in Section~5.3.
Figure~\ref{fig:app_e_quadrants} reports prerequisite $N$- and $S$-Scores under the remaining three state contrasts $(1,0)$, $(0,\varnothing)$, and $(\varnothing,0)$ (complementing $(1,\varnothing)$ in Figure~\ref{fig:finding_contribution_quadrant_1_empty}).
Figure~\ref{fig:app_e_pattern} disaggregates these three contrasts by task pattern and model group (extending Figure~\ref{fig:finding_contrib_4panel_quadrants}).
Figure~\ref{fig:app_e_similarity} shows the inter-model rank agreement across all eight contribution scores, and Figure~\ref{fig:app_e_parent} provides the per-prerequisite breakdown across the four state contrasts.
To prevent small denominators from producing numerically unstable ratios in Eqs.~(5)--(6), we set a score to $0$ whenever its baseline denominator falls below $0.10$.
Crucially, this threshold has zero effect on six of the eight metrics and affects only $3.0\%$ ($94/3{,}168$) of all scores, concentrated in $N(0,\varnothing)$ (with a small remainder in $N(\varnothing,0)$).
This is expected by design: $N(0,\varnothing)$ measures the success drop from setting a prerequisite to $\varnothing$ in the all-incorrect ($0$) environment, where baseline success $\mathrm{RC}(Y)$ is naturally near zero whenever incorrect prerequisites prevent task completion.

\begin{figure}[t]
  \centering
  \includegraphics[width=0.8\linewidth]{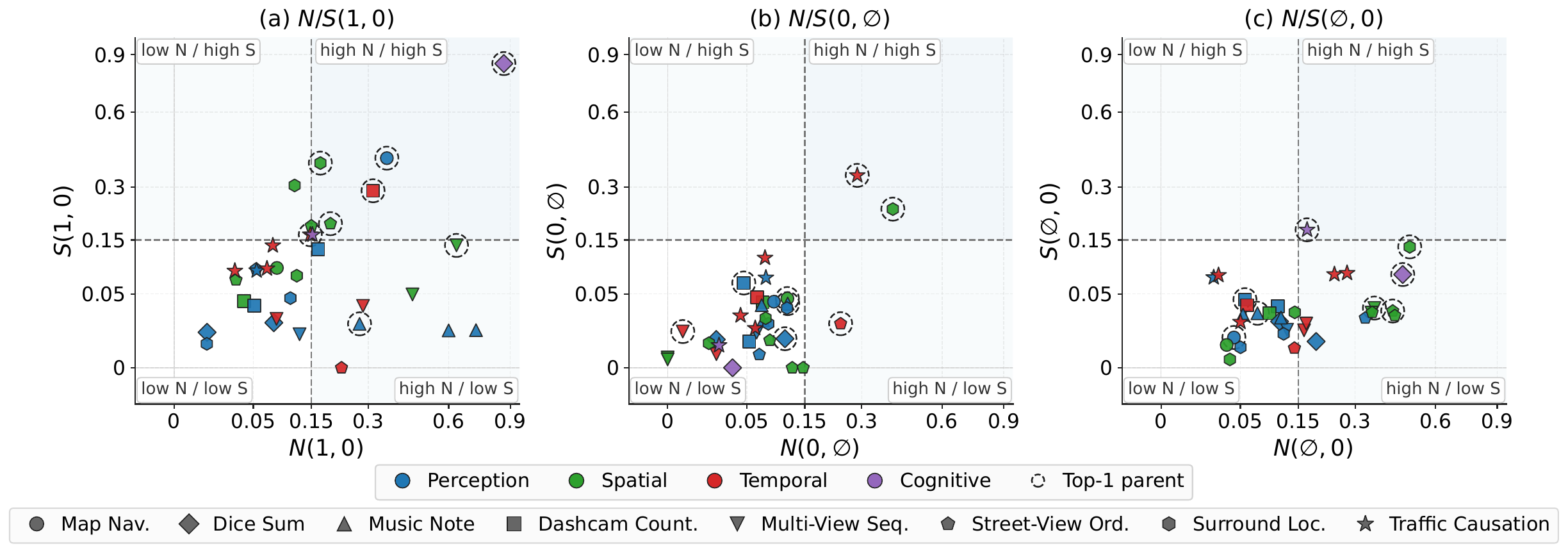}
  \caption{Prerequisite $N/S$-Scores under state contrasts (a) $(1,0)$, (b) $(0,\varnothing)$, and (c) $(\varnothing,0)$, averaged across models. Dashed circles mark the highest-$S$ prerequisite within each task. Axes and quadrant thresholds follow Figure~\ref{fig:finding_contribution_quadrant_1_empty}.}
  \label{fig:app_e_quadrants}
\end{figure}
\noindent\textbf{Corrupting a prerequisite reduces success more than leaving it unassisted.}
$N(1,0)$ and $N(1,\varnothing)$ share the same oracle baseline where all prerequisites are correctly provided, differing only in whether $X_i$ receives an incorrect answer or is left unassisted.
Across prerequisites, $N(1,0)$ consistently exceeds $N(1,\varnothing)$ (Figure~\ref{fig:app_e_quadrants}a vs.\ Figure~\ref{fig:finding_contribution_quadrant_1_empty}), shifting more prerequisites into the high-$N$ / low-$S$ quadrant and fewer into the low-$N$ / high-$S$ quadrant.
For example, on the perception prerequisites of Music Note, $N(1,\varnothing)$ is low ($0.04$--$0.16$) and $S(1,\varnothing)$ is high ($0.51$--$0.59$): omitting a single prerequisite in the oracle setting incurs minimal loss, while supplying it alone in the natural state recovers over half of the failures.
Under $(1,0)$, however, $N(1,0)$ rises to $0.27$--$0.73$ and $S(1,0)$ falls to at most $0.02$: a single incorrect prerequisite removes a substantial fraction of oracle success, while correcting one prerequisite when others remain incorrect recovers almost no failures.
Similarly, under $(\varnothing,0)$ (Figure~\ref{fig:app_e_quadrants}c), prerequisites shift along $N(\varnothing,0)$ into the high-$N$ / low-$S$ region, whereas under $(0,\varnothing)$ (Figure~\ref{fig:app_e_quadrants}b), most prerequisites remain near zero with positive values concentrated in specific prerequisites examined in Figure~\ref{fig:app_e_parent}.

\begin{figure}[t]
  \centering
  \includegraphics[width=0.9\linewidth]{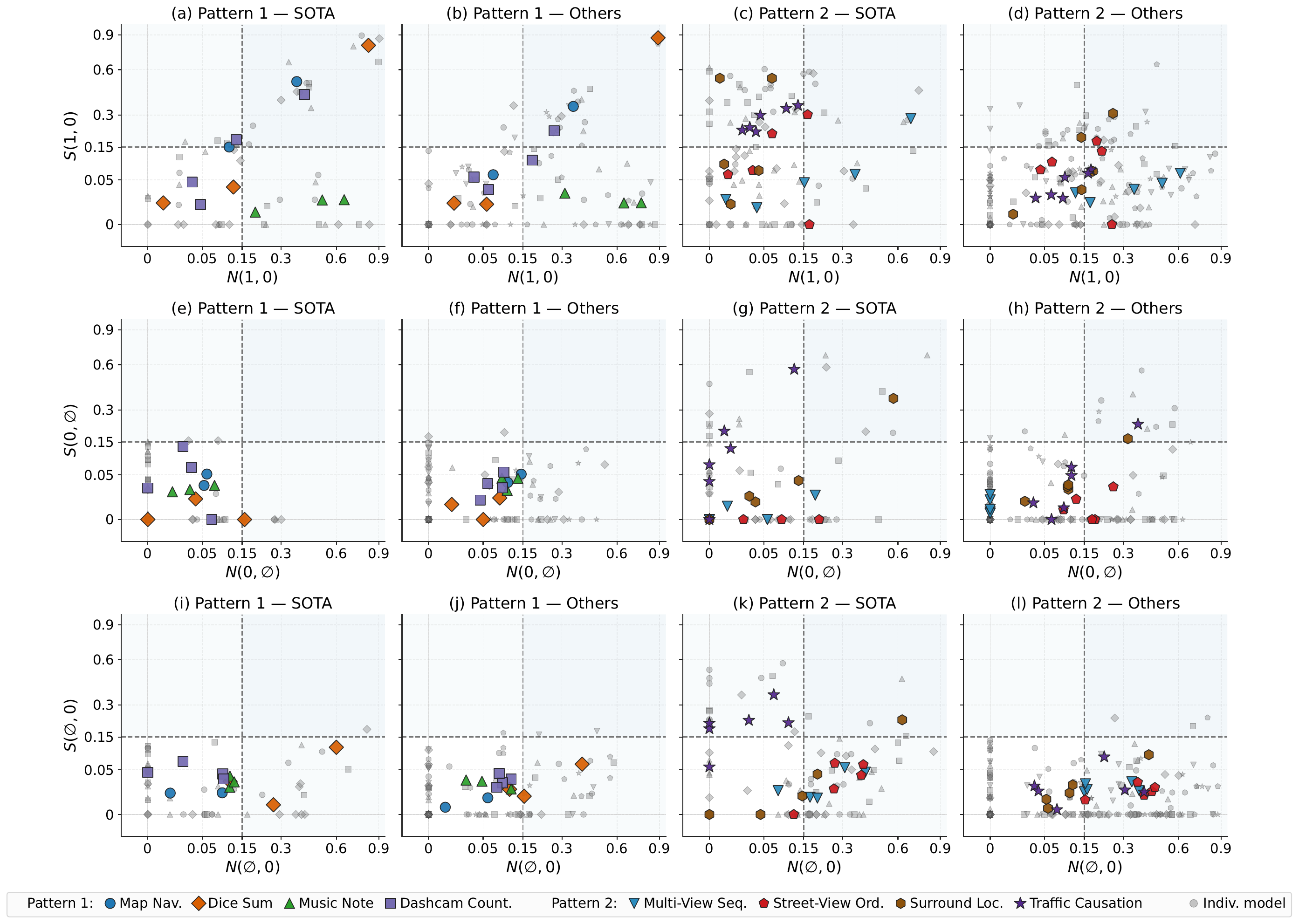}
  \caption{Prerequisite $N/S$-Scores under $(1,0)$ (a--d), $(0,\varnothing)$ (e--h), and $(\varnothing,0)$ (i--l) by task pattern and model group. Large markers denote group means; grey markers indicate individual models.}
  \label{fig:app_e_pattern}
\end{figure}
\noindent\textbf{Separation between SOTA and Others under prerequisite corruption.}
When disaggregated by task pattern and model group (Figure~\ref{fig:app_e_pattern}), SOTA and Others separate clearly in Pattern~2 across all three contrasts involving incorrect prerequisites: $S(1,0)$, $S(\varnothing,0)$, and $N(1,0)$ (Figure~\ref{fig:app_e_pattern}c,d,k,l).
Starting from the all-incorrect environment, SOTA achieves consistently higher single-prerequisite recovery than Others both when correcting a single prerequisite ($S(1,0)$: $0.143$--$0.263$ for SOTA vs.\ $0.029$--$0.105$ for Others) and when removing a single incorrect answer ($S(\varnothing,0)$: $0.044$--$0.152$ vs.\ $0.015$--$0.028$).
Symmetrically, starting from the oracle environment, corrupting a single prerequisite causes smaller drops for SOTA ($N(1,0)$: $0.097$--$0.106$) than for Others ($0.108$--$0.295$), mirroring the shift toward higher $N$ and lower $S$ for Others under $(1,\varnothing)$ in Figure~\ref{fig:finding_contrib_4panel_quadrants}.

\begin{figure}[t]
  \centering
  \includegraphics[width=0.8\linewidth]{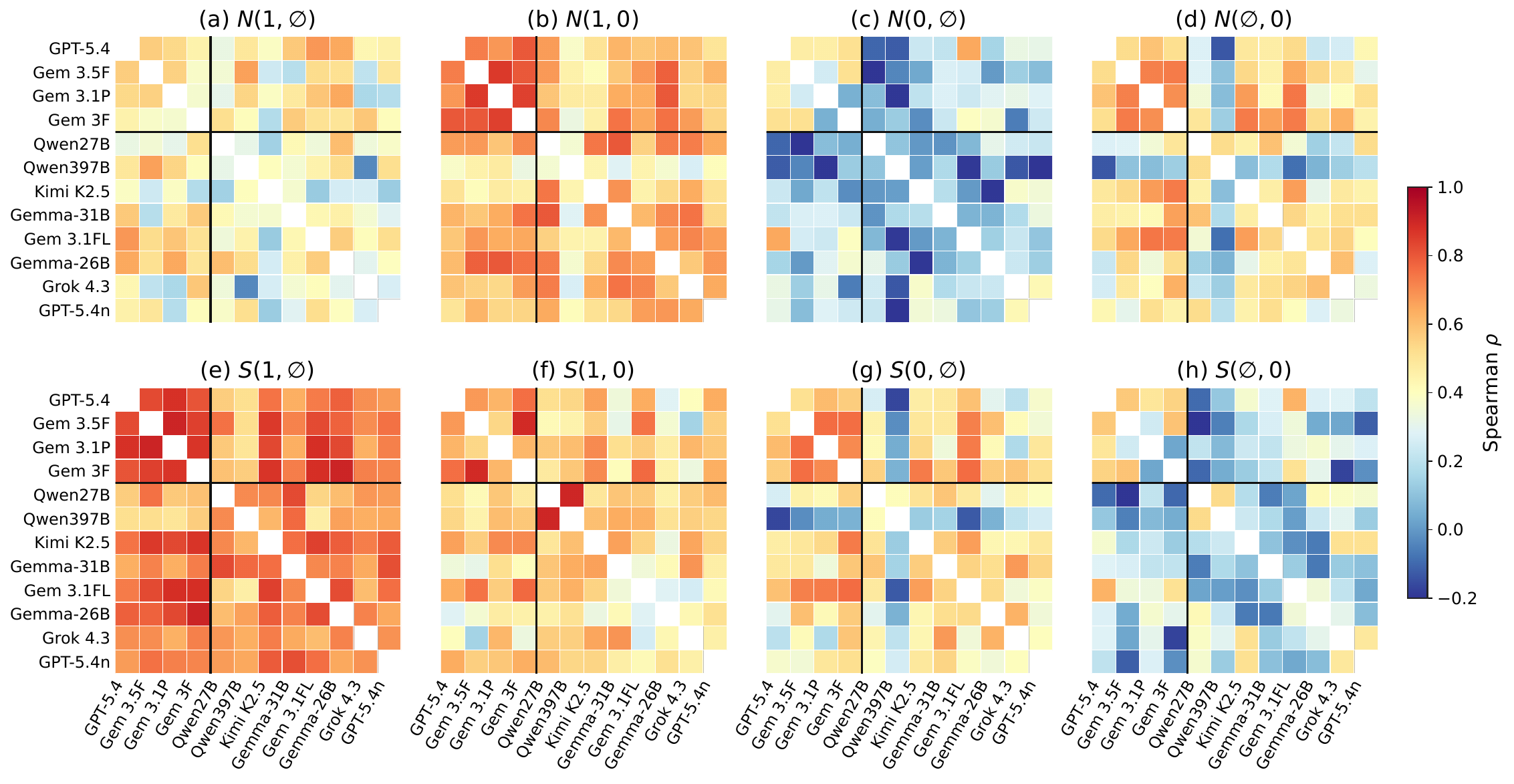}
  \caption{Pairwise Spearman rank correlation between models on $N$-Scores (a--d) and $S$-Scores (e--h) across prerequisites under the four state contrasts. Lines separate SOTA from Others.}
  \label{fig:app_e_similarity}
\end{figure}
\noindent\textbf{Inter-model agreement on high-$S(1,\varnothing)$ prerequisites.}
Among the eight contribution metrics, $S(1,\varnothing)$ exhibits the strongest inter-model agreement, with a mean pairwise Spearman correlation of $0.72$, compared with $0.41$ for $N(1,\varnothing)$ (Figure~\ref{fig:app_e_similarity}).
Thus, the high-leverage prerequisites that recover natural failures when supplied alone are largely shared across models, whereas sensitivity to single-prerequisite omission in the oracle setting varies more across models.
Agreement on $S(1,\varnothing)$ is especially strong among the four SOTA models (mean $\rho = 0.85$ vs.\ $0.70$ among Others), holding both within the Gemini series ($0.87$) and across families between GPT-5.4 and Gemini ($0.84$), and within individual tasks the top-$S(1,\varnothing)$ prerequisite is identical or shared between the same top pair across models on the majority of tasks.

\begin{figure}[t]
  \centering
  \includegraphics[width=\linewidth]{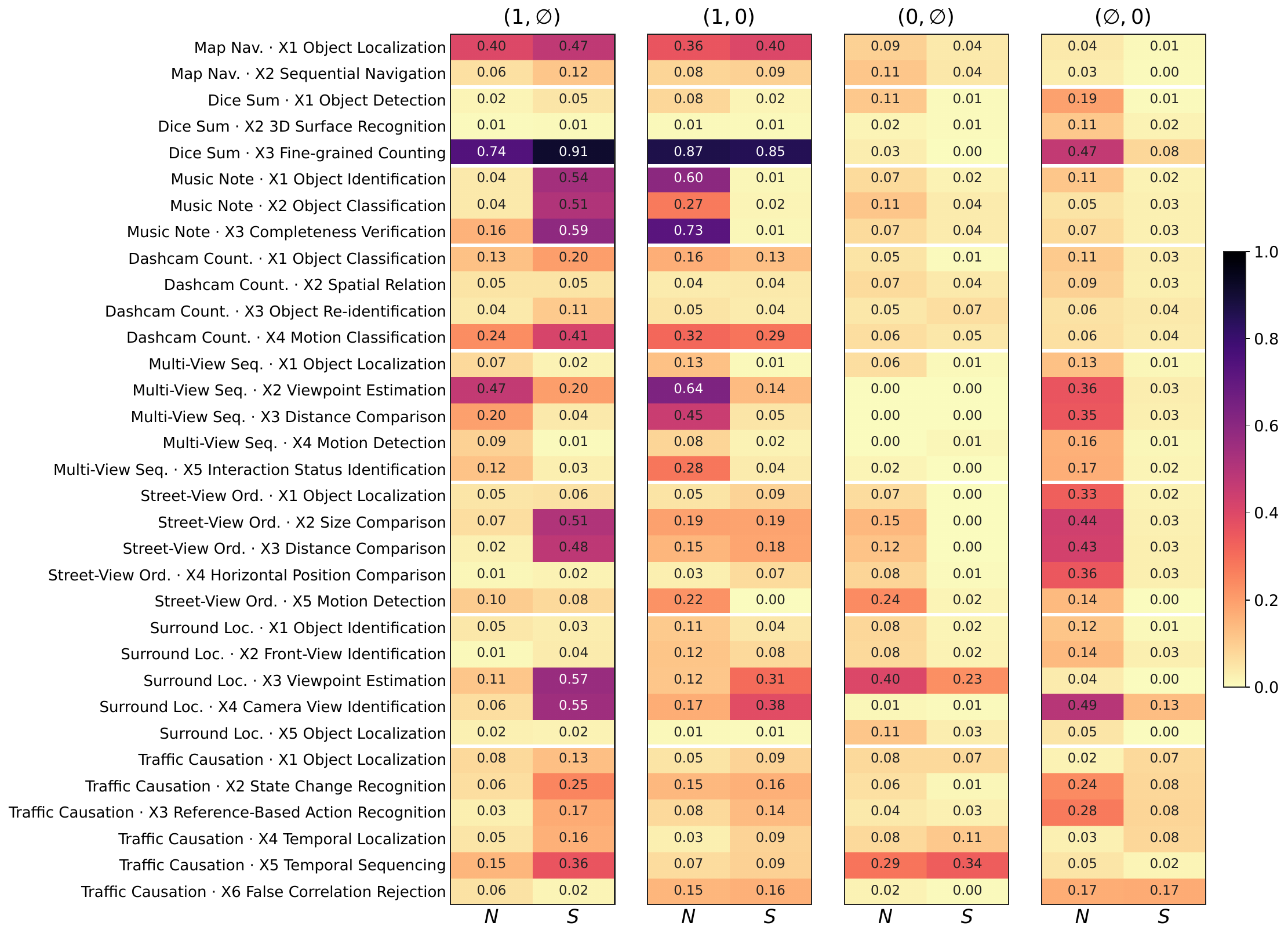}
  \caption{Prerequisite $N/S$-Scores under the four state contrasts, averaged over evaluated models.}
  \label{fig:app_e_parent}
\end{figure}
\noindent\textbf{Cross-contrast consistency and per-prerequisite breakdown.}
Across prerequisites, rankings correlate positively between $N(1,\varnothing)$ and $N(1,0)$ (Spearman $\rho = 0.64$) and between $S(1,\varnothing)$ and $S(1,0)$ ($\rho = 0.53$; Figure~\ref{fig:app_e_parent}), indicating that prerequisites whose omission most reduces oracle performance also tend to be the most destructive when corrupted, and those that best recover natural failures also drive recovery in the all-incorrect setting.
For transitions between $\varnothing$ and $0$, $(\varnothing,0)$ and $(0,\varnothing)$ measure opposite directions of the same state pair; under zero-clipping, at most one direction is active for a given model and prerequisite, separating performance degradation from incorrect prerequisites ($N(\varnothing,0), S(\varnothing,0)$) from positive recovery when an incorrect prerequisite is provided over no prerequisite ($S(0,\varnothing), N(0,\varnothing)$) under Property~3.6.
Consistent with Section~5.2 and Figure~9, positive $S(0,\varnothing)$ values appear primarily on prerequisites feeding into downstream cognitive or spatial-relational integration—such as Temporal Sequencing ($0.34$), Temporal Localization ($0.11$), and Object Localization ($0.07$) in Traffic Causation, and Viewpoint Estimation ($0.23$) in Surround Loc.—while remaining near zero on foundational perception and spatial prerequisites.

\subsection{Detail Explanation of Markov Chain Unit Task Results}
\label{app:markov_results}

Section~5.4 analyzes the three Markov chain unit tasks in CADET (Figure~\ref{fig:finding_markov_causal_profiles}): \textit{Maze} and \textit{Multistep Reasoning}, whose primary task structure is a Markov chain, together with the route-planning unit task embedded within \textit{Map Navigation}. While all three share the same Markov sequential dependency structure (Definition~\ref{def:markov_chain}), they differ in their concrete step-wise logic, and because this distinction directly manifests in their empirical behavior, we group them into the two regimes discussed in Section~5.4: \emph{goal-directed} and \emph{process-cumulative}. In \textit{Maze}, the trajectory from the start to the exit is partitioned into $T{=}10$ steps; because the overall query (identifying which candidate is the exit) cannot be further decomposed into distinct sub-questions, at each step $t$ the model is positioned at the $t$-th waypoint to answer this same exit-identification question. Similarly, in the Markov chain node of \textit{Map Navigation}, the model stands at each intermediate step along the route and looks toward a given destination to decide the next move. Both tasks are therefore goal-directed: conditioned on the current step $t$, solving the task depends only on the remaining path toward the goal rather than the history already traversed. By contrast, \textit{Multistep Reasoning} is process-cumulative: starting from an initial state without a pre-specified destination, the model follows movement rules where each step alternates between identifying the current \textit{Position} and determining the current \textit{Action}. Because the terminal state is unknown in advance, accurately resolving either node at step $t$ requires tracking the cumulative execution process from the starting point up to step $t$.

To operationalize the ternary intervention at step $t$ along an unfolded chain, we provide the full preceding trajectory context from step $1$ to $t{-}1$ in the prompt. Under $A{=}1$, we inject the ground-truth questions, answers, and corresponding visual annotations for all preceding steps. A special case arises in \textit{Maze}: because every step asks the same exit-identification question and shares the identical ground-truth answer, injecting the textual answer of any preceding step would directly reveal the answer at step $t$. For $A{=}1$ in \textit{Maze}, we therefore inject the preceding step questions without their textual answers, accompanied by the visual annotation on the image marking the current waypoint and the human-annotated path from the start up to step $t{-}1$. Under $A{=}0$, each preceding step from $1$ to $t{-}1$ is independently assigned a randomly sampled incorrect answer from the valid answer space via text only (Section~5.1), while $A{=}\varnothing$ leaves the preceding trajectory unassisted.

\subsection{Additional Experiments and Findings}

\begin{figure}[t]
  \centering
  \begin{minipage}[t]{0.6\textwidth}
    \vspace{0pt}
    \centering
    \includegraphics[width=\linewidth]{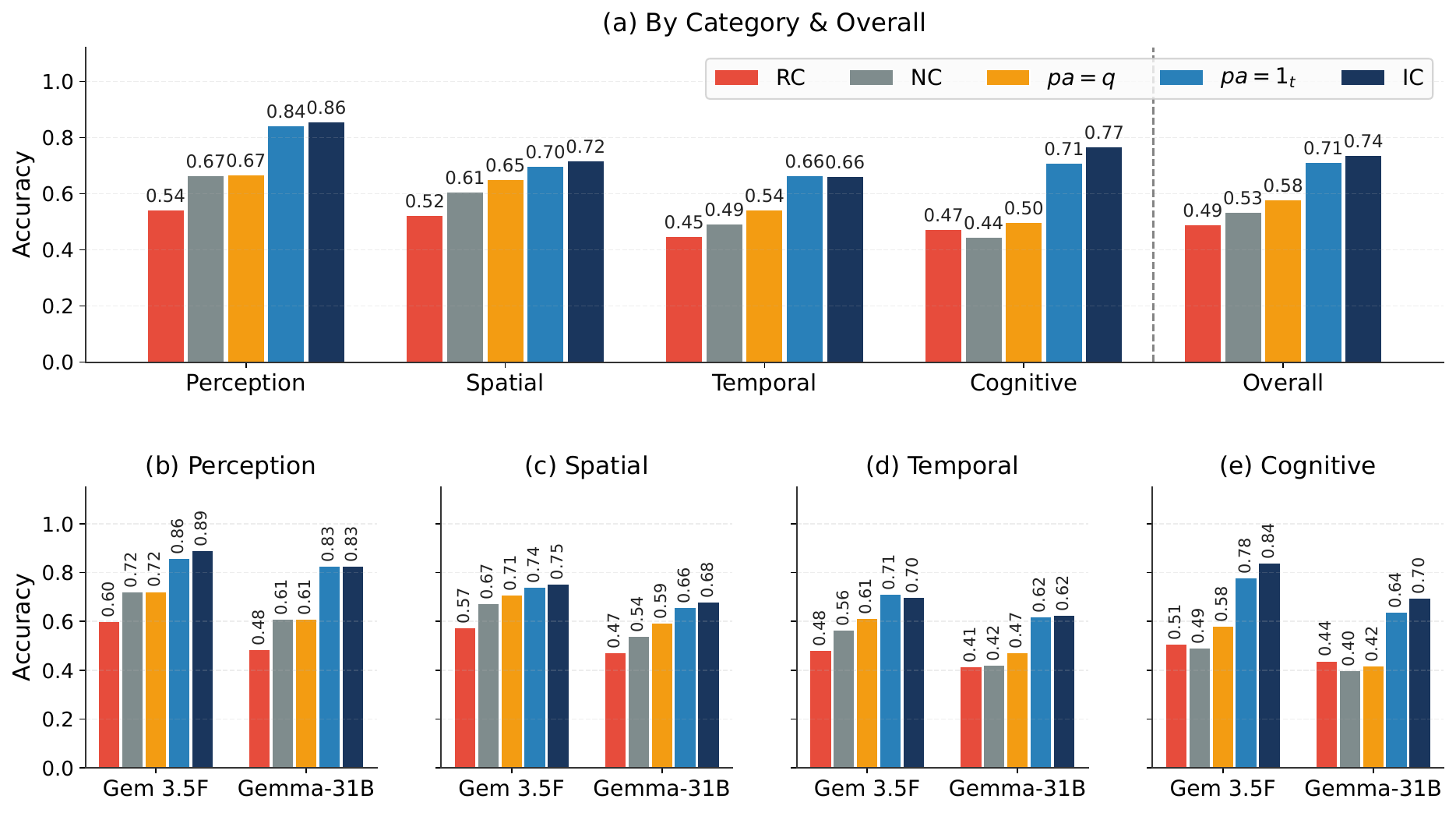}
    \caption{Capability on non-root unit tasks under five prerequisite states: RC, NC , question-only ($q$), text-answer-only ($1_t$), and IC. (a) Category-level and overall averages over 30 non-root unit tasks and two models; (b--e) per-model results within each category.}
    \label{fig:abl_capability}
  \end{minipage}\hfill
  \begin{minipage}[t]{0.38\textwidth}
    \vspace{0pt}
    \centering
    \includegraphics[width=\linewidth]{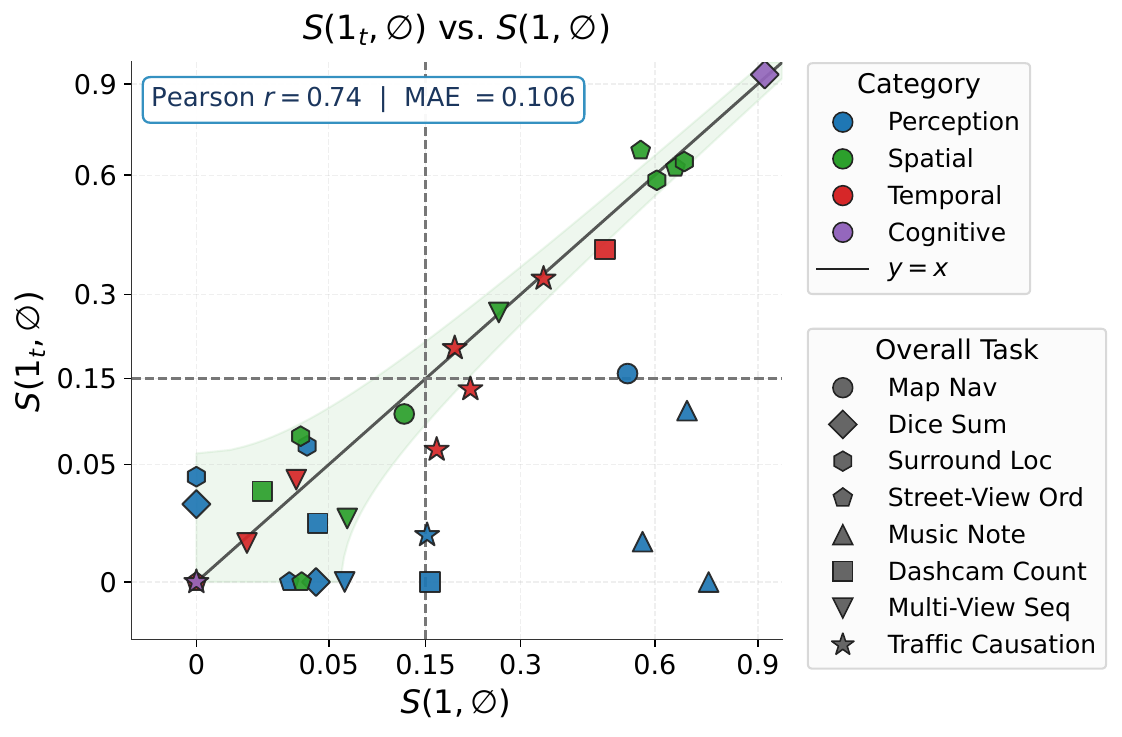}
    \caption{$S(1_t,\varnothing)$ vs.\ $S(1,\varnothing)$ for each prerequisite of an overall task, averaged over two models. Color denotes the prerequisite's capability category and marker shape the overall task. The solid line denotes $y=x$ and the shaded band $|y-x|\leq 0.06$; dashed lines mark the 0.15 threshold as in Figure~\ref{fig:finding_contribution_quadrant_1_empty} ($\mathrm{sqrt}(\cdot)$-scaled axes).} % TODO: label of main-paper Figure 5
    \label{fig:abl_s_parity}
  \end{minipage}
\end{figure}

\subsubsection{Ablation on Intervention Modality}
\label{app:abl_modality}

In the main experiments, the correct intervention ($A_i=1$) injects both a textual answer and, where applicable, visual annotations (e.g., bounding boxes, trajectories) onto the image. To examine the role of visual annotations in this intervention, we also evaluate a text-answer-only condition $A_i=1_t$, which provides the same textual answer as $A_i=1$ but keeps the original image unmodified. We run this condition on two models, Gemini 3.5 Flash from the top-4 (SOTA) group and Gemma-4-31B from the remaining (Others) group, across all 30 non-root unit tasks.

Figure~\ref{fig:abl_capability}(a) shows that $1_t$ closely approaches IC across all four capability categories. In temporal, the two conditions are virtually identical (both 0.66); in perception and spatial, they differ by only 0.02 (0.84 vs.\ 0.86; 0.70 vs.\ 0.72). The largest gap appears in cognitive (0.71 vs.\ 0.77), where $1_t$ nonetheless remains far above NC (0.44) and RC (0.47). Overall, $1_t$ reaches 0.71 against an IC of 0.74. The per-model results (Figure~\ref{fig:abl_capability}(b--e)) are consistent: for both models, $1_t$ tracks IC in every category, with a gap of at most 0.06 in any model--category pair. These results indicate that the textual answer accounts for the large majority of the capability gain under correct prerequisites, while the remaining difference between $1_t$ and IC reflects the additional contribution of visual annotations.

The same pattern holds for the causal contribution metrics. Figure~\ref{fig:abl_s_parity} compares $S(1_t,\varnothing)$ with $1_t$ in place of state $1$, against $S(1,\varnothing)$ for each prerequisite. % TODO: label of Definition 3.8
Spatial, temporal, and cognitive prerequisites lie along the diagonal, indicating that the magnitude and ranking of S-Scores are largely preserved without visual annotations. The systematic departures below the diagonal correspond to perception prerequisites, whose $S(1_t,\varnothing)$ falls well below $S(1,\varnothing)$. The corresponding N-Score comparison shows a similar agreement (Pearson $r=0.87$ between $N(1_t,\varnothing)$ and $N(1,\varnothing)$ across all prerequisites).

\subsubsection{Ablation on Question-Only Injection}
\label{app:abl_question}

In Section~\ref{subsec:exp_cap}, incorrect prerequisites improve cognitive tasks for most models (RC\,$>$\,NC), and in Section~\ref{sec:exp:contribution}, $S(0,\varnothing)$ is nonzero for some prerequisite categories, indicating that even incorrect answers occasionally help. % TODO: labels of Section 5.2 / 5.3
Since the incorrect intervention ($A_i=0$) injects both a sub-question and an incorrect answer, a natural follow-up is to ask how much of the prerequisite effect comes from the sub-question itself. We therefore evaluate a question-only condition $A_i=q$, which injects each prerequisite's sub-question without any answer, on the same two models. As shown in Figure~\ref{fig:abl_capability}(a), $q$ yields a modest improvement over NC (overall 0.58 vs.\ 0.53) but remains far below $1_t$ (0.71) and IC (0.74). Across categories, $q$ exceeds NC by at most 0.06, whereas $1_t$ exceeds NC by 0.09 (spatial) to 0.26 (cognitive). The causal contribution metrics show the same trend: $S(q,\varnothing)$ averages 0.05 across prerequisites, compared with 0.26 for $S(1,\varnothing)$ on the same two models, and $S(q,\varnothing)$ does not exceed 0.09 in any prerequisite category. 
Thus, providing the sub-question without answer does improve performance, meanwhile the improvement is limited and remains far smaller than that from providing its answer.

\begin{figure}[p]
  \centering
  \includegraphics[width=0.9\linewidth]{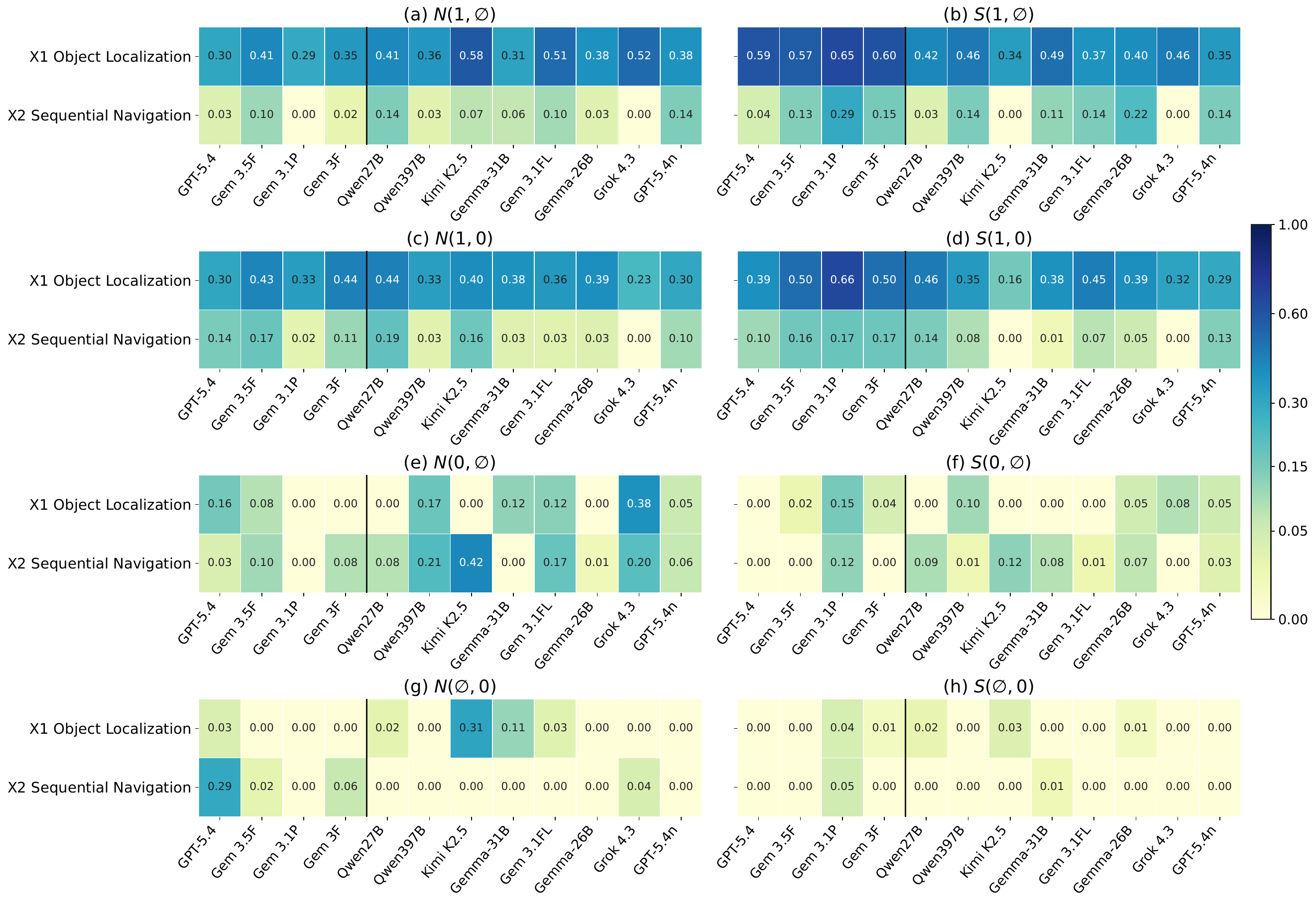}
  \caption{Per-model $N$- and $S$-Scores for the parents of Map Nav. under the four state contrasts. The vertical line separates SOTA from Others.}
  \label{fig:app_e3c_map}
\end{figure}
\begin{figure}[p]
  \centering
  \includegraphics[width=0.9\linewidth]{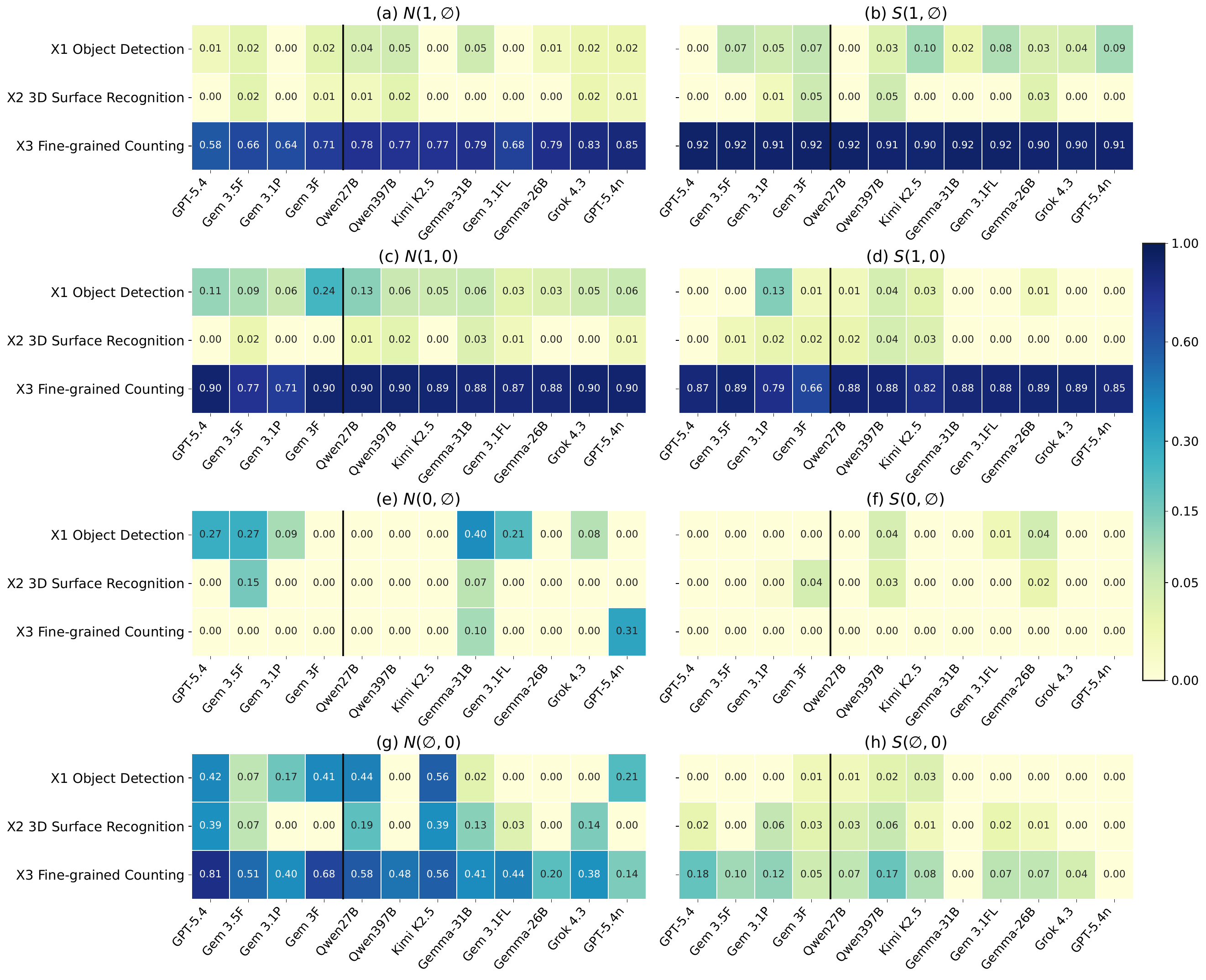}
  \caption{Same as Figure~\ref{fig:app_e3c_map}, for Dice Sum.}
  \label{fig:app_e3c_dice}
\end{figure}
\begin{figure}[p]
  \centering
  \includegraphics[width=0.9\linewidth]{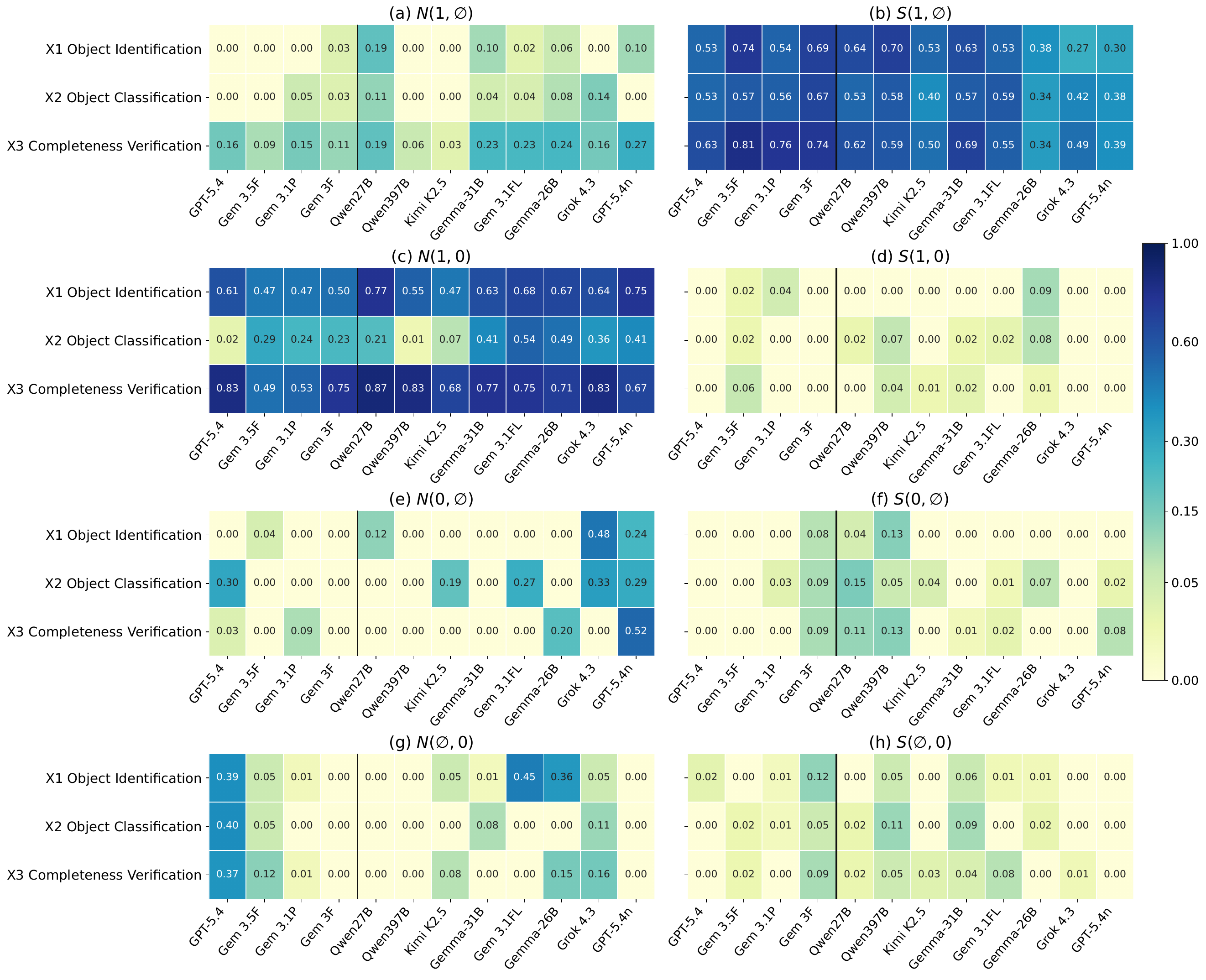}
  \caption{Same as Figure~\ref{fig:app_e3c_map}, for Music Note.}
  \label{fig:app_e3c_music}
\end{figure}
\begin{figure}[p]
  \centering
  \includegraphics[width=0.9\linewidth]{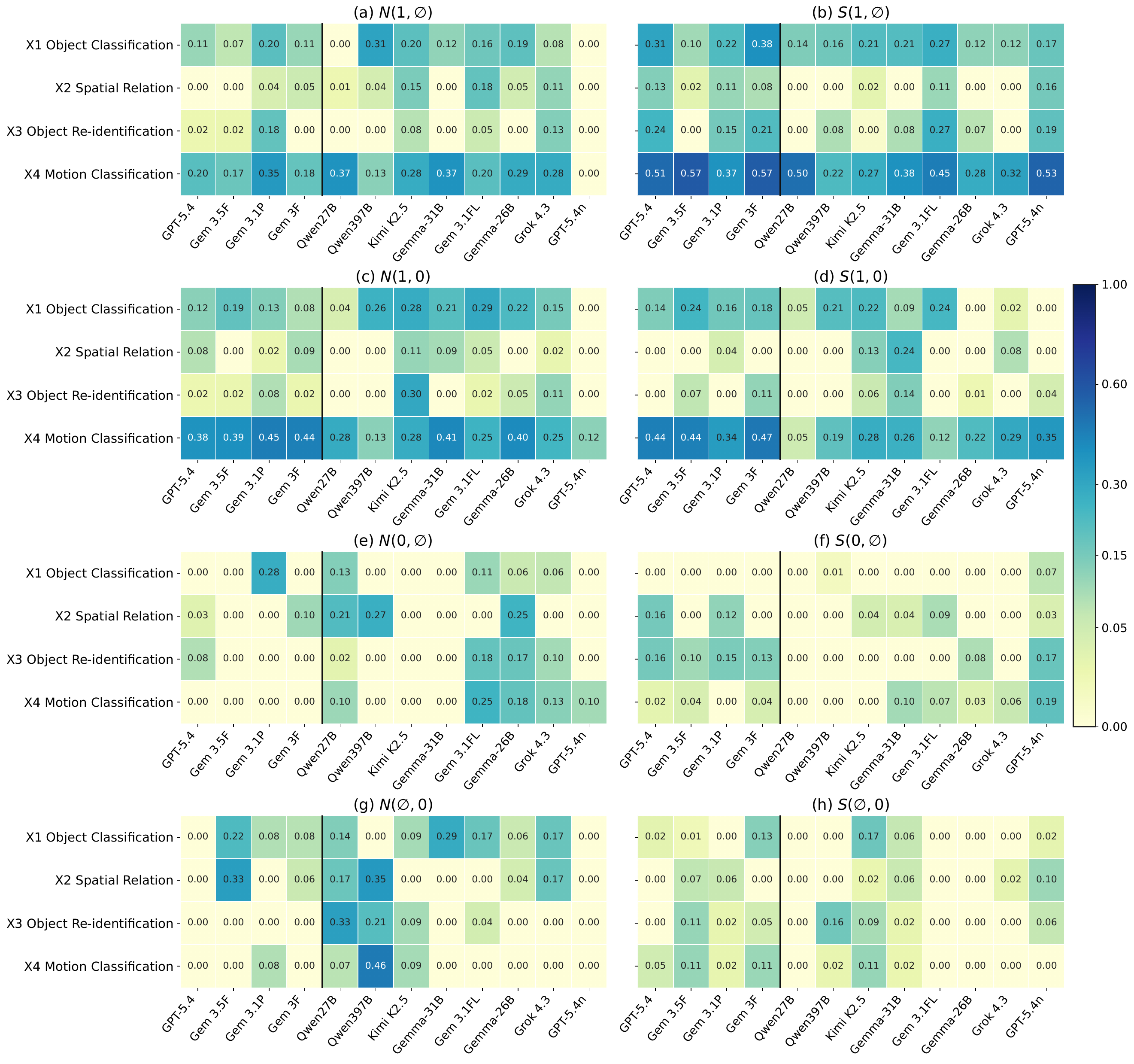}
  \caption{Same as Figure~\ref{fig:app_e3c_map}, for Dashcam Count..}
  \label{fig:app_e3c_dashcam}
\end{figure}
\begin{figure}[p]
  \centering
  \includegraphics[width=0.9\linewidth]{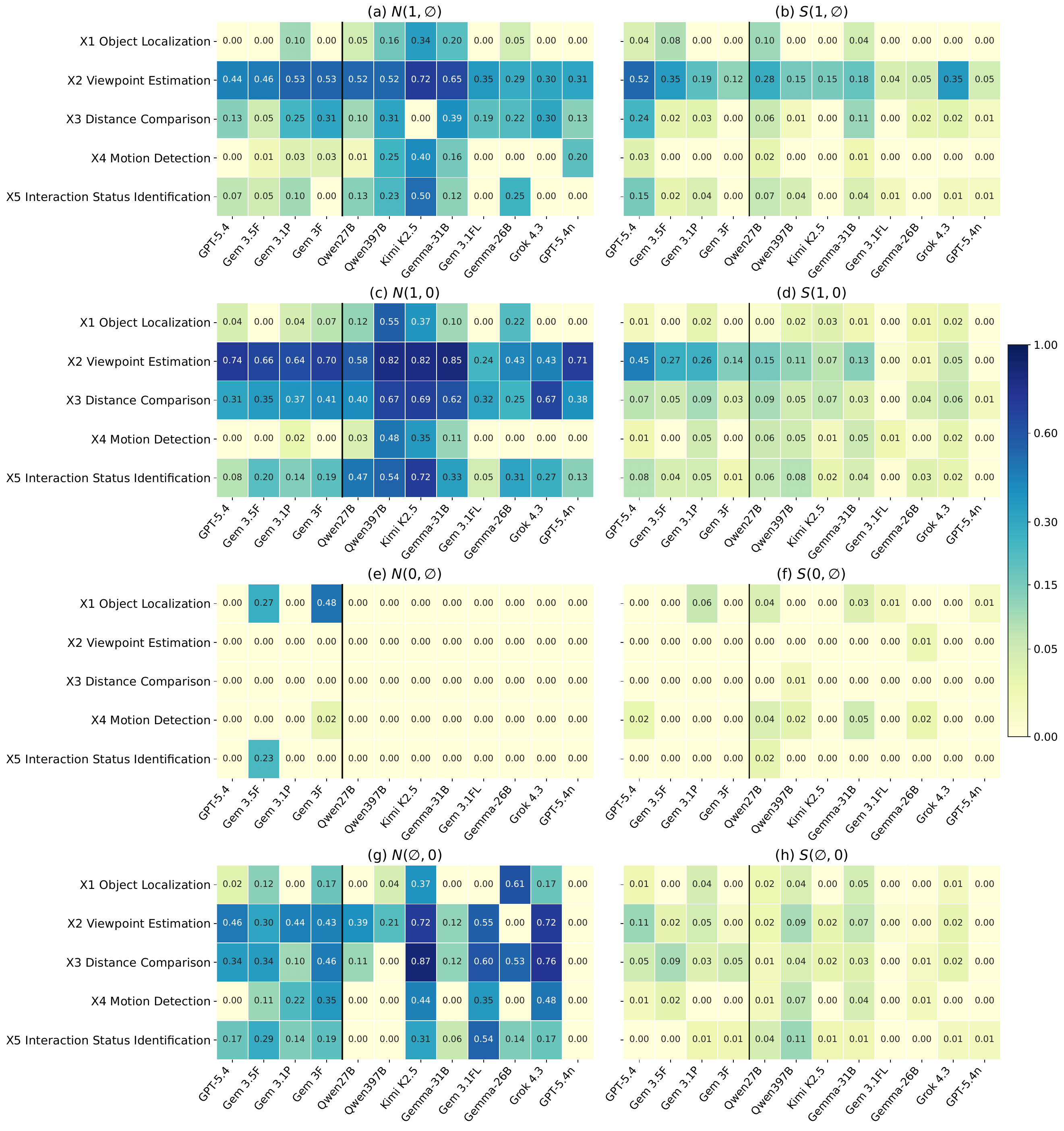}
  \caption{Same as Figure~\ref{fig:app_e3c_map}, for Multi-View Seq..}
  \label{fig:app_e3c_multiview}
\end{figure}
\begin{figure}[p]
  \centering
  \includegraphics[width=0.9\linewidth]{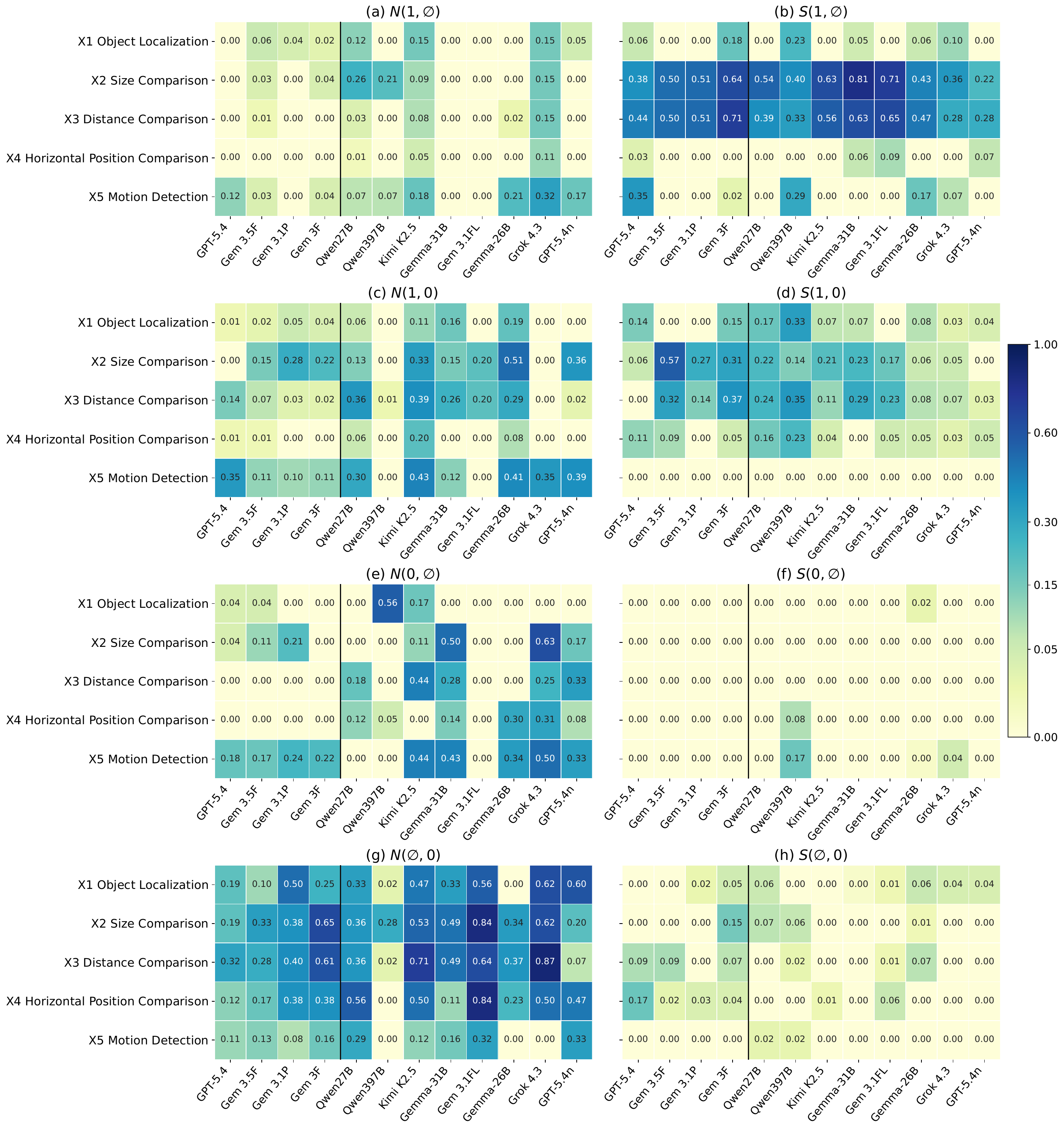}
  \caption{Same as Figure~\ref{fig:app_e3c_map}, for Street-View Ord..}
  \label{fig:app_e3c_streetview}
\end{figure}
\begin{figure}[p]
  \centering
  \includegraphics[width=0.9\linewidth]{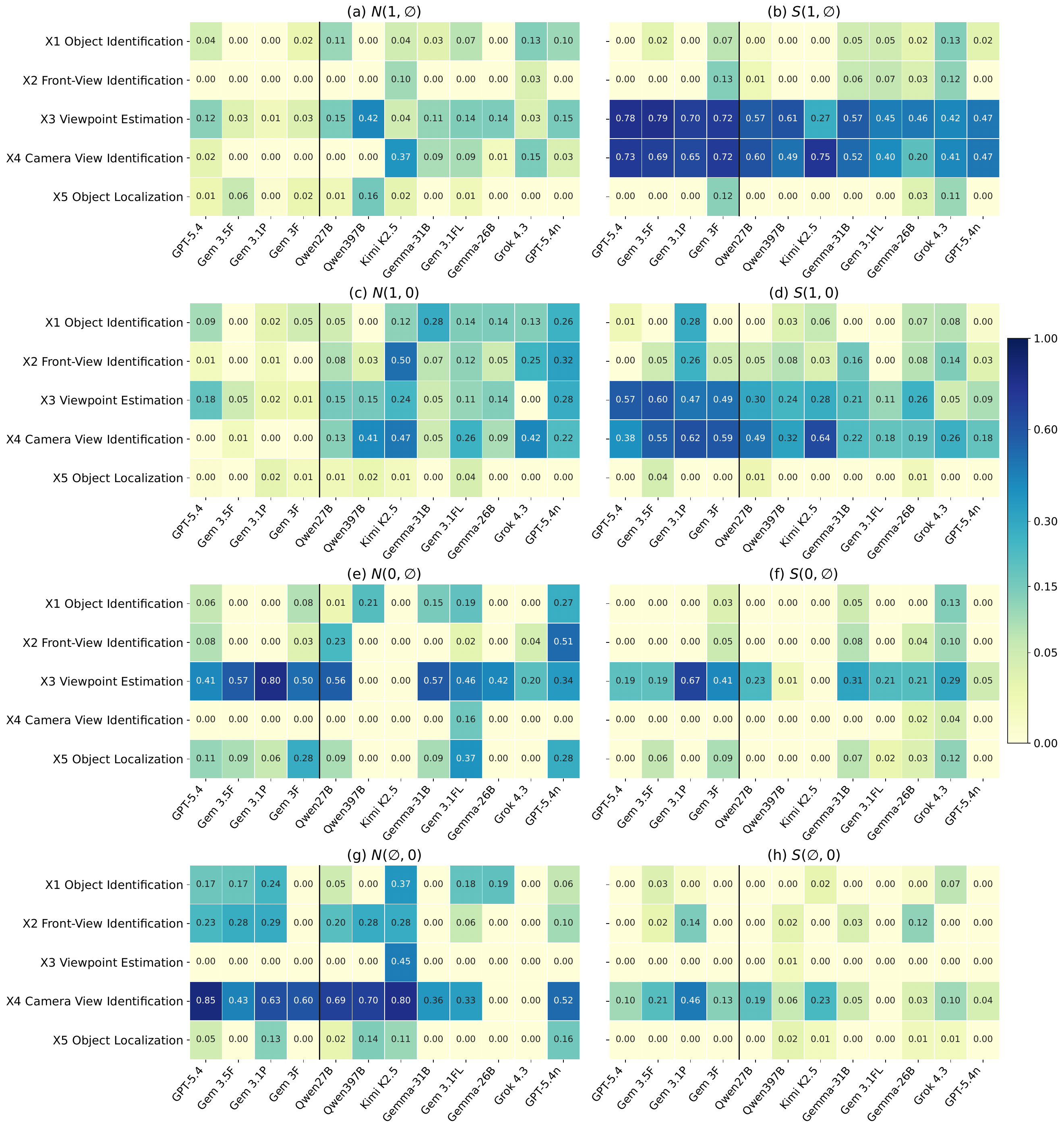}
  \caption{Same as Figure~\ref{fig:app_e3c_map}, for Surround Loc..}
  \label{fig:app_e3c_surround}
\end{figure}
\begin{figure}[p]
  \centering
  \includegraphics[width=0.9\linewidth]{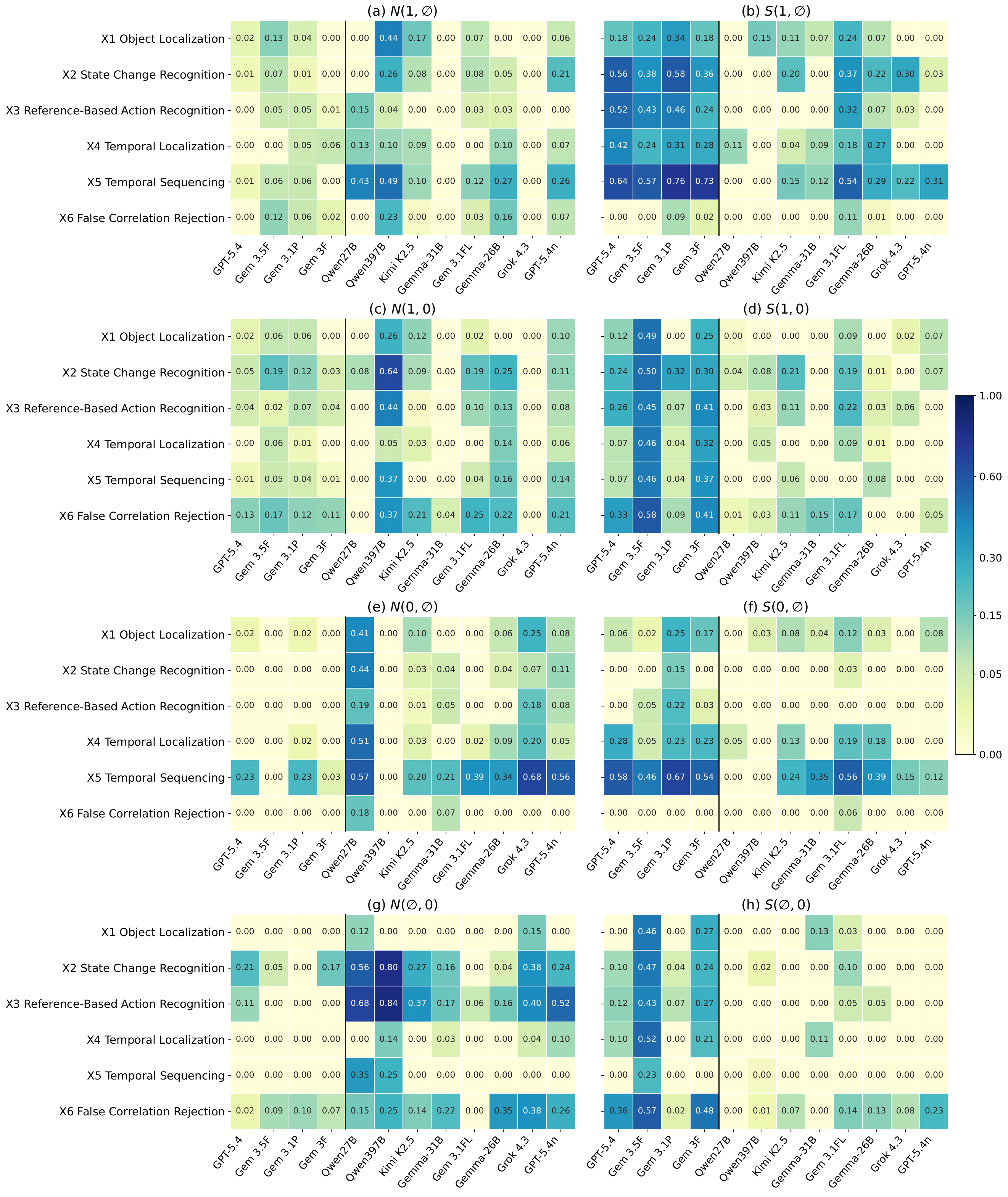}
  \caption{Same as Figure~\ref{fig:app_e3c_map}, for Traffic Causation.}
  \label{fig:app_e3c_traffic}
\end{figure}

%%%%%%%%%%%%%%%%%%%%%%%%%%%%%%%%%%%%%%%%%%%%%%%%%%%%%%%%%%%%%%%%%%%%%%%%
\section{Discussion}
\label{sec:discussion}

\textbf{On the Scope and Completeness of Task Decompositions.}

The SCM of each task declares, for every target task, which unit tasks are placed under controlled intervention. It is not intended as an exhaustive characterization of the task's causal structure or of the model's internal computation, and we do not claim that it is complete: the set of factors that may influence a model's answer on a real-world task is open-ended. A unit task is declared as a prerequisite of a target when solving the target requires the information it represents (Appendix.~\ref{appdx:subsec:bench_design_retionale}), and the same declared structure is used for all evaluated models. Any factor outside the declared parent sets belongs to the unmodeled factors captured by $N_i$ (Definition~\ref{def:scm}); it is not separately intervened on and remains in its natural regime.

Completeness is not required for the validity of the metrics. Every quantity entering the capability and contribution metrics is an interventional probability, estimated by directly running the model under the corresponding regime, including the natural regime $\varnothing$, rather than inferred from co-occurrences between unit-task outcomes. The lower-bound guarantees (Theorems~\ref{thm:general_n_score_bound} and~\ref{thm:general_s_score_bound}) hold for any configuration $\mathbf{c}$ of the remaining declared prerequisites and rest on prerequisite regimes being assigned through injection (Assumption~\ref{assum:interv}), not on the declared parent sets capturing every factor that influences the outcome. N- and S-Scores therefore remain valid lower bounds on PN and PS for the context actually realized, in which the target's declared prerequisites are set as specified and all other factors are left in their natural regime.

What the decomposition determines is the reference set against which each metric is read (Section~\ref{subsec:capability}): IC measures performance on $X_i$ when its declared prerequisites are supplied correctly, $\text{IC}-\text{NC}$ is the net gain from supplying them (the cascading component in Section~\ref{subsec:exp_cap}), and $1-\text{IC}$ is the residual error that persists under correct prerequisites.
The category-level findings in Section~\ref{subsec:exp_cap} are statements of this form, namely which share of the natural error is resolved by supplying the declared prerequisites and which share persists, obtained under the same inclusion criterion for every category. 

Finally, declared dependencies are task-defined, whereas the extent to which a model relies on each of them is measured rather than assumed. Averaged within each category, supplying the declared prerequisites improves performance for every model ($\text{IC}>\text{NC}$ in all 48 model-category cells; Tab~\ref{tab:appd-model-category-nc-ic-rc}). Individually, a declared prerequisite may have small N- and S-Scores under $(1,\varnothing)$ (Section~\ref{sec:exp:contribution}), meaning that neither withholding it in the oracle environment nor supplying it alone in the natural environment substantially changes the outcome. This can occur for a genuine prerequisite: both scores can be small when the model already resolves it correctly within the full task, N when its omission is compensated by the other supplied prerequisites, and S when the outcome remains limited by the other, unsupplied ones. Small scores therefore characterize a model's reliance under the given contrast.

\noindent\textbf{Scope and Limitations.}

\textit{Input-level interventions.}
In our framework, prerequisite states are controlled by injecting answers into the prompt, so all metrics characterize how a model's output responds to prerequisite information provided in its input, rather than how the prerequisite is represented inside the model. Operating at the input level allows the same controlled states to be applied across models, including closed-weight models accessible only through APIs, without relying on unverified readouts of intermediate states. 
How a model processes the injected prerequisite~(whether adopting it, checking it against the visual input, or overriding it—is) therefore part of the behavior measured by the metrics rather than an external assumption: IC measures task performance when correct prerequisite answers are available in the context, and RC measures performance when incorrect ones are supplied. Empirically, models consistently respond to the injected content: under the same prompt template, supplying correct answers improves performance in all 48 model-category cells, whereas supplying incorrect answers degrades all 12 models in perception and spatial tasks (Section~\ref{subsec:exp_cap}). Investigating the internal mechanisms by which models attend to and integrate injected prerequisites is an important direction for future work.

\textit{Design of the incorrect intervention.}
For the incorrect state $A_i=0$, we inject an answer randomly sampled from the valid answer space of unit task $X_i$ (Section~\ref{sec:benchmark}). This provides a standardized, model-independent perturbation that is identical across all evaluated models, ensuring that RC and the contribution scores involving state $0$ are directly comparable across models. It is intended to test how models behave under an explicitly false premise, rather than to simulate the natural errors that a specific model would make on its own, which are already captured under the unassisted state $A_i=\varnothing$.

\textit{Evaluation cost.}
Evaluating a target task $Y$ with $k=|\mathrm{Pa}(Y)|$ prerequisites across the full ternary state space would in principle require $3^k$ prerequisite configurations. By setting the remaining prerequisites $A_{-i}$ to homogeneous states (Appendix Section~\ref{appdx:subsec:remaining_parent}), our framework reduces the required runs per instance to linear in $k$: for each prerequisite $X_i$, pairing $A_i \in \{1,0,\varnothing\}$ with $A_{-i} \in \{1,0,\varnothing\}$ yields nine state pairs, where the three all-same states ($\text{NC}, \text{IC}, \text{RC}$) are shared across all prerequisites, resulting in $3+6k$ runs per instance across all four contrasts. In this paper, we evaluate all four contrasts to provide a complete characterization of model behavior across the ternary intervention space. In practical applications, the $(1,\varnothing)$ contrast is the most directly informative, as it diagnoses how much each prerequisite limits and can improve performance relative to the model's natural state; evaluating $(1,\varnothing)$ alone requires only $2+2k$ runs per instance (the standard end-to-end evaluation plus $1+2k$ interventional runs). Exploring automated task decomposition is also a promising direction for scaling this diagnostic pipeline to broader task domains.

\textit{Scope and future directions.}
This work focuses on diagnostic evaluation, establishing a causal framework and benchmark to identify where and why MLLMs fail on compositional tasks. Developing training or inference-time methods to address the diagnosed bottlenecks is a natural next step that we hope our work will inform. More broadly, our diagnostic framework can be extended to several related areas in future work, including localizing failure bottlenecks in multi-step agentic workflows, guiding targeted capability improvement in post-training and reinforcement learning, and complementing chain-of-thought reasoning analysis.

\end{document}